\documentclass[final,12pt]{alt2021} 

\title[Self-Tuning Bandits over Unknown Covariate-Shifts]{Self-Tuning Bandits over Unknown Covariate-Shifts}
\usepackage{times}

\altauthor{%
 \Name{Joseph Suk} \Email{js5338@columbia.edu}\\
 \addr Columbia University
 \AND
 \Name{Samory Kpotufe} \Email{skk2175@columbia.edu}\\
 \addr Columbia University%
}

\usepackage[utf8]{inputenc} 
\usepackage[T1]{fontenc}    
\usepackage{hyperref}       
\usepackage{url}            
\usepackage{booktabs}       
\usepackage{amsfonts}       
\usepackage{nicefrac}       
\usepackage{microtype}      
\usepackage{lipsum}
\usepackage{graphicx}
\graphicspath{ {./images/} }
\usepackage[title]{appendix}
\usepackage{bbm}
\usepackage{setspace}
\usepackage{natbib}

\usepackage{amsmath}
\usepackage{amsfonts}
\usepackage{graphicx}
\usepackage{mathrsfs,amsmath} 
\usepackage{amssymb}
\usepackage{titlesec}
\usepackage{hyperref}
\usepackage{etoolbox}
\usepackage{algorithm}
\usepackage{algorithmic}
\usepackage{enumitem}
\usepackage{xcolor}
\usepackage{thmtools}

\usepackage{chngcntr}

\newcommand{\skn}[1]{{{#1}}}
\newcommand{\sk}[1]{{#1}}

\DeclareMathOperator{\argmax}{argmax}

\newcommand{\remove}[1]{}

\newtheorem*{note}{Notation}
\let\oldnote\note
\renewcommand{\note}{\oldnote\normalfont}
\newtheorem{thm}{Theorem}
\newtheorem{prop}{Proposition}
\newtheorem{cor}{Corollary}
\newtheorem{assumption}{Assumption}
\newtheorem{defn}{Definition}

\newtheorem{rmk}{Remark}

\usepackage{eqparbox}

\usepackage{etoolbox}  
\makeatletter
\patchcmd{\algorithmic}{\addtolength{\ALC@tlm}{\leftmargin} }{\addtolength{\ALC@tlm}{\leftmargin}}{}{}
\makeatother

\newcommand\numberthis{\addtocounter{equation}{1}\tag{\theequation}}

\usepackage[format=plain,
  justification=RaggedRight, 
  singlelinecheck=false, 
  figurename=Fig., 
  aboveskip=7pt, 
  belowskip=0pt]{caption}
  
\usepackage{wrapfig}
\usepackage[calc]{adjustbox}
\usepackage{cutwin}
\newlength{\strutheight}
\hypersetup{
	colorlinks,
	citecolor=blue,
	filecolor=blue,
	linkcolor=blue,
	urlcolor=blue
}

\title{Self-Tuning Bandits over Unknown Covariate-Shifts}

\begin{document}
\maketitle

\begin{abstract}
Bandits with covariates, a.k.a. \emph{contextual bandits}, address situations where optimal actions (or arms) at a given time $t$, depend on a \emph{context} $x_t$, e.g., a new patient's medical history, a consumer's past purchases. 
While it is understood that the distribution of contexts might change over time, e.g., due to seasonalities, or deployment to new environments, the bulk of studies concern the most adversarial such changes, resulting in regret bounds that are often worst-case in nature. 

\emph{Covariate-shift} on the other hand has been considered in classification as a middle-ground formalism that can capture mild to relatively severe changes in distributions. We consider nonparametric bandits under such middle-ground scenarios, and derive new regret bounds that tightly capture a continuum of changes in context distribution. Furthermore, we show that these rates can be \emph{adaptively} attained without knowledge of the time of shift (change point) nor the amount of shift. 
\end{abstract}

\section{Introduction}
Bandits with covariates, or \emph{contextual} bandits, concern situations where the reward of an action  depends on a current context $x$, e.g., a patient's medical record (actions are treatments), or a user's profile and past history (actions are new products to propose). The problem is to maximize the total rewards of actions over time as similar contexts appear and rewards are observed over past actions. We consider the \emph{stochastic} setting where covariates and rewards are jointly distributed at any time. 

In the nonparametric version, little is assumed about the distribution of rewards over contexts, beyond \emph{Lipschitz} conditions that capture the idea that rewards should be somewhat close for nearby contexts. Now suppose contexts are drawn from a \emph{fixed} distribution; this ensures typicality, i.e., we can expect similar contexts to have appeared previously, so there is much potential to gradually learn the right actions for each context. Most recent advances have been made in nonparametric settings with a fixed distribution, with early consistency results in \citet{yang2002randomized}, minimax rates in \citet{rigollet-zeevi, perchet-rigollet}, followed by various refinements on the earlier algorithmic approaches \citep{reeve, guan-jiang}. 

However, it is understood that the distribution of contexts can change over time, e.g., seasonal changes in consumers' profile and purchases, or the extension of clinical trials to new populations. This reality has received little attention so far\footnote{We actually are not aware of any such work to date for the nonparametric setting.} in the nonparametric literature on the problem, although it has long been recognized in the more established parametric 
setting on contextual bandits, under various formalisms of often adversarial nature \citep{besbes2014nonstationary,hariri2015adapting, karnin2016multi, luo2017efficient, liu2018change, wu2018learning,chen2019nonstationary,cheung2019drift}. 

As an initial take on a nonparametric setting with changes in distributions, we focus attention to the less adversarial case of \emph{covariate-shift}, a formalism often adopted in works on \emph{domain adaptation} in classification, starting with {\citet{sugiyama2008direct, cortes2008sample, gretton2009covariate, ben2012hardness}}. For example, consider a situation where clinical trials are to be extended to a new population, or where the makeup of the underlying population changes unbeknownst to the experimenter. While the distribution on patients' profiles $X$ changes, the predictors in $X$ (e.g., biometrics, medical history) are expected to remain predictive of treatment outcomes; in other words, the conditional distribution of rewards given $X$ remains unchanged. Formally, in \emph{covariate-shift}, $Q_{Y|X} = P_{Y|X}$ but $Q_X \neq P_X$, where $P$ and $Q$ denote previous and new joint-distributions on context-reward pairs $(X, Y)$. 

\paragraph{Algorithmic and Statistical Goals.} We are interested in achievable regret in the time period corresponding to a fixed distribution $Q_X$ over contexts, right after one or multiple unknown shifts in context distribution. Let $P_X$ be a previous such distribution. Intuitively, performance under $Q_X$ depends on \emph{how far} $P_X$ is from $Q_X$, since typical contexts under $Q_X$ might not have been observed under $P_X$. In particular, prior distributions $P_X$ can bias an algorithm towards choices that are suboptimal to performance under $Q_X$ {(see Remark \ref{rem:biasedchoices} for biased choices under typical nonparametric approaches)}. We therefore aim for a procedure which \emph{not only} can automatically recover from such bias, but can draw as much benefit as possible from past covariate distributions $P_X$, whenever they are not too far from $Q_X$. \emph{As such, our analysis is more optimistic, and not only recovers the vanilla rates achievable without distributional shifts, but also integrates the information gained before the shift, resulting in faster rates.} In particular, our regret bounds tightly characterize the \emph{effective amount of past experience} contributed to $Q$ by previous runs on earlier distributions $P$'s, in terms of both the length of previous runs and the discrepancies $P_X \to Q_X$.  

We remark that the covariate-shift problem in this setting is harder than in classification where established approaches would compare observed $X\sim Q_X$ to prior observations $X\sim P_X$ to evaluate and adjust to the change. Here, however, we might not know the change point, and therefore cannot readily identify which data is which: for example, in an ongoing clinical study, or online recommender system, seasonal shifts in population makeup are unlikely to be known a priori. 
Interestingly, while one might then try and detect such change points, we show that this is not necessary. Our proposed procedure automatically \emph{adapts} to unknown change points, along with unknown levels of change in covariate distributions, while achieving regrets of near optimal order for the setting. 

\paragraph{Capturing distributional changes.} In order to admit a more natural range of covariate distributions across changes, we had to significantly relax the usual distributional conditions in prior work on nonparametric bandits. Namely, a so-called \emph{strong-density} condition was always assumed, which roughly states that the density of covariates is nearly uniform on the support and greatly simplifies the analysis with respect to margin conditions on rewards (see Remark \ref{rem:strongdensityand margin}). We instead only assume that distributions are supported inside $[0, 1]^D$, yet are able to properly capture the benefits of margins in rewards through careful integration arguments. \emph{Such new arguments are likely of independent interest even in the vanilla bandit settings with no distributional shift.} 

Finally the relation between distributions across runs is captured through a so-called \emph{transfer-exponent} adapted from earlier work on covariate-shift in classification \citep{kpotufe-martinet}.

\paragraph{Other Related Work.} There is by now an expansive literature on \emph{parametric}\footnote{The term characterizes settings displaying $O(\sqrt{n})$ regrets, due to either parametric constraints on rewards or on \emph{hindsight} baseline policies.} contextual bandits, ranging from fully adversarial to stochastic settings \citep[see e.g., ][]{hazan2007online,langford2008epoch, bubeck2012regret,  auer2016algorithm, rakhlin2016bistro}. 

In the stochastic parametric setting, the earlier cited results \citep{besbes2014nonstationary,hariri2015adapting, karnin2016multi, luo2017efficient, liu2018change, wu2018learning,chen2019nonstationary,cheung2019drift} are closest in spirit to the present work. However, they consider settings of a more adversarial nature, as their aim is to achieve regrets -- over stationary periods of length $\Delta_t$ -- of similar order $O(\sqrt{\Delta_t})$ as would have been achieved without distribution shifts; in other words, they contend that past experiences could adversarially affect regret over stationary periods, and the aim is to mitigate such adversity. In contrast, as we will see, past experience is actually useful  under covariate-shift (to a variable extent depending on shift characteristics), as long as the bandits procedure is reasonably conservative in least-observed regions of context space. In a similar vein, \citet{azar2013} considers a non-contextual setting where it is possible to benefit from previous experience. Moreover, the work assumes knowledge of the time of shift, whereas we do not.

In recent works on {\em non-stationary} bandits \citep{besbes2014nonstationary, luo2017efficient, chen2019nonstationary, cheung2019drift}, regret is expressed either in terms of the number of shifts or in terms of a budget on the total variation between subsequent distributions. The latter is somewhat more natural, as it captures a total shift in distribution rather than the number of shifts. In  a similar natural way, our regret bounds are expressed over a total shift, as captured by an \emph{aggregate transfer-exponent}, and thus do not worsen with the unknown number of shifts in distribution (see Theorem \ref{thm:multiple-shifts} of Appendix \ref{app:multiple-shifts}, while for ease of presentation, Theorem~\ref{thm:adaptive-algo-one-arm} of the main text covers the case of a single unknown shift). 

\citet{slivkins2014contextual} considers nonparametric settings with Lipschitz rewards, however, with possibly adversarial non-stochastic contexts. Given adversarial assumptions on past contexts, the work requires more conservative procedures than ours (see e.g., experimental comparison in Appendix~\ref{app:experiments}). In particular, their regret bounds do not account for the benefits of \emph{margins in rewards between arms}, which can significantly reduce achievable regrets. Nonetheless, their bounds are expressed in terms of notions of metric dimensions which can be tighter than the 
notion of \emph{box-dimension} employed in the present work; however, their resulting bounds are only clearly tighter than ours in the regime without margin ($\alpha = 0$ in our notation) since otherwise (e.g., for sufficient margin $\alpha \geq d$) our rates' exponent changes to $1/2$, i.e. no longer depend exponentially on dimension, which they cannot avoid without considering noise margin.

Finally, in the setting of active online regression with multiple domains, \citet{chen-luo-ma-zhang} establishes adaptive regret guarantees in terms of the domain dimensions and durations.

We start with a formal setup in Section \ref{sec:setup}, followed by an overview of results in Section \ref{sec:results}. Algorithms and proof ideas are discussed in Section \ref{sec:Algorithms}.

\section{Setup}
\label{sec:setup}

\subsection{Bandits With Covariate-Shift}

We consider a finite set of actions (or arms) $[K] \doteq \{1, 2 \ldots, K\}$, and let $Y \in [0, 1]^K$ denote 
the rewards of each action $i \in [K]$. We assume that the covariate $X$, {lying in} ${\cal X} \doteq [0, 1]^{D}$, is jointly 
distributed with $Y$, and we therefore assume a random independent sequence\footnote{The indexing set $\sk{\mathbb{N}}$ denotes the natural numbers excluding $0$.} of covariate-reward  pairs $\{(X_t, Y_t)\}_{t \in \sk{\mathbb{N}}}$, \emph{identically distributed} over \emph{different} and possibly unknown periods of time\footnote{We use the terms \emph{time} $t$ or \emph{round} $t$ interchangeably, the latter in the context of a procedure.}. 

\paragraph{Single-shift vs Multiple shifts.} To simplify presentation, in the main part of the paper we focus on the case of a single shift in distribution, and analyze the case of multiple shifts in the appendix. The discussion loses little generality as the multiple shift case follows easily as shown in Appendix \ref{app:multiple-shifts}. 

In the simplest case of a single shift, we assume that for some $n_P \geq 0$, possibly a priori unknown, the sequence $\{(X_t, Y_t)\}_{t \in [n_p]}$ is i.i.d. according
to a distribution $P$, while $\{(X_t, Y_t)\}_{t > n_p}$ is i.i.d. according to a \emph{new} distribution $Q$ with different marginals. We will be interested in performance under $Q$, i.e., after such a shift. 


\begin{assumption}[Covariate shift]\label{def:covariateshift}
While the distribution of covariates $X_t$ might change overtime, the conditional distribution of $Y_t\mid X_t$ remains fixed (i.e., in our context $Q_X \neq P_X$, but $P_{Y\mid X} = Q_{Y \mid X}$). In particular, the aim is to \emph{maximize} expected rewards conditioned on $X_t$; this is captured through the \textbf{fixed regression function} $f: {\cal X} \to [0, 1]^K$ as 
	$f^{i}(x)\doteq\mathbb{E}\  (Y^{i}|X=x), \, i \in [K]$.  
\end{assumption}

In the bandits setting, a so-called \emph{policy}\footnote{We remark that the term
\emph{policy} is often used to denote a mapping from state (or covariate) to action; here we simply equate it with any decision procedure taking action based on current and past observations.} (or bandit procedure) chooses actions at each round $t$, based on observed covariates (up to round $t$) and passed rewards, whereby at each round $t$ only the rewards $Y_t^i$ of chosen actions $i$ are revealed. We say an arm $i$ is pulled if action $i$ is chosen by the policy. We adopt the following formalism.

\begin{defn}[Policy]
A policy $\pi \doteq \{ \pi_t\}_{t \in \sk{\mathbb{N}}}$ is a random sequence of functions 
$\pi_t: {\cal X}^t \times [K]^{t-1} \times [0, 1]^{t-1} \to [K] $. In an abuse of notation, in the context of a sequence of observations till round $t$, we will let $\pi_t \in [K]$ also denote the action chosen at round $t$. In the case of a \textbf{randomized} policy, i.e., where $\pi_t$ in fact maps to distributions on $[K]$, 
we will still let $\pi_t \in [K]$ denote the (random) action chosen at round $t$.

\end{defn}

We let $\textbf{X}_t \doteq \{X_s\}_{s\leq t}, \textbf{Y}_t \doteq \{Y_s\}_{s\leq t}$ denote the observed covariates and (observed and unobserved) rewards from rounds $1$ to $t$.
The performance of a policy is evaluated as follows (visualized in Figure~\ref{fig:nq}). 

\begin{defn}[Cumulative regret]
	Define the regret between rounds {$n_P < n$} of a policy $\pi$, as \[
	\textbf{R}_{n_P,n}(\pi)\doteq
	\sum_{t = n_P +1}^{n} \max_{i \in [K]} \left(f^i(X_t) -  f^{\pi_t}(X_t)\right ). 
    \]
In our context of a shift to $Q$, we often will use the short notation $\textbf{R}^Q_n (\pi)$ to denote $\textbf{R}_{n_P,n}(\pi)$. 
\end{defn}

\begin{figure} 

\centering 
\includegraphics[height=2cm]{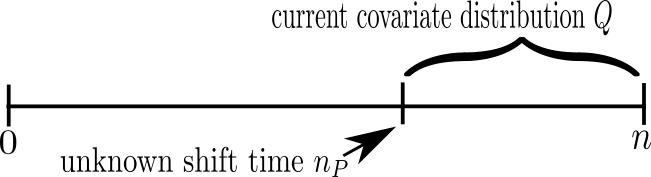}
\caption{\noindent \small We are interested in the regret $\textbf{R}_n^Q(\pi)$, w.r.t. a static distribution phase $Q$, restricted to the corresponding rounds.}
\label{fig:nq}
\end{figure}

The oracle policy $\pi^*$ refers to the strategy that maximizes the expected reward at any round $t$, and is given by 
$	\pi_t^*(X_t) \in {\argmax}_{i \in [K]} \, f^{i}(X_t) $.
The regret of a policy $\pi$ is therefore the \emph{excess} expected reward of $\pi^*$ relative to $\pi$ over ${\bf X}_n$. We seek a policy $\pi$ that minimizes $\mathbb{E}_{\textbf{X}_n,\textbf{Y}_n}\ \textbf{R}^Q_n(\pi)$. 

We emphasize that, while we will be interested in regret over particular periods $n_P+1: n$ (corresponding to a fixed $Q$), it is understood by definition that $\pi$ runs starting at $t=1$, and $\textbf{R}^Q_n(\pi) \doteq \textbf{R}_{n_P, n}(\pi)$ therefore depends on prior decisions up till time $n_P$.  
Finally, usual bounds for stationary distributions are recovered simply by letting $n_P =0$. 

\subsection{Nonparametric Setting}\label{sec:nonparametric}
Our main assumptions below are stated under the $\ell_\infty$ norm on $[0, 1]^{D}$ for convenience, as we build our procedures $\pi$ over regular grids of $[0, 1]^{D}$. It should be clear however that the relevant conditions hold under any norm (e.g., any $\ell_p$, $p\geq 1$) when they hold under $\ell_\infty$, by the equivalence of $\mathbb{R}^d$ norms. 

\paragraph {\bf $\bullet$ Standard Assumptions and Conditions.} 

We assume, \sk{as in prior work} on nonparametric contextual bandits \citep{rigollet-zeevi, perchet-rigollet, slivkins2014contextual, reeve, guan-jiang}, that the regression function is Lipschitz, with \sk{some known} upper-bound $\lambda$ on the Lipschitz constant (often simply assumed to be $1$).  

\begin{assumption}[Lipschitz $f$]\label{assumption:lipschitz}
	There exists $\lambda >0$ such that for all $i\in[K]$ and $x,x'\in\mathcal{X}$,
\begin{equation}\label{eqn:smoothness}
	|f^{i}(x)-f^{i}(x')|\leq \lambda\|x-x'\|_\sk{\infty}.
\end{equation}
\end{assumption}

Furthermore, the difficulty of detecting the optimal arm $\pi^*(x)$ at any $x$ is parametrized through the following \emph{margin} condition of $f$ w.r.t. $Q$, originally due to \sk{\citet{tsybakov2004optimal}} (for nonparametric classification).  

\begin{defn}[Margin Condition]\label{assumption:margin}
Let $f^{(1)}(x), f^{(2)}(x)$ denote the highest and second highest values of $f^i(x), i \in [K]$, if they are not all equal; otherwise let $f^{(1)}(x) = f^{(2)}(x)$ be that value.  There exists $\delta_0>0,C_{\alpha}>0$ so that $\forall\delta\in[0,\delta_0]$,
\begin{equation}\label{eqn:margin}
	Q_X(0<|f^{(1)}(X)-f^{(2)}(X)| \leq \delta)\leq C_{\alpha}\delta^{\alpha}.
\end{equation}
\end{defn}
In particular, the above is always satisfied with at least $\alpha =0$. 
Intuitively, the larger the margin $f^{(1)}(x)-f^{(2)}(x)$ at $x$, the easier it is to detect the best arm, in the sense that a rough approximation to $f$ is sufficient. The above condition, common in prior work on nonparametric bandits, encodes the margin distribution under $Q_X$. Interestingly, we need no assumption on the margin distribution under $P_X$, although our setting assumes that the procedure $\pi$ is first ran on covariates $X_t \sim P_X, t \leq n_P$; in fact, we will see that we only need to ensure that $\pi$ maintains good choices of arms for every potential $x\in \cal X$, along with sufficient arm pulls, up till the distribution shifts at round $n_P +1$.

\paragraph {\bf $\bullet$ A Relaxed Distributional Condition.}
In this work we only assume that marginals $P_X, Q_X$ are supported in $[0, 1]^D$. 
The following 
definition then serves to capture the \emph{complexity} of a support. 
The quantity $d$ therein, called a \emph{box dimension}, can be viewed as capturing the intrinsic dimension of the support ${\cal X}_Q$: low-dimensional supports intersect fewer cells of a partition of $[0, 1]^D$. We need no direct condition on past distributions $P_X$, as we simply need to capture their discrepancy from $Q_X$. 

\begin{defn}[Support complexity]\label{defn:bcn}
For $r \in \left\{2^{-i}: i \in \mathbb{N}\right \}$, 
     let ${\cal P}_r$ denote the regular partition of  $[0,1]^D$ into {hypercubes of side length $r$}. Denote by $\mathcal{G}(\mathcal{X}_Q,r)$ the number of cells of ${\cal P}_r$ which intersect $\mathcal{X}_Q$. We say that 
        $\mathcal{X}_Q$ has $(C_d,d)$ {\bf box dimension}, for $d \geq 1$, $C_d \geq 1$, if $\forall r \in (0,1]$, $\mathcal{G}(\mathcal{X}_Q, r) \leq C_d \cdot r^{-d}$.
\end{defn}%
The condition clearly holds for at least $d = D$, and would imply faster rates when $d \ll D$. It has been employed as a measure of the complexity of a data space, and we refer the reader, e.g., to \cite{scott2006minimax, clarkson2006nearest} for detailed expositions.

\emph{We emphasize that we do not assume knowledge of such support complexity $d$}. 

\begin{rmk}[Strong Density Condition] \label{rem:strongdensityand margin}
Prior cited work on nonparametric contextual bandits invariably assumed that $Q_X$ is \textbf{nearly} uniform on its support, namely that for any ball $B$ of radius $r$, 
$Q_X(B) \gtrsim r^d$ for some $d$ ($d = D$ is usual). This in particular ensures that, \textbf{uniformly over the data space}, any small region receives a sufficient amount of covariates $X$'s for good estimation of the expected reward $f$ at candidate arms; this simplifies both algorithmic design and analysis. This is not the case here, and we therefore require refined algorithmic choices and integration arguments.
\end{rmk}

\noindent {$\bullet$ \bf Quantifying the Shift from $P$ to $Q$.} 
Next we aim to quantify how much the earlier covariate distribution $P_X$ differs from the shift $Q_X$. Intuitively, $P_X$ has information on $Q_X$ if it yields data useful to $Q_X$, in other words, if it has sufficient mass in regions of large $Q_X$ mass. The next condition, adapted from recent work \citep{kpotufe-martinet} on classification, parametrizes such intuition. 

\begin{defn}\label{assumption:transfer-exponent}
 We call $\gamma\geq 0$ a {\bf transfer exponent} between $P_X$ and $Q_X$, if 
$\exists\ C_\gamma \geq  0$ such that, for all $\ell_\infty$ balls $B\subset [0, 1]^{D}$ of diameter $r\in(0, 1]$, we have 
    $P_X(B)\geq C_{\gamma} \cdot r^{\gamma}\cdot Q_X(B).$
\end{defn}

Note that the above condition always holds with at least $\gamma = \infty$ (this occurs, e.g., when $Q_X$ has mass outside $P_X$'s support). The larger the shift, the larger $\gamma$, with $\gamma = 0$ capturing the mildest such shifts in covariate distribution. Some examples, borrowed from \citet{kpotufe-martinet}, are given in Figure \ref{fig:gamma}. 
As we will see, the transfer exponent $\gamma$ manages to tightly capture a continuum of easy to hard shifts in covariate distributions as evident in achievable regret rates $\mathbf{R}^Q_n$ over the period corresponding to the current distribution $Q_X$.

\begin{figure} 

\centering 
\includegraphics[height=3.6cm]{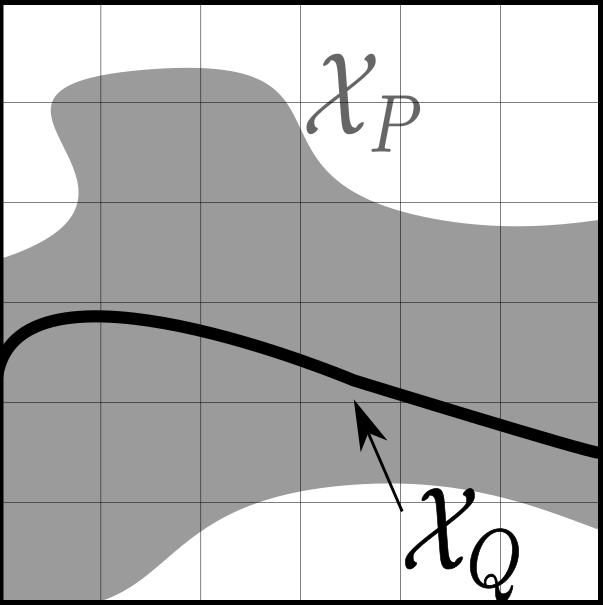}\hspace{0.9cm}
\includegraphics[height=3.7cm]{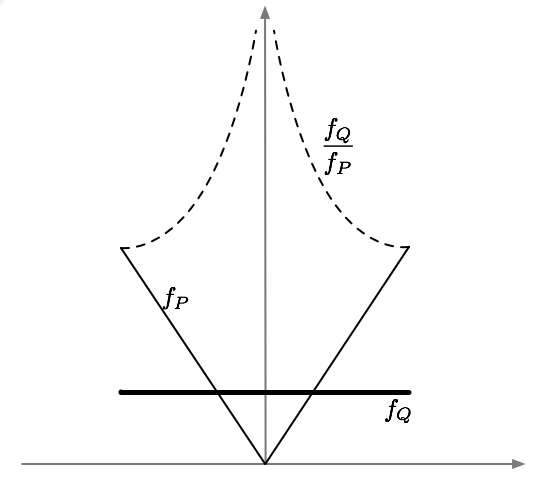}\hspace{0.9cm}
\includegraphics[height=3.7cm]{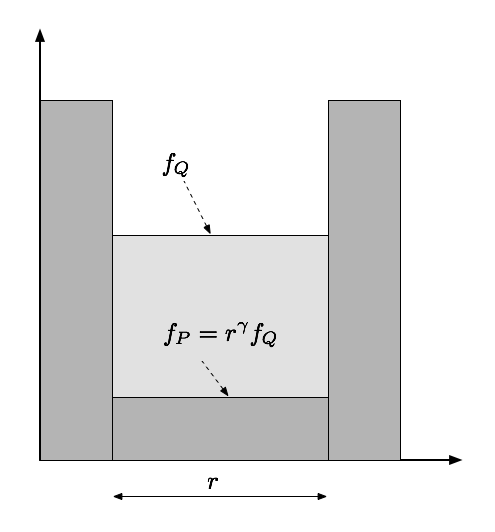}
\caption{\noindent \small Some settings with $0< \gamma < \infty$. Left: $Q_X$ is of smaller dimension $d$, while $P_X$ is of higher dimension $d_P$ (with gray support); $\gamma$ can then be shown to be $(d_P - d)$ under mild regularity. Center: the density $f_P\propto |x|^\gamma$ goes fast to $0$, while $f_Q$ is uniform; $f_Q/f_P$ then diverges (so \emph{density ratios}, and \emph{$f$-divergences} are ill defined). Right: $P_X$ moves mass away from regions of large $Q_X$ mass, with relative densities captured by $\gamma$ and the size of the region ($r$).}
\label{fig:gamma}
\end{figure}

\section{Results Overview}
\label{sec:results}

A common algorithmic approach in nonparametric contextual bandits, starting from earlier work 
\citep{rigollet-zeevi, perchet-rigollet}, is to maintain \emph{tree}-based (regression) estimates $\hat f_t$ of the expected reward function $f$, so that at any time $t$, upon observing $X_t$, only those arms $i$ with $f^i(X_t)$ close to $f^{(1)}(X_t)$ might be played. This assumes a good estimate of $f$ at any time $t$, which in the context of tree-based estimates boils down to choosing an optimal \emph{level} in the tree -- where each level $r$ corresponds to a piecewise-constant regression estimate $\hat f_t$ over bins of side-length $r$ in ${\cal X}$. 

\begin{rmk}[Crucial algorithmic choices can be biased] \label{rem:biasedchoices}
 Assuming a \textbf{strong density} condition on $Q_X$ (where balls of radius $r$ have mass of order $r^d$ for some $d$), in the usual setting with a stationary distribution $Q$, an optimal level $r = r_t$ is chosen as $O(t^{-1/(2+d)})$ 
 yielding optimal regression that would result in the best provable regret rates \citep{rigollet-zeevi}. However, suppose in our context with distribution drift that covariates were at first distributed as $P_X$, forcing a different choice of level 
 $O(t^{-1/(2+d_P)})$, i.e., under a different covariate \emph{dimension} $d_P$. This would be suboptimal under $Q_X$, which we wish to quickly detect. 
\end{rmk} 

Furthermore, under our relaxed conditions on $Q_X$, it turns out that the optimal oracle choice of level $r_t$ (even for a fixed distribution) is of the form 
$t^{-1/(2+\alpha + d)}$ -- i.e., higher levels in terms of \emph{unknown} margin parameter $\alpha$  -- so as to ensure sufficient samples near decision boundaries. 

The right choice of level is further complicated in context with distributional shift: 
if we want to optimally benefit from past observations under $P$, after the shift to $Q$, 
the optimal choice of level has to also account for the discrepancy between $P$ and $Q$ (as captured here by $\gamma$). 
As it turns out, a first analysis (not shown, but implicit in our arguments) reveals that the optimal choice of level is then of the form $\skn{r_t(\gamma, n_P) \doteq O(\min\{n_P^{-1/(2+\alpha+d + \gamma)}, (t-n_P)^{-1/(2+\alpha+d)}\})}$, i.e., now also depends on \emph{unknown} discrepancy level $\gamma$, and potentially \emph{unknown} change point $n_P$. A main aim is therefore to design a procedure which, without such knowledge, still makes near optimal adaptive choices of levels at any time $t$.


Our adaptive strategies, detailed in Section \ref{sec:Algorithms}, rely directly on the relative proportions of samples observed on a path from the root of a tree $T$ down to a leaf containing $X_t$. Roughly, let $n_r(X_t)$ denote the covariate count in the bin containing $X_t$ at level $r$ (by time $t$) locally at $X_t$. We then choose, roughly, the smallest level $r$ such that $n_r^{-1}(X_t) \leq r^2$. For intuition, this choice roughly balances regression variance (controlled by $n_r^{-1}$) and bias (controlled by $r$). Such a choice stems from prior insights on adaptive tree-based regression with fixed data distribution but unknown $d$  \citep[see e.g. ][]{kpotufe2012tree}, which we show here to yield a regression rate similar to that of the oracle choice $r_t(\gamma, n_P)$. As a result, we can automatically adapt to unknown problem parameters $d, \gamma, \alpha, n_P$ while the algorithm remains agnostic to underlying shifts in covariate distribution.

However, such an adaptive choice introduces nontrivial book-keeping issues which complicate the analysis. Namely, the number of observed rewards for a given arm $i$ -- which drive the estimates $\hat f$ -- might significantly differ from the number of covariates $n_r(X_t)$ in a bin, as we eliminate suboptimal arms over time. 
Further care is thus required for such book-keeping on observed rewards (or \emph{arm pulls}).

\paragraph{Adaptive Bandits.}
Our main theorem considers an adaptive policy $\pi$ given by the randomized procedure of Algorithm~\ref{algo:one-arm} from Section~\ref{sec:Algorithms}, which operates with a \emph{confidence parameter} $\delta \in (0, 1)$. While $\pi$ is randomized to simplify choices of arms to pull, we note that a  deterministic variant with similar guarantees is feasible (see Remark~\ref{rem:scheduling}).

As previously mentioned, the first theorem below is stated under a single change in distribution $P_X \to Q_X$, for simplicity of presentation, while a similar result for multiple (unknown) changes in distribution is given in Appendix \ref{app:multiple-shifts}. 

Just as in previous work for the stationary case, the margin parameter $\alpha$ needs not be known, while in addition here, we do not need to know the drift parameters $n_P, \gamma$, nor the dimension $d$ of $Q_X$ either. The expectation in the statement below is over the entire sequence ${\bf X}_n, {\bf Y}_n\sim P^{n_P}\times Q^{n-n_P}$, plus the randomness in $\pi$.

\begin{thm}\label{thm:adaptive-algo-one-arm}
Let $\pi$ denote the procedure of Algorithm~\ref{algo:one-arm}, ran, with parameter $\delta \in (0, 1)$, up till time $n>n_P \geq 0$, with the change point $n_P$ possibly unknown. Suppose $P_X$ has unknown transfer exponent $\gamma$ w.r.t. $Q_X$, that $Q_X$ has $(C_d,d)$ box dimension, and that the average reward function $f$ satisfies a margin condition with unknown $\alpha$ under $Q_X$.
Let $n_Q \doteq n-n_P$ denote the (possibly unknown) number of rounds after the drift, i.e., over the phase $X_t\sim Q_X$. We have for some $C>0$:
\[
		\mathbb{E}\  \textbf{R}_n^Q(\pi)\leq Cn_Q\left[\min\left(\left(\frac{K\log\left(K/\delta\right)}{n_P}\right)^{\frac{\alpha+1}{2+\alpha+d+\gamma}},\left(\frac{K\log\left(K/\delta\right)}{n_Q}\right)^{\frac{\alpha+1}{2+\alpha+d}}\right) + \frac{K\log\left(K/\delta\right)}{n_Q} + n\delta\right]
	\]
\end{thm}
The following corollary is immediate. 
\begin{cor}\label{cor:one-arm-regret}
	Under the setup of Theorem~\ref{thm:adaptive-algo-one-arm}, letting $\delta=O(1/n^2)$ yields: \[
		\mathbb{E}\  \textbf{R}_{n}^Q(\pi)\leq Cn_Q  \left[\min\left(\left(\frac{K\log(Kn)}{n_P}\right)^{\frac{\alpha+1}{2+\alpha+d+\gamma}},\left(\frac{K\log(Kn)}{n_Q}\right)^{\frac{\alpha+1}{2+\alpha+d}}\right) + \frac{K \log(Kn)}{n_Q}\right].
	\]
\end{cor}

The above rates interpolate between two terms: one involving $n_P$ past observations and the drift parameter $\gamma$, the other involving $n_Q$. This second term is of the form $n_Q^{1-\frac{\alpha + 1}{2+ \alpha + d}}$ and is attained by the adaptive $\pi$ when there is no drift, i.e., for $n_P = 0$. Note that, under the \emph{strong density} assumptions of \cite{rigollet-zeevi, perchet-rigollet}, the regret would have been $n_Q^{1-\frac{\alpha + 1}{2+ d}}$, i.e., smaller due to the easier setting \skn{(our rates for the same algorithm would in fact be of the same order under the strong density assumption; see Theorem \ref{thm:adaptive-algo-one-arm-dm} of Appendix \ref{app:bm})}. 

The regret of $n_Q^{1-\frac{\alpha + 1}{2+ \alpha + d}}$, for $n_P =0$, is however tight under our relaxed assumptions on $Q_X$: this becomes evident by noticing that the corresponding average regret of $n_Q^{-\frac{\alpha + 1}{2+ \alpha + d}}$ matches minimax lower-bounds of  \cite{audibert-tsybakov} for classification under equivalent nonparametric conditions (see Remark \ref{rem:lowerbounds} below). 

For $n_P>0$, the regret, interpolating both terms, can be rewritten as 
$n_Q\cdot {\left (n_P^{d_\gamma} + n_Q\right)}^{-1 \land \frac{\alpha + 1}{2+\alpha+d}}$ for $d_\gamma = (2+\alpha+d)/(2+\alpha+ d+ \gamma)$; in other words $n_P^{d_\gamma}$ can be viewed as the effective amount of \emph{past experience} 
contributed despite the drift; the quantity $n_P^{d_\gamma}$ is largest when $\gamma = 0$, lowering regret, and vanishes as $\gamma \to \infty$, i.e., with larger discrepancy between $P_X$ and $Q_X$.
At $\gamma = \infty$, e.g., when $Q_X$ has sizable mass outside the support of $P_X$, past experience under $P_X$ can only improve constants in the regret under $Q$, as the rate defaults to what it would have been under $Q$ without distributional shift. 
Such intuition on adapting to increasing $\gamma$ is confirmed in simulations (Figure~\ref{fig:simulations}). 

\begin{cor}[total regret vs. ${\bf R}_n^Q$]\label{cor:total-regret}
	Under the setup of Theorem~\ref{thm:adaptive-algo-one-arm}, suppose also that the support of $P_X$, ${\cal X}_P$, has $(C_{d_P},d_P)$ box dimension and that the reward function $f$ additionally satisfies a margin condition with unknown $\alpha$ under $P_X$. Then, letting $\delta=O(1/n^2)$ in Algorithm~\ref{algo:one-arm}, the {\bf total regret} $\textbf{R}_{1,n}(\pi)$ is bounded as follows:
	\begin{align*}
		\mathbb{E}\,\textbf{R}_{1,n}(\pi) \leq C \left[ K\log(Kn) + n_P \left( \frac{K\log(Kn)}{n_P}\right)^\frac{\alpha+1}{2+\alpha+d_P} \right.\\
		\left. + n_Q\min\left(\left(\frac{K\log\left(Kn\right)}{n_P}\right)^{\frac{\alpha+1}{2+\alpha+d+\gamma}},\left(\frac{K\log\left(Kn\right)}{n_Q}\right)^{\frac{\alpha+1}{2+\alpha+d}}\right)\right]
	\end{align*}
	In contrast to the bound on the regret ${\bf R}_n^Q(\pi)$ considered in Theorem~\ref{thm:adaptive-algo-one-arm}, the above result is directly comparable to the total regret bounds typically considered in the literature. We note that a total regret bound can always be recovered from a bound on $\textbf{R}_n^Q(\pi)$ by setting $n_P=0$. However, it is not true that a bound on the total regret necessarily implies a bound on the regret after the time of shift.
\end{cor}

\begin{rmk}[Multiple shifts]
{As previously mentioned, the results readily extend to the case of multiple changes in distribution before time $n_P$, with $\gamma$ above replaced by a weighted average $\bar \gamma$ of transfer exponents between past $P_X$'s and the current $Q_X$ (see Theorem \ref{thm:multiple-shifts} of Appendix \ref{app:multiple-shifts}). }
\end{rmk}

\begin{rmk}[Lower Bounds]\label{rem:lowerbounds}
Finally, we note that the above rates are tight (up to $\log$ terms) in the sense that the average regret ${\left (n_P^{d_\gamma} + n_Q\right)}^{-\frac{\alpha + 1}{2+\alpha+d}}$ matches minimax lower bounds for classification under covariate-shift 
of \citep{kpotufe-martinet}. Formally, by a simple reduction via \textbf{online-to-batch} conversion, 
the contextual bandit problem is at least as hard as its classification counterpart (see discussion and Corollary \ref{cor:lower-bound} in Appendix~\ref{app:lower-bound}). 
\end{rmk}

\ifx
\begin{rmk}[Deterministic Alternative]\label{rem:scheduling}
	A non-randomized policy similar to Algorithm~\ref{algo:one-arm} can also attain the rates of Theorem~\ref{thm:adaptive-algo-one-arm}. Instead of randomizing the arm-pull on Line 12 of Algorithm~\ref{algo:one-arm}, we can instead {\em schedule} the pulls of all arms in ${\cal I}_B$ for future covariates falling in bin $B$, 
\end{rmk}

\fi

\begin{figure}
    \centering
    \includegraphics[height=4cm]{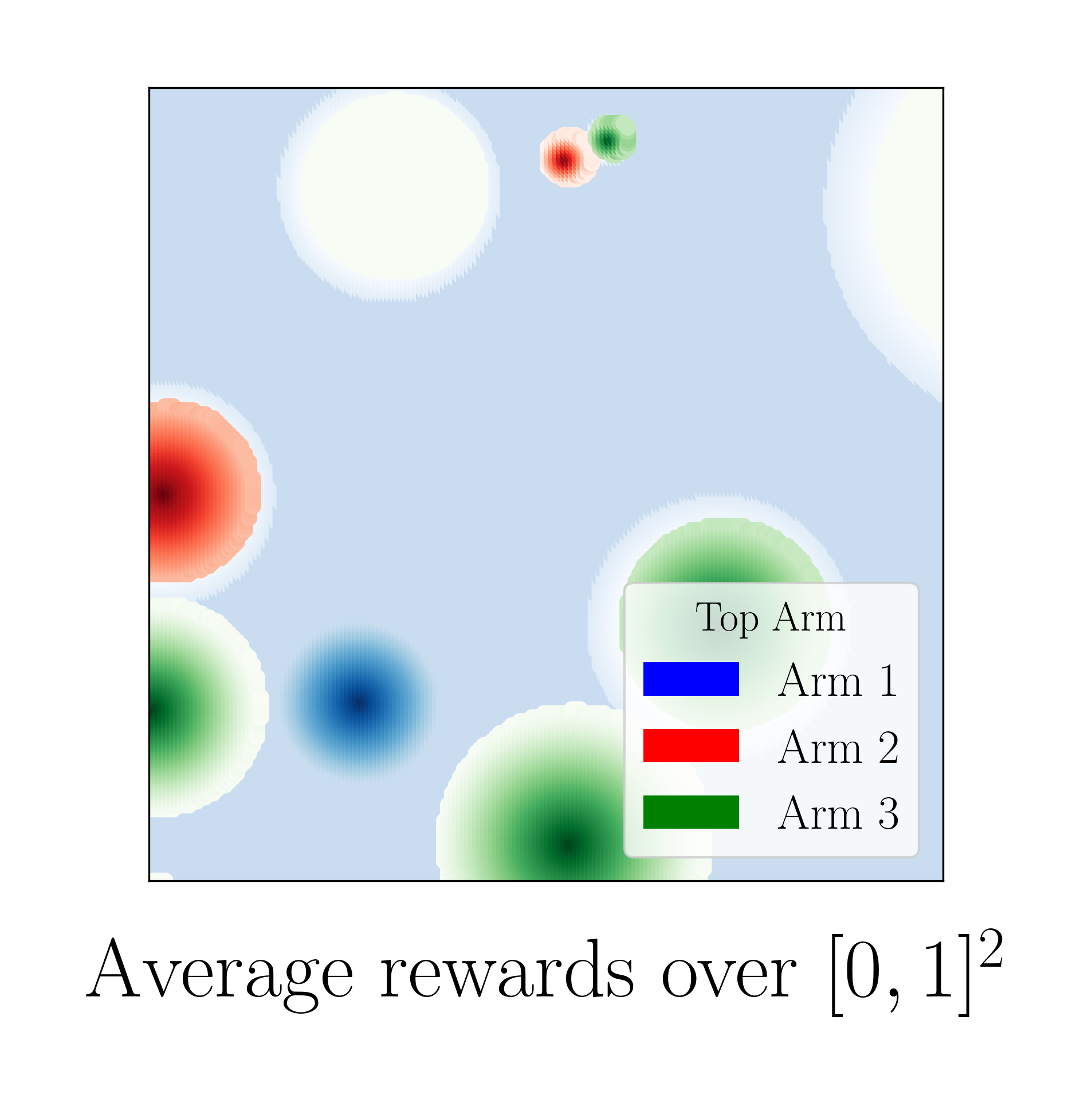} \hspace{0.3cm}
    \includegraphics[height=4cm]{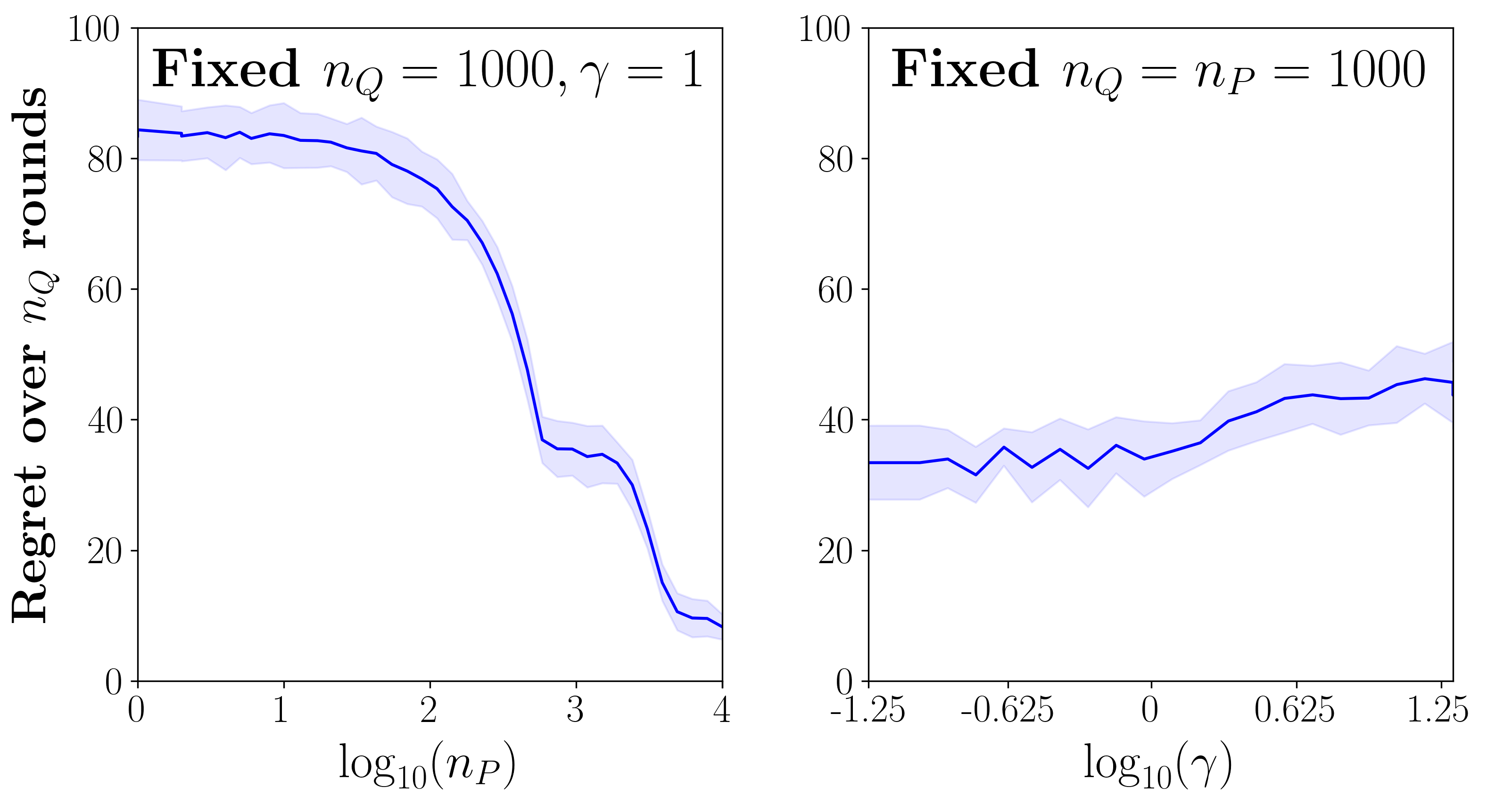} 
    \caption{\small Simulation Results. $Q_X\sim {\cal U}([0,1]^2)$, $P_X$ has density {$\propto\|x\|_2^{\gamma}$}, $K=3$ arms, with rewards $Y^i= f^i(X) + {\cal N}(0, .05), i\in [K]$, where  $f^i(x) \propto \sum_{k} \omega_{i,k} (1-\|x-z_k\|_2/r_k)_+$ for 25 randomly placed \emph{bumps} with centers $z_k$, Rademacher signs $\omega_{i,k}\in\{\pm 1\}$, and radius $r_k$. A profile of $f$ is shown on the left, with lower gradient colors corresponding to least margins (white meaning \emph{no} margin). The right plots average {20} runs of Algorithm \ref{algo:one-arm}, and verify the guarantees of Theorem \ref{thm:adaptive-algo-one-arm}, namely that the procedure \emph{adapts} to unknown shift parameters $n_p$ and $\gamma$. In particular, the amount of \emph{past experience} $n_P^{d_\gamma}$ clearly helps, and how much it helps depends on the level of shift $P\to Q$ as captured by $\gamma$. We also compare Algorithm~\ref{algo:one-arm} to an  \emph{adversarial} nonparametric contextual bandits algorithm from \citet{slivkins2014contextual} in Appendix~\ref{app:experiments}.
    } 
    \label{fig:simulations}
\end{figure}

\section{Algorithms and Analysis Overview}
\label{sec:Algorithms}

\subsection{Algorithm Overview}\label{subsec:algorithm-overview}

\begin{algorithm}[t]
\small
\setstretch{1}
\caption{Adaptive Bandits}
\label{algo:one-arm}
\begin{algorithmic}[1]
\STATE \textbf{Requires}: upper bound on Lipschitz constant $\lambda \geq 1$, set of arms $[K]$, tree $T$ with levels $r\in\mathcal{R}$
\STATE \textbf{Input Parameters}: $\delta\in(0,1)$, covariates $X_1,X_2,\ldots$
\STATE \textbf{Initialization}: For any bin $B$ at any level in $T$, set $\mathcal{I}_B \leftarrow [K]$ \hskip1em \# Set of candidate arms in $B$.
\vspace{0.1cm}
\FOR{$t=1,2,\ldots$}
\STATE If $t \leq \lceil 8K\log(K/\delta)\rceil$, play a random arm $i \in [K]$ selected with probability $1/K$. 
\vspace{0.2cm}
        \STATE Otherwise, for $t > \lceil 8K \log(K/\delta)\rceil$:
        \begin{ALC@g}
           		\STATE {\bf Choose a level $r_t \in \mathcal{R}$ for $X_t$:} 
           $r_t \leftarrow  \min \left\{ r\in\mathcal{R}: r \geq \sqrt{\frac{8 K \log(K/\delta)}{n_r(X_t)}} \right\}$
	        \vspace{0.2cm}
		    		\STATE {\bf Update candidate arms for the bin $B$ containing $X_t$ at level $r_t$:}
		    		\STATE  Set $\mathcal{I}_B\leftarrow\bigcap_{B'\in T,B\subseteq B'} \mathcal{I}_{B'}$ \COMMENT{Discard arms previously discarded by ancestor bins}
		    \vspace{0.2cm}
		    		\STATE  Compute $\hat{f}^i(B)$ for any $i\in\mathcal{I}_B$ over $B$ \COMMENT{See Definition~\ref{defn:arm-pull-counts} for estimate $\hat f^i$}
		\vspace{0.2cm}
		    		\STATE  Refine candidate arms: $\mathcal{I}_B \leftarrow \mathcal{I}_B \setminus \{i:  \hat{f}^{i}(B) < \hat{f}^{(1)}(B) - 8\lambda r_t\}$.
		\vspace{0.2cm}
	    			\STATE {\bf Play a random arm $i\in\mathcal{I}_B$ selected with probability $1/|\mathcal{I}_B|$.}
	    	\end{ALC@g}
\ENDFOR 
\end{algorithmic}
\end{algorithm}

All algorithms build on a dyadic partitioning tree $T$ defined as follows. 

\begin{defn}[Partition Tree]
Let ${\cal R} \doteq \{2^{-i}: i \in \mathbb{N}\cup\{0\}\}$, and let $T_r, r\in \cal R$ denote a 
regular partition of $[0, 1]^{D}$ into hypercubes (which we refer to as {\bf bins}) of \sk{side length} (a.k.a. bin size) $r$. We then define the dyadic {\bf tree} 
$T \doteq \{ T_r\}_{r \in {\cal R}}$, i.e., a hierarchy of nested partitions of $[0, 1]^{ D}$. 
We will refer to the {\bf level $r$ of $T$} as the collection of bins in partition $T_r$. The {\bf parent} of a bin $B \in T_r, r < 1$ is the bin $B' \in T_{2r}$ containing $B$; {\bf child}, {\bf ancestor} and {\bf descendant} relations follow naturally. The notation $T_r(x)$ will then refer to the bin at level $r$ containing $x$.
\end{defn}

Note that, while in the above definition, $T$ has infinite levels $r \in \cal R$, at any round $t$ in a procedure, we implicitly only operate on the subset of $T$ containing data. Our procedures, as in prior work on nonparametric bandits, maintain estimates $\hat f$ of the average reward function $f$ over levels of $T$. 

\begin{defn}[Regression estimates and arm pull counts]\label{defn:arm-pull-counts}
	At any round $t>\lceil 8K\log(K/\delta)\rceil$, for any bin $B$ at any level in the tree, we define the following regression estimate for arm $i$: \begin{align*} { \hat{f}_t^i(B)\doteq\frac{1}{m_t(B,i)}\sum_{X_s\in B,s\leq t-1,\pi_s=i} Y_s^i,} \end{align*}
	where {$m_t(B, i)$} denotes the number of times arm $i$ was pulled in $B$ before time $t$. If $m_t(B,i)=0$, we take $\hat{f}_t^i(B)=0$. For any $B$ at level $r$ in the tree, $\hat f_t^i(B)$ serves as a regression estimate for any covariate $x \in B$.   
	We often drop $B$ or $t$ in the above definitions, 
	when understood from context. 
\end{defn}



\begin{defn}[Covariate counts]\label{defn:bin-covariate-count}
	Let $B \doteq T_r(X_t)$. We write: 
	$ n_{r}(X_t) \doteq \sum_{s \in [t-1]} \mathbbm{1}\{X_s \in B\}.$
\end{defn}

At any round $t$, upon observing $X_t$, a level $r_t$ is chosen according to the \emph{covariate counts} $n_r(X_t)$ along the path $\{T_r(X_t)\}_{r\in \cal R}$. Roughly, $r_t$ is picked as the smallest $r\in \cal R$ such that 
$1/\sqrt{n_r(X_t)} \leq r$. The level $r_t$, more precisely the bin $B= T_{r_t}(X_t)$ containing $X_t$ at that level, then determines the estimate $\hat f(B)$ to be used at time $t$. 

To understand this choice, recall that a main aim is to quickly identify which arms are suboptimal -- and shouldn't be pulled for $X_t$, and we ought to therefore use a good estimate of $f(X_t)$. We will show that $r_t$ indeed provides such an estimate at a near optimal regression rate in terms of unknown $n_P$ and $\gamma$ (see Lemma \ref{lem:acb-bound}). In particular, the covariate counts $n_r(X_t)$, at any level $r$, account for covariates from both $P_X$ and $Q_X$, whenever $t> n_P$. Intuitively, we will then expect 
$n_r(X_t)\approx n_P\cdot P_X(T_r(X_t)) + (t-1-n_P)\cdot Q_X(T_r(X_t)) \gtrsim n_P\cdot r^{d+\gamma} + (t-1-n_P)\cdot r^d$. The choice of $r_t$ can then be shown to properly balance regression variance and bias in terms of unknown $n_P$ and $\gamma$. 


Once this choice is made, only those arms deemed safe are pulled for $X_t$ at time $t$. These so-called \emph{candidate arms} are maintained as ${\cal I}_B\subseteq [K]$ for each bin $B$ over time, and exclude identified suboptimal arms whose 
average rewards are clearly below that of the unknown best arm $\pi^*(x)$ for any $x\in B$. In particular, 
suppose at any time $t$, we can ensure that $|\hat f_t^i(B) - f^i(x)| \lesssim {r_t}$ for all remaining arms $i$ over $x\in B$. Then we can safely discard $i$ if $\hat f_t^{(1)}(B) - \hat f_t^{i}(B) \gtrsim {r_t}$. It then makes sense to also discard such an arm in all descendants of $B$. 

\begin{rmk}[Margin Adaptation]\label{rmk:margin-adaptation}
Adaptation to the unknown margin parameter $\alpha$ comes through such decisions over ${\cal I}_B $. 
Namely, if the margin $f^{(1)}(x) - f^{(2)}(x) \gg {r_t}$ for all $x \in B$, then \emph{all} suboptimal arms are discarded by time $t$ so we suffer no regret for $X_t \in B$ at time $t$. Otherwise, 
all arms $i$ left in ${\cal I}_B$ satisfy $f^{(1)}(x) - f^i(x) \lesssim {r_t}$, for $x\in B$, i.e., a bound on regret; on the other hand, the margin distribution ensures that the $Q_X$-probability of $X_t$ landing in such a bin with low margins is small (if $r_t$ were not a function of $X_t$). The main difficulty is to ensure that ${r_t}$ is of the right order in terms of $t$ and the unknown $n_P, \gamma$, even though it can significantly differ over the spatial location of $X_t$ (as we have minimal assumptions on covariate distributions -- see Section \ref{sec:regret-analysis} below).  
\end{rmk}

\begin{rmk}[Book-keeping]\label{rmk:book-keeping}
As discussed earlier, our adaptive choice of level $r_t$ brings in additional difficulty in the book-keeping of arm pulls. In fact, the above discussion assumes that covariate counts $n_r(X_t)$ (used in choosing $r_t$, towards adapting to unknown $n_P, \gamma$) and arm-pull counts $m_t(B, i)$ (used in estimating $\hat f_t(B)$) are of similar order. However, this needs not be the case for a couple reasons. First, a single arm $i$ is pulled whenever a covariate $X_t$ lands in a bin $B$ so that $m_t(B, j)$ is not updated for $j \neq i$ even though covariate counts for $B$ are. Second, more nuanced, the following situation can happen since we can have $r_t>r_{t-1}$ as $X_t, X_{t-1}$ fall in different regions of space: 
in a given bin $B \doteq T_{r_t}(X_t)$ chosen at time $t$, some arms in ${\cal I}_B$ might have been eliminated in a descendant of $B$ at an earlier time, and therefore not pulled as much as other arms in ${\cal I}_B$. Such situations do not arise under more common global choice of $r_t \leq r_{t-1}$, i.e., where choices of levels are monotonic in time. 

It turns out that, in fact, enough regularity is built into our adaptive choice to mitigate this issue, namely our choices of levels are approximately monotonic in that a descendant of $B$ cannot be chosen before the first time $B$ is ever chosen. This fact, and its implications, are outlined below in Section \ref{sec:regret-analysis}. 

\end{rmk}

\begin{rmk}[Deterministic Alternative]\label{rem:scheduling}
We note that a deterministic alternative to Algorithm~\ref{algo:one-arm} should be possible, by properly \emph{scheduling} arms to be played in a round-robin fashion over ${\cal I}_B$. We instead focus in the present work on the simpler random choice as the additional technicality appears to add no significant further insight.

\end{rmk}


\subsection{Outline for the Proof of Theorem~\ref{algo:one-arm}}
\label{sec:regret-analysis}
 Continuing on the discussion of Section~\ref{subsec:algorithm-overview}, here we give an outline of the proof of Theorem~\ref{thm:adaptive-algo-one-arm}. Supporting lemmas, propositions, and their proofs are found in Appendix~\ref{app:upper-bound-single}.

\paragraph{\textbullet\  Bounding the regression error.} First, it is easy to verify that the criterion for choosing $r_t$ on Line 7 of Algorithm~\ref{algo:one-arm} is well-defined for $t> \lceil 8K\log(K/\delta)\rceil$, that is, the set being minimized over is non-empty as it contains at least $r=1$. At any such round $t$ with selected bin $B \doteq T_{r_t}(X_t)$, we have by standard arguments (Lemma~\ref{lem:bias-variance} Appendix~\ref{app:upper-bound-single}) that, with probability at least $1-\delta$ (over random rewards, conditioned on all past covariates): 
\begin{equation}\label{eqn3}
	\forall x \in B,i\in\mathcal{I}_{B}:|\hat{f}_t^i(B) - f^i(x)| \leq \sqrt{\frac{\log(2K/\delta)}{m_t(B,i)}}+2\lambda r_t.
\end{equation}

    Note that the first term on the RHS above contains $m_t(B,i)$ while the second term $r_t$ is chosen according to $n_{r_t}(X_t)$. Following the discussion of Remark~\ref{rmk:book-keeping}, we will relate these two terms by arguing that $m_t(B,i) \gtrsim n_{r_t}(X_t)$. To this end, as we will see, it will suffice to show that the first time $B$ is picked, say at time $s \leq t$, we indeed have $m_s(B,i) \gtrsim n_{r_s}(X_s)$ (where, by assumption, $r_s = r_t$ since $B$ is picked).

    Therefore, suppose $B$ is first visited at round $s \leq t$. We establish in  Lemma~\ref{lem:no-skip} that no descendant of $B$ had been selected by time $s$. 
    Thus, at time $s$, the second situation described in Remark~\ref{rmk:book-keeping} has not happened, i.e.,  
    no arm in ${\cal I}_B$ has been eliminated in a descendant bin of $B$. In other words, all arms in $\cal{I}_B$ are \emph{expected} to have so far been pulled the same amount of time, roughly $n_{r_s}(X_s)/|\cal{I}_B|$. More precisely, as established w.h.p. in  Lemma~\ref{lem:arm-count-randomization} via concentration, we have that $m_s(B,i) \geq n_{r_s}(X_s)/ (4K)$ for all $i\in {\cal I}_B$.
    
    Pulling it all together, observe that $m_t(B,i)\geq m_s(B,i)$, while by the definition of $r_s $ (Line 7 of Algorithm~\ref{algo:one-arm}), we have 
    $\sqrt{1/n_{r_s}(X_s)} \lesssim r_s = r_t$, thus relating $m_t(B,i)$ to $r_t$. 
  
    Formally, plugging back into \eqref{eqn3}, we have:
		\begin{equation}\label{eqn3-1}
			\forall x\in B, i \in {\cal I}_B: |\hat{f}_t^i(B)-f^i(x)| \leq \sqrt{\frac{\log(2K/\delta)}{m_s(B,i)}} + 2\lambda r_t
			\leq \sqrt{\frac{4K\log(2K/\delta)}{n_{r_s}(X_s)}} + 2\lambda r_t 
			\leq 4\lambda r_t, 
		\end{equation}
		establishing that the regression error at round $t$ is at most $4\lambda r_t$ with high probability. 
		
		\paragraph{\textbullet\  Relating regression error to the regret at time $t$.} So far we have established that \eqref{eqn3-1} holds with high probability at all rounds $t$.
		It follows that, given our elimination criteria whereby we only discard an arm $i$ if $\hat{f}_t^{(1)}(B) - \hat{f}_t^i(B) \geq 8\lambda r_t$, the best arm for any $x\in B$ is never discarded (Corollary~\ref{cor:best-arm-candidate}), before or after the unknown shift time $n_P$. In particular for $t > n_P$, it follows by simple triangle inequalities that the regret $|f^{(1)}(X_t) - f^{\pi_t}(X_t)| \lesssim \lambda r_t$ (Corollary~\ref{cor:classification-hard}). 
		
		Furthermore, we can similarly argue that whenever \eqref{eqn3-1} holds at some $x\in B$ with sufficient margin $f^{(1)}(x) - f^{(2)}(x)\gtrsim \lambda r_t$, the regret must be $0$ at time $t$, since then all suboptimal arms would have been eliminated (Corollary~\ref{cor:classification-hard}). This was discussed in Remark~\ref{rmk:margin-adaptation}, and comes in handy below as we integrate the margin parameter $\alpha$ into the regret bound (last bullet point). 
		
		
	
	
	\paragraph{\textbullet\  Relating adaptive $r_t$ to an oracle choice $r_t^{*}$.} Next, we show that at each round $t > n_P$, $r_t$ is of small order in terms of $n_P$ and $t$. To this end, we will show that it cannot lead to worse regression estimates than a suitable global choice $r_t^{*}$ (that in turn can be shown to yield optimal regret). 
   First, Proposition~\ref{prop:rt-minimizer} establishes that $r_t$ satisfies: 
	\begin{equation}\label{eqn:rt-optimal}
	    \lambda r_t\lesssim \underset{r \in \mathcal{R}}\min\,  \{ n_r(X_t)^{-1/2}+\lambda r \}, 
	\end{equation}
	in other words, by \eqref{eqn3} (and the ensuing discussion leading to \eqref{eqn3-1}), $\lambda r_t$ is at most the best variance plus bias terms achievable by \emph{any other choice of level $r$}, in particular, any suitable $r_t^{*}$. 
	
	Therefore, consider $r_t^{*} \doteq r_t(\gamma, n_P)$, the oracle choice of level introduced in Section~\ref{sec:results}. As discussed, it will later become clear that it leads to our main (tight) upper-bound on regret. 
	Formally, we define $r_t^{*}$ as the smallest level $r \in \mathcal{R}$ greater than or equal to
    \[
	    \min\left(\left(\frac{K\log(K/\delta)}{n_P}\right)^{\frac{1}{2+\alpha+d+\gamma}},\left(\frac{K\log(K/\delta)}{\tau}\right)^{\frac{1}{2+\alpha+d}}\right), 
    \]
    where, to further simplify notation, we let $\tau \doteq t-n_P-1$, the time elapsed after round $n_P$. Clearly, as discussed before, the choice $r_t^{*}$ requires knowledge of all parameters $n_P, \gamma, \alpha, d$, and is independent of $X_t$ (it is a global choice at time $t$). 
    Following \eqref{eqn:rt-optimal}, we can relate $r_t$ to $n_{r_t^{*}}(X_t)$ (and $r_t^{*}$). We therefore proceed to bounding $n_{r_t^{*}}(X_t)$ as follows.
    
	
	First, we restrict attention to times $t$ where $X_t$ falls in a bin of sufficient mass (for concentration) at level $r_t^{*}$ (since the probability of $X_t$ not falling in such a bin is negligible). To this end, define the event 
    \[
        A_t=\left\{\max(\tau Q_X(T_{r_t^{*}}(X_t)),n_P P_X(T_{r_t^*}(X_t))) \geq 8\log(1/\delta) \right\}.
    \]
    Under event $A_t$, by a Chernoff bound (Proposition~\ref{prop:optimal-bias-variance}), with probability at least $1-\delta$ (over random rewards, conditioned on $X_t$ and $A_t$):
    \begin{equation}\label{eqn:mass-lower-bound}
        n_{r_t^{*}}(X_t) \gtrsim \max(\tau Q_X(T_{r_t^{*}}(X_t)),n_P P_X(T_{r_t^{*}}(X_t))).  
    \end{equation}
    Then, combining \eqref{eqn:rt-optimal} with \eqref{eqn:mass-lower-bound}, and recalling \eqref{eqn3-1}, gives the following bound on $\lambda r_t$:
	\begin{align}\label{eqn:oracle-bv}
		\lambda r_t \lesssim \sqrt{\frac{K\log(K/\delta)}{n_{r_t^{*}}(X_t)}} + \lambda r_t^{*} &\lesssim \sqrt{\frac{K\log(K/\delta)}{\max(\tau Q_X(T_{r_t^{*}}(X_t)),n_P P_X(T_{r_t^{*}}(X_t)))}} + \lambda r_t^{*} \\
		& \doteq \sigma_t^* + \lambda r_t^*. \nonumber
	\end{align}

Equipped with this upper-bound in terms of $r_t^{*}$, we are now ready to take expectation (over all randomness till time $t$) and properly account for the margin parameter $\alpha$. 
Now, as $\lambda r_t$ controls the regret (either upper-bounds it, or drives it to $0$ under margin as discussed above), so does $\sigma^*_t + \lambda r^*_t$ (where $\sigma^*_t$ is the first term on the R.H.S. of \eqref{eqn:oracle-bv}, and can be viewed as the regression \emph{variance} induced by choosing level $r^*_t$). 

	
	

\paragraph{\textbullet\  Expected regret at time $t$: integrating over margin distribution.}  The arguments below constitute a second departure (on top of the above adaptive choice of $r_t$) from traditional analyses of Lipschitz bandits, due to the fact that we don't constrain the covariate distribution to be near-uniform, nor even stationary (see Remark~\ref{rmk:strong-density-margin} on how the usual uniform conditions remove much of the technicality herein). 
As a result, although $r^*_t$ is independent of location $X_t$, the variance term $\sigma^*_t$ can vary significantly with the random choice of $X_t$ (as different bins at level $r^*_t$ can have significantly different mass). This has to therefore be carefully integrated into the margin distribution across space. Since we consider a fixed round $t$ for now, let $X \doteq X_t$ to simplify notation.

Therefore, let $\delta_f(X) \doteq f^{(1)}(X) - f^{(2)}(X)$ be the margin at $X$.
Next, we recall and condition on all the favorable events discussed thus far.  
First, let $G_t$ be the event that \eqref{eqn3-1} holds and, for rounds $t>n_P$, additionally that \eqref{eqn:mass-lower-bound} holds. Let $F_t \doteq \cap_{s=1}^t G_s$. As a reminder, $F_t$ occurs with high probability so we can safely 
condition on this event in taking expectation. As discussed so far, under event $F_t$, a non-zero regret (at most $r_t$) is incurred at time $t$ \emph{only if} $\delta_f(X) \lesssim r_t$. In other words, we have: 
\begin{align*}
		\mathbb{E}[(f^{(1)}(X) - f^{\pi_t}(X)) \mathbbm{1}\{F_t\}] &\leq  \mathbb{E}[(f^{(1)}(X) - f^{\pi_t}(X)) \mathbbm{1}\{ (f^{(1)}(X) - f^{\pi_t}(X)) \vee \delta_f(X) \lesssim \sigma_t^{*} + r_t^*\}]\\
		&\lesssim \mathbb{E}[ r_t^* \cdot \mathbbm{1}\{0 < \delta_f(X) \lesssim r_t^*\}] + \mathbb{E} [\sigma_t^{*}\cdot\mathbbm{1}\{0 < \delta_f(X) \lesssim \sigma_t^{*}\}]. \numberthis \label{eqn:v+b}
\end{align*}
where, in the second inequality above, we use $\mathbbm{1}\{x\leq a+b\} \leq \mathbbm{1}\{x\leq 2a\} + \mathbbm{1}\{x\leq 2b\}$.
By the margin condition (Definition~\ref{assumption:margin}), the first term on the R.H.S. of \eqref{eqn:v+b} is of order at most $(r_t^*)^{1+\alpha}$. The main technicality left is to show that the second term is of the same order, despite the instability in $\sigma^*_t$ over the spatial choice of $x$. For this purpose, first consider the following two random quantities in the definition of $\sigma^*_t$:
\[
    a(X) \doteq Q_X(T_{r_t^{*}}(X))^{-1/2}, \text{ and } 
    b(X) \doteq P_X(T_{r_t^{*}}(X))^{-1/2}.
\]
Notice that $\sigma^*_t$ is of order less than either of $a(X)\cdot \tau^{-1/2}$ and $b(X)\cdot n_P^{-1/2}$. As these are unstable quantities (unlike in the uniform distribution case), the main idea is to break integration into two terms, below and above a threshold $\sqrt{\epsilon}$, for some arbitrary $\epsilon >0$, which will then be optimized over. Next we present the main argument in the case of $a(X)$, which yields a first bound on regret in terms of elapsed time $\tau$, while adapting the same line of arguments to $b(X)$ yields a second bound in terms of unknown $n_P$ and $\gamma$; both of these bounds are then combined into Theorem \ref{thm:adaptive-algo-one-arm}. 

Thus, for some $\epsilon>0$, decompose the second term on the R.H.S. of \eqref{eqn:v+b} as:
\begin{equation}\label{eqn:variance-decomposition-body}
	\mathbb{E}\left[ \sigma_t^{*} \cdot \mathbbm{1}\left\{0< \delta_f(X) \lesssim \sigma_t^{*}\right\} \cdot \left( \mathbbm{1}\{a(X) < \sqrt{\epsilon}\} + \mathbbm{1}\{a(X) \geq \sqrt{\epsilon}\}\right)\right]. \nonumber
\end{equation}
We start by handling the case where $a(X) < \sqrt{\epsilon}$. By definition, as mentioned above, we have: 
\begin{equation}\label{eqn:sigma-a}
    \sigma_t^{*} \leq a(X) \cdot \sqrt{\frac{K\log(K/\delta)}{\tau}}. 
\end{equation}
Combining the above with the margin condition (Definition~\ref{assumption:margin}) gives
\begin{align}\label{eqn:vbound1}
	\mathbb{E}\  [\sigma_t^{*} \cdot \mathbbm{1}\left\{0< \delta_f(X) \lesssim \sigma_t^{*}, a(X) < \sqrt{\epsilon}\right\}]
	&\lesssim \sqrt{\frac{\epsilon K \log(K/\delta)}{\tau}} \cdot 
	\mathbb{E}\ \mathbbm{1}\left\{0< \delta_f(X) \leq \sqrt{\frac{\epsilon K \log(K/\delta)}{\tau}}\right\}\nonumber \\
	&\lesssim \left (\sqrt{\frac{\epsilon K \log(K/\delta)}{\tau}}\right)^{\alpha+1}.
\end{align}
Next, in the case where $a(X) \geq \sqrt{\epsilon}$, we have again by \eqref{eqn:sigma-a} that: 
\begin{align} 
\mathbb{E}\, [\sigma_t^{*} \cdot \mathbbm{1}\left\{ 0 < \delta_f(X) \lesssim \sigma_t^{*} \right\} \mathbbm{1}\{a(X) \geq \sqrt{\epsilon}\}] &\leq 
\sqrt{\frac{K \log(K/\delta)}{\tau}} \cdot \mathbb{E}\, [a(X)  \mathbbm{1}\{a(X) \geq \sqrt{\epsilon}\}] \nonumber \\
&\lesssim  \sqrt{\frac{K \log(K/\delta)}{\tau}} \cdot\frac{(r_t^{*})^{-d}}{\sqrt{\epsilon}}, \label{eqn:a>eps}
\end{align} 
where the last inequality is a technical result derived in Proposition~\ref{prop:tail-bound}, relying on the fact that $\mathbb{E}\ [a^2(X)] \lesssim (r^*_t)^{-d}$ (this latter fact is relatively standard under \emph{box cover dimension} $d$ -- Definition \ref{defn:bcn}).

	Finally, \eqref{eqn:vbound1} + \eqref{eqn:a>eps} is optimized by setting $\epsilon \propto \tau^{\frac{\alpha+d}{2+\alpha+d}}$, showing that the second term of \eqref{eqn:v+b} is $\tilde{O}(\tau^{-\frac{\alpha+1}{2+\alpha+d}})$.	We can similarly show an upper-bound of  $\tilde{O}(n_P^{-\frac{\alpha+1}{2+\alpha+d+\gamma}})$ by decomposing expectation in terms of $b(X)$, where analogously, a term of the form 
	$\mathbb{E}\, [b(X)  \mathbbm{1}\{b(X) \geq \sqrt{\epsilon}\}]$ is bounded via the technical result of Proposition~\ref{prop:tail-bound}, and using the (less standard but similarly obtained) fact that 
	$\mathbb{E}\ [b^2(X)] \lesssim (r^*_t)^{-(\gamma + d)}$.

    Thus, we have that the regret at round $t$, under event $F_t$, is $\tilde{O}(\min(\tau^{-\frac{\alpha+1}{2+\alpha+d}},n_P^{-\frac{\alpha+1}{2+\alpha+d+\gamma}})) \propto (r_t^{*})^{1+\alpha}$. Summing over time $t$ then yields the bound of Theorem \ref{thm:adaptive-algo-one-arm}.
    

The case of multiple shifts is handled similarly by properly bounding $\mathbb{E}[n_{r_t^{*}}(x)]$ in the derivation of \eqref{eqn:mass-lower-bound} to incorporate multiple previous distributions $P_X^j$ for $j\in[N]$ (see Appendix~\ref{app:multiple-shifts}).

\begin{rmk}[Contrast with the strong density case]\label{rmk:strong-density-margin}
    We note that under the usual {\em strong density} condition on $Q_X$ (see Remark~\ref{rem:strongdensityand margin} for a definition), bringing $\alpha$ into the regret bound is more direct. Namely, due to the near-uniformity of $Q_X$ under strong density, the optimal regression rate is of the same order $r^{**}_t \propto \tau^{-1/(2+d)}$ spatially for all possible $X_t$. As a consequence, since nonzero regret only happens when the margin $\delta_f(X_t) \lesssim r^{**}_t$ (a non-random quantity), we almost immediately have that the regret at time $t$ -- upon optimal regression -- is upper-bounded by 
    $$r^{**}_t \cdot \mathbb{E}\ \mathbbm{1}\{\delta_f(X_t) \lesssim r^{**}_t \}\lesssim (r^{**}_t)^{\alpha +1}.$$
     On the other hand, we note the following interesting but nuanced point: even in our current setting without the strong density assumption, it can be shown that the optimal choice of level w.r.t. $L_2$ \emph{regression} (i.e. in bounding $\mathbb{E}\ (\hat f^i(X_t) - f^i(X_t))^2$) remains $r^{**}_t$, rather than than $r^*_t$ as defined in our case. In particular, our larger (oracle) choice of level $r^*_t > r^{**}_t$ -- is required to mitigate variance in lower-density regions over space -- and implies that the bandit problem differs in this case more significantly from the underlying regression problem (enough that the optimal regression choice of $r^{**}_t$ is now suboptimal for bandits). 
    
    Finally, the case of strong density with unknown $d, \gamma$ and $\alpha$ is handled in Theorems \ref{thm:adaptive-algo-one-arm-dm} and \ref{thm:multiple-shifts-dm}.

\end{rmk}

\begin{rmk}[Beyond Covariate Shift]
    Under more severe shifts in the rewards, for instance ones where previously discarded arms now become optimal, our procedure will be sub-optimal as it has no mechanism to detect such changes. However, the regret rates of Theorem~\ref{thm:adaptive-algo-one-arm} can still be attained by Algorithm~\ref{algo:one-arm} under mild changes in the reward function $f_P\to f_Q$. Let $\epsilon(x) \doteq \max_{i\in[K]} |f_P^i(x) - f_Q^i(x)|$, which roughly quantifies the amount of shift in the rewards. Intuitively, if $\epsilon(x) \ll \sqrt{1/n_P(r_{n_P})}$ for all $x \in \mathcal{X}_Q$ where $n_P(r)$ is the covariate count in bin $T_r(x)$ from $P_X$, then the amount of change is so small that the new best arm under $Q$, $i_Q^*(x)$, must have been retained as a candidate arm at time $n_P$. Moreover, we can still accurately estimate $f^i$ using $\hat{f}_t^i$ for any candidate arm $i$ \emph{as if there was no shift}. To see this, consider decomposing $\hat{f}_t^i = \alpha\cdot \hat{f}_P^i + (1-\alpha)\cdot \hat{f}_Q^i$ as a weighted sum of oracle estimates $\hat{f}_P^i$ (resp. $\hat{f}_Q^i$) using only the observed data from $P$ (resp. $Q$) with weight $\alpha \doteq n_P(r_t)/(n_P(r_t)+n_Q(r_t))$ ($n_Q(r)$ defined analogously):
	\begin{align*}
		\max_{i\in {\cal I}}|\hat{f}_t^i(x) - f_Q^i(x)| &\lesssim \alpha\cdot |\hat{f}_P^i(x) - f_P^i(x)| + \alpha\cdot \epsilon(x) + (1-\alpha)\cdot |\hat{f}_Q^i(x) - f_Q^i(x)|\\
		&\lesssim r_t + \sqrt{\frac{1}{n_P(r_t)+n_Q(r_t)}}\lesssim r_t.
	\end{align*}
	On the other hand, a different approach is required to handle the more challenging setting of severe shifts in the rewards.
\end{rmk}

\section*{Acknowledgements}
Samory Kpotufe thanks Google AI Princeton, and the Institute for Advanced Study at Princeton for hosting him during part of this project. He also acknowledges support from NSF:CPS:Medium:1953740.

\bibliography{main.bib}

\begin{thebibliography}{32}
\providecommand{\natexlab}[1]{#1}
\providecommand{\url}[1]{\texttt{#1}}
\expandafter\ifx\csname urlstyle\endcsname\relax
  \providecommand{\doi}[1]{doi: #1}\else
  \providecommand{\doi}{doi: \begingroup \urlstyle{rm}\Url}\fi

\bibitem[Audibert and Tsybakov(2007)]{audibert-tsybakov}
Jean-Yves Audibert and Alexander~B Tsybakov.
\newblock Fast learning rates for plug-in classifiers.
\newblock \emph{The Annals of Statistics}, 35\penalty0 (2):\penalty0 608--633,
  2007.

\bibitem[Auer and Chiang(2016)]{auer2016algorithm}
Peter Auer and Chao-Kai Chiang.
\newblock An algorithm with nearly optimal pseudo-regret for both stochastic
  and adversarial bandits.
\newblock In \emph{Conference on Learning Theory}, pages 116--120, 2016.

\bibitem[Azar et~al.(2013)Azar, Lazaric, and Brunskill]{azar2013}
Mohammad~Gheshlaghi Azar, Alessandro Lazaric, and Emma Brunskill.
\newblock Sequential transfer in multi-armed bandit with finite set of models.
\newblock In \emph{Advances in neural information processing systems}, 2013.

\bibitem[Ben-David and Urner(2012)]{ben2012hardness}
Shai Ben-David and Ruth Urner.
\newblock On the hardness of domain adaptation and the utility of unlabeled
  target samples.
\newblock In \emph{International Conference on Algorithmic Learning Theory},
  pages 139--153, 2012.

\bibitem[Besbes et~al.(2014)Besbes, Gur, and Zeevi]{besbes2014nonstationary}
Omar Besbes, Yonatan Gur, and Assaf Zeevi.
\newblock Stochastic multi-armed-bandit problem with non-stationary rewards.
\newblock In \emph{Advances in neural information processing systems}, 2014.

\bibitem[Bubeck and Cesa-Bianchi(2012)]{bubeck2012regret}
S{\'e}bastien Bubeck and Nicolo Cesa-Bianchi.
\newblock Regret analysis of stochastic and nonstochastic multi-armed bandit
  problems.
\newblock \emph{arXiv preprint arXiv:1204.5721}, 2012.

\bibitem[Cesa-Bianchi et~al.(2004)Cesa-Bianchi, Conconi, and
  Gentile]{cesa-bianchi2004online}
Nicol\'{o} Cesa-Bianchi, Alex Conconi, and Claudio Gentile.
\newblock On the generalization ability of on-line learning algorithms.
\newblock \emph{Information Theory, IEEE Transactions}, 50\penalty0
  (9):\penalty0 2050--2057, 2004.

\bibitem[Chen et~al.(2019)Chen, Lee, Luo, and Wei]{chen2019nonstationary}
Yifang Chen, Chung-Wei Lee, Haipeng Luo, and Chen-Yu Wei.
\newblock A new algorithm for non-stationary contextual bandits: efficient,
  optimal, and parameter-free.
\newblock In \emph{32nd Annual Conference on Learning Theory}, 2019.

\bibitem[Chen et~al.(2020)Chen, Luo, Ma, and Zhang]{chen-luo-ma-zhang}
Yining Chen, Haipeng Luo, Tengyu Ma, and Chicheng Zhang.
\newblock Active online domain adaptation.
\newblock \emph{arXiv preprint arXiv:2006.14481}, 2020.

\bibitem[Chi~Cheung et~al.(2019)Chi~Cheung, Simchi-Levi, and
  Zhu]{cheung2019drift}
Wang Chi~Cheung, David Simchi-Levi, and Ruihao Zhu.
\newblock Hedging the drift: learning to optimize under non-stationarity.
\newblock In \emph{Proceedings of the 22nd International Conference on
  Artificial Intelligence and Statistics}, 2019.

\bibitem[Clarkson(2006)]{clarkson2006nearest}
Kenneth~L Clarkson.
\newblock Nearest-neighbor searching and metric space dimensions.
\newblock \emph{Nearest-neighbor methods for learning and vision: theory and
  practice}, pages 15--59, 2006.

\bibitem[Cortes et~al.(2008)Cortes, Mohri, Riley, and
  Rostamizadeh]{cortes2008sample}
Corinna Cortes, Mehryar Mohri, Michael Riley, and Afshin Rostamizadeh.
\newblock Sample selection bias correction theory.
\newblock In \emph{International conference on algorithmic learning theory},
  pages 38--53. Springer, 2008.

\bibitem[Gretton et~al.(2009)Gretton, Smola, Huang, Schmittfull, Borgwardt, and
  Sch{\"o}lkopf]{gretton2009covariate}
Arthur Gretton, Alex Smola, Jiayuan Huang, Marcel Schmittfull, Karsten
  Borgwardt, and Bernhard Sch{\"o}lkopf.
\newblock Covariate shift by kernel mean matching.
\newblock \emph{Dataset shift in machine learning}, 3\penalty0 (4):\penalty0 5,
  2009.

\bibitem[Guan and Jiang(2018)]{guan-jiang}
Melody~Y Guan and Heinrich Jiang.
\newblock Nonparametric stochastic contextual bandits.
\newblock \emph{AAAI}, 2018.

\bibitem[Hariri et~al.(2015)Hariri, Mobasher, and Burke]{hariri2015adapting}
Negar Hariri, Bamshad Mobasher, and Robin Burke.
\newblock Adapting to user preference changes in interactive recommendation.
\newblock In \emph{Twenty-Fourth International Joint Conference on Artificial
  Intelligence}, 2015.

\bibitem[Hazan and Megiddo(2007)]{hazan2007online}
Elad Hazan and Nimrod Megiddo.
\newblock Online learning with prior knowledge.
\newblock In \emph{International Conference on Computational Learning Theory},
  pages 499--513. Springer, 2007.

\bibitem[Karnin and Anava(2016)]{karnin2016multi}
Zohar~S Karnin and Oren Anava.
\newblock Multi-armed bandits: Competing with optimal sequences.
\newblock In \emph{Advances in Neural Information Processing Systems}, pages
  199--207, 2016.

\bibitem[Kpotufe and Dasgupta(2012)]{kpotufe2012tree}
Samory Kpotufe and Sanjoy Dasgupta.
\newblock A tree-based regressor that adapts to intrinsic dimension.
\newblock \emph{Journal of Computer and System Sciences}, 78\penalty0
  (5):\penalty0 1496--1515, 2012.

\bibitem[Kpotufe and Martinet(2018)]{kpotufe-martinet}
Samory Kpotufe and Guillaume Martinet.
\newblock Marginal singularity, and the benefits of labels in covariate-shift.
\newblock \emph{COLT}, 2018.

\bibitem[Langford and Zhang(2008)]{langford2008epoch}
John Langford and Tong Zhang.
\newblock The epoch-greedy algorithm for multi-armed bandits with side
  information.
\newblock In \emph{Advances in neural information processing systems}, pages
  817--824, 2008.

\bibitem[Liu et~al.(2018)Liu, Lee, and Shroff]{liu2018change}
Fang Liu, Joohyun Lee, and Ness Shroff.
\newblock A change-detection based framework for piecewise-stationary
  multi-armed bandit problem.
\newblock In \emph{Thirty-Second AAAI Conference on Artificial Intelligence},
  2018.

\bibitem[Luo et~al.(2018)Luo, Wei, Agarwal, and Langford]{luo2017efficient}
Haipeng Luo, Chen-Yu Wei, Alekh Agarwal, and John Langford.
\newblock Efficient contextual bandits in non-stationary worlds.
\newblock In \emph{31st Annual Conference on Learning Theory (COLT)}, 2018.

\bibitem[Perchet and Rigollet(2013)]{perchet-rigollet}
Vianney Perchet and Philippe Rigollet.
\newblock The multi-armed bandit problem with covariates.
\newblock \emph{The Annals of Statistics}, 41\penalty0 (2):\penalty0 693–721,
  2013.

\bibitem[Rakhlin and Sridharan(2016)]{rakhlin2016bistro}
Alexander Rakhlin and Karthik Sridharan.
\newblock Bistro: An efficient relaxation-based method for contextual bandits.
\newblock In \emph{ICML}, pages 1977--1985, 2016.

\bibitem[Reeve et~al.(2018)Reeve, Mellor, and Brown]{reeve}
Henry W.~J. Reeve, Joe Mellor, and Gavin Brown.
\newblock The $k$-nearest neighbour ucb algorithm for multi-armed bandits with
  covariates.
\newblock \emph{JMLR}, 2018.

\bibitem[Rigollet and Zeevi(2010)]{rigollet-zeevi}
Phillipe Rigollet and Assaf Zeevi.
\newblock Nonparametric bandits with covariates.
\newblock \emph{COLT}, 2010.

\bibitem[Scott and Nowak(2006)]{scott2006minimax}
Clayton Scott and Robert~D Nowak.
\newblock Minimax-optimal classification with dyadic decision trees.
\newblock \emph{IEEE transactions on information theory}, 52\penalty0
  (4):\penalty0 1335--1353, 2006.

\bibitem[Slivkins(2014)]{slivkins2014contextual}
Aleksandrs Slivkins.
\newblock Contextual bandits with similarity information.
\newblock \emph{The Journal of Machine Learning Research}, 15\penalty0
  (1):\penalty0 2533--2568, 2014.

\bibitem[Sugiyama et~al.(2008)Sugiyama, Nakajima, Kashima, Buenau, and
  Kawanabe]{sugiyama2008direct}
Masashi Sugiyama, Shinichi Nakajima, Hisashi Kashima, Paul~V Buenau, and
  Motoaki Kawanabe.
\newblock Direct importance estimation with model selection and its application
  to covariate shift adaptation.
\newblock In \emph{Advances in neural information processing systems}, pages
  1433--1440, 2008.

\bibitem[Tsybakov et~al.(2004)]{tsybakov2004optimal}
Alexander~B Tsybakov et~al.
\newblock Optimal aggregation of classifiers in statistical learning.
\newblock \emph{The Annals of Statistics}, 32\penalty0 (1):\penalty0 135--166,
  2004.

\bibitem[Wu et~al.(2018)Wu, Iyer, and Wang]{wu2018learning}
Qingyun Wu, Naveen Iyer, and Hongning Wang.
\newblock Learning contextual bandits in a non-stationary environment.
\newblock In \emph{The 41st International ACM SIGIR Conference on Research \&
  Development in Information Retrieval}, pages 495--504, 2018.

\bibitem[Yang et~al.(2002)Yang, Zhu, et~al.]{yang2002randomized}
Yuhong Yang, Dan Zhu, et~al.
\newblock Randomized allocation with nonparametric estimation for a multi-armed
  bandit problem with covariates.
\newblock \emph{The Annals of Statistics}, 30\penalty0 (1):\penalty0 100--121,
  2002.

\end{thebibliography}

\newpage
\appendix

	\section{Additional Experiments}
	\label{app:experiments}

\begin{figure}[H]
    \centering
    \includegraphics[height=4cm]{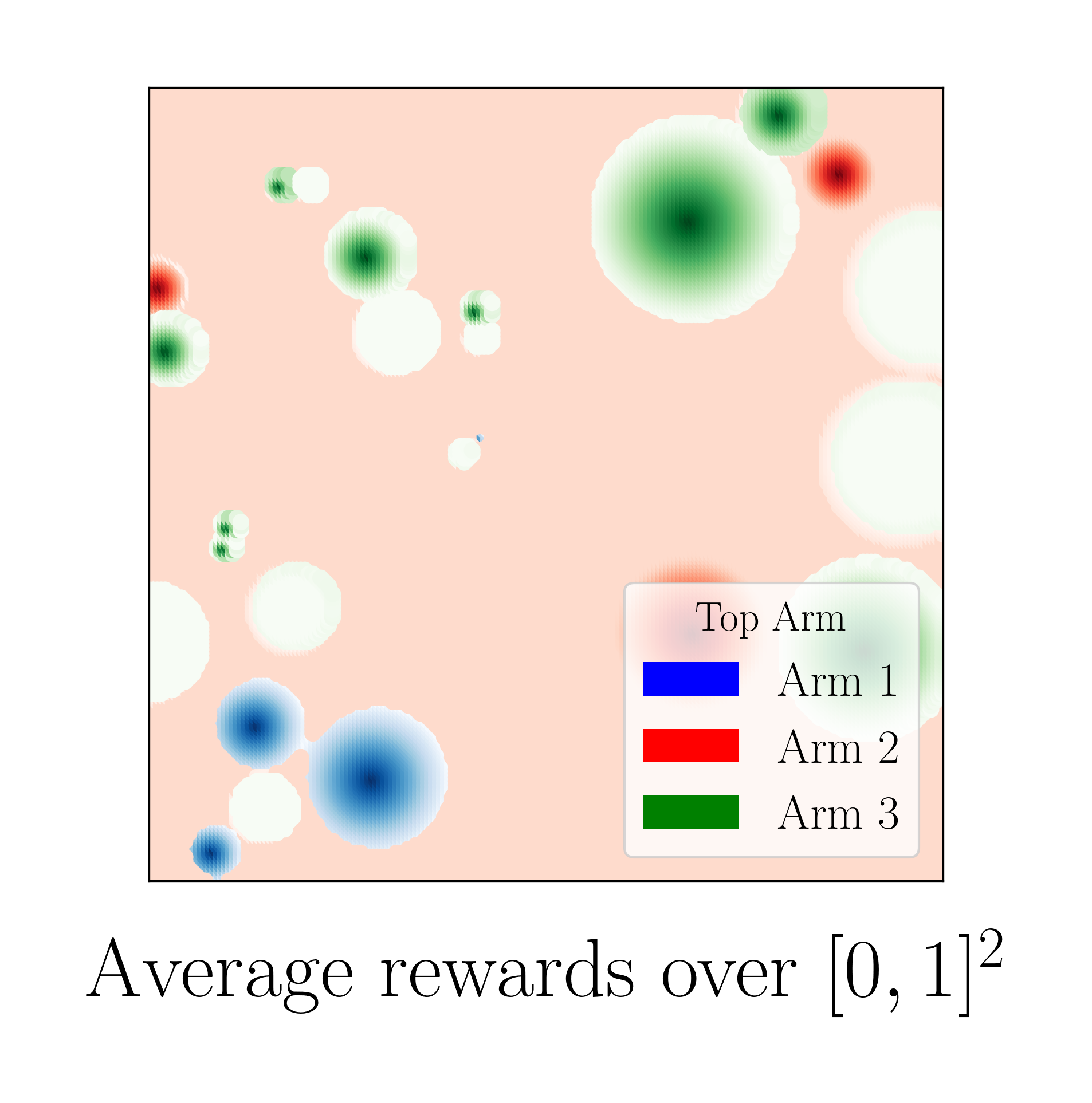} \hspace{0.3cm}
    \includegraphics[height=4cm]{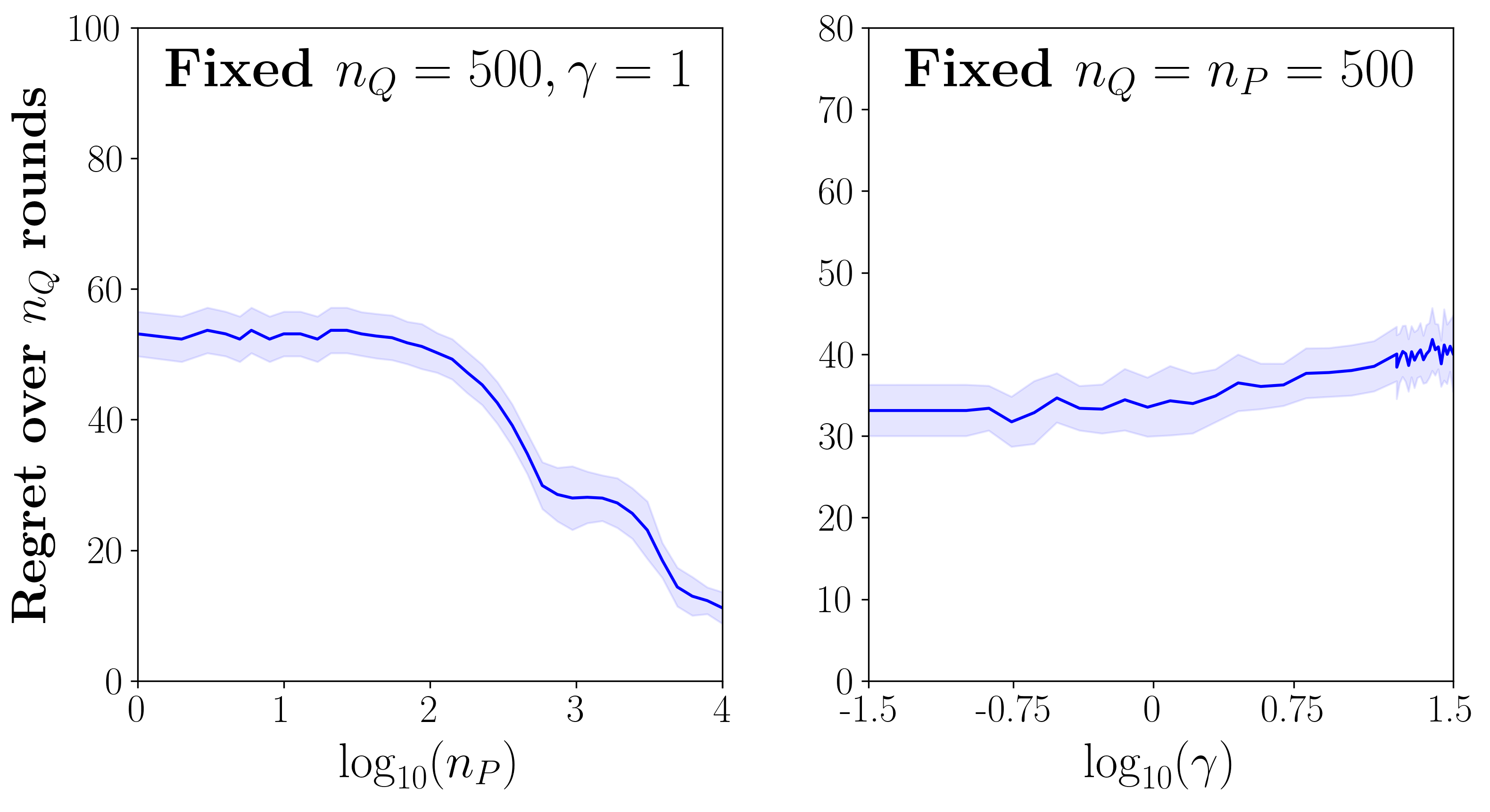} 

    \caption{\small 
   Additional simulations with a different choice of bump centers than in Figure~\ref{fig:simulations}: here, the bump centers are chosen uniformly in $[0,1]^2$. Under this setup, we see that the behavior of the regret in the center and right plots are similar to that of Figure~\ref{fig:simulations}.
    }
    \label{fig:sim-uniform}
\end{figure}

For all our experiments, we fix a covariate space ${\cal X}=[0,1]^2$. Our terminal covariate distribution is $Q_X\sim {\cal U}([0,1]^2)$, the uniform on $[0,1]^2$, and our initial covariate distribution $P_X$ has density $\propto \|x\|_2^{\gamma}$ so that $P_X,Q_X$ satisfy Definition~\ref{assumption:transfer-exponent} with transfer exponent $\gamma$. 

The reward function common to both $P$ and $Q$ is constructed as the sum of $25$ bump functions, each with a circular support disjoint from the other bumps, in the following manner:
\begin{enumerate}
	\item The bump centers $\{z_k\}_{k=1}^{25}$ are first randomly sampled from the Gaussian ${\cal N}\left( (0.5,0.5), 0.5 \cdot \text{Id}_2\right)$ in Figure~\ref{fig:simulations} and from the uniform ${\cal U}([0,1]^2)$ in Figure~\ref{fig:sim-uniform}.
	\item The bumps' radii $\{r_k\}_{k=1}^{25}$ are then chosen in a random order to maximize the bump areas.
	\item Then, for each of the $K=3$ arms, we determined the sign of each of the $25$ bumps randomly and independently. 
	\item To introduce additional heterogeneity in the top arm identity $\pi^*(x)$, each reward function $f^i$ was further raised or lowered by a randomly selected height in the range $[-0.3,0.3]$, in the area outside of the bumps.
\end{enumerate}
The fourth step determines a unique top arm (Arm 1 in Figure~\ref{fig:simulations} and Arm 2 in Figure~\ref{fig:sim-uniform}) in the region outside of the bumps. To summarize the above, the reward functions $f^i$ can be written as
\[
	f^i(x) \propto \sum_{k} \omega_{i,k} (1-\|x-z_k\|_2/r_k)_+,
\]
where $\omega_{i,k}$ are independent random Rademacher variables. Furthermore, Gaussian noise was added to each $f^i$ to produce the observed rewards $Y^i$ according to $Y^i=f^i(X)+{\cal N}(0,0.05)$ for each $i\in[K]$.

Now, having determined $f$, we considered a range of different values for the parameters $n_Q,n_P,\gamma$, shown in the horizontal axes of Figures~\ref{fig:simulations} and \ref{fig:sim-uniform}. Then,  Algorithm~\ref{algo:one-arm} was run on $20$ simulations of data $({\bf X}_n, {\bf Y}_n)$ for each choice of parameters, so that the plots show the mean and standard deviation of the regret ${\bf R}_n^Q$ across $20$ trials.

 The first plot in each of Figures~\ref{fig:simulations} and \ref{fig:sim-uniform} exhibits the guarantee of increasing past experience $n_P$ improving the regret, for fixed $n_Q,\gamma$. The second plot in each figure shows the effect of increasing $\gamma$ worsening the regret, for fixed $n_P,n_Q$. Together, the two plots in each of Figure~\ref{fig:simulations} and \ref{fig:sim-uniform} demonstrate the guarantees of Theorem~\ref{thm:adaptive-algo-one-arm}.

\begin{wrapfigure}[6]{r}{0.21\textwidth}
\label{fig:sim-adversarial}
    \centering
    \includegraphics[width=0.21\textwidth]{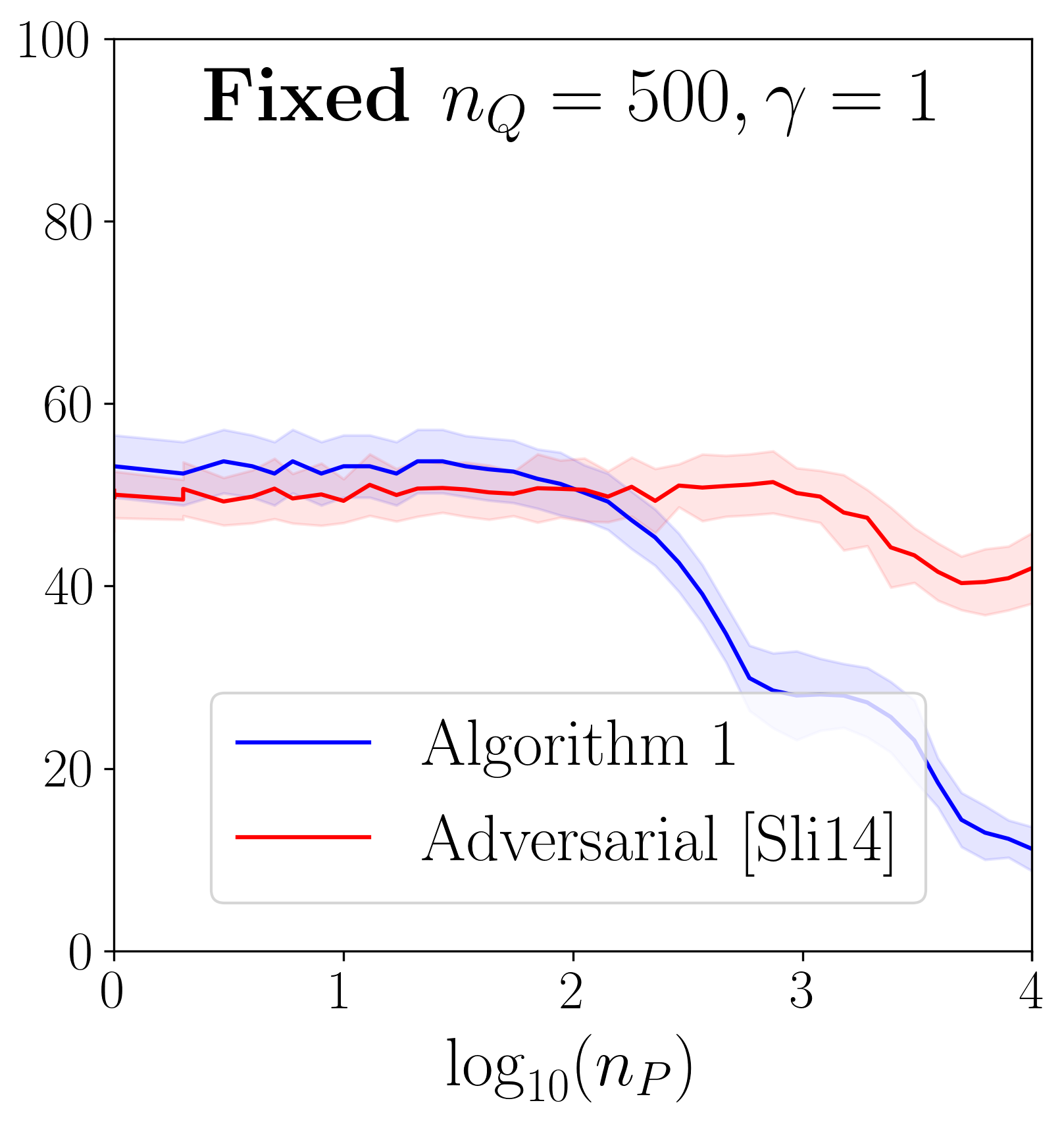} 
\end{wrapfigure}

We also compare Algorithm 1 to the  \emph{adversarial} nonparametric contextual bandits algorithm ``ContextualBandit'' of \citet{slivkins2014contextual} under the same setting as Figure~\ref{fig:sim-uniform}. We instantiate their algorithm using the static bandits algorithm Exp3 with a common learning rate $\eta = \sqrt{\log K/ TK}$. These simulations show that adversarial design can be too conservative in leveraging past experience: for small $n_P$ we see no benefit for either procedure, while our procedure better leverages the past for sufficiently large $n_P$.
\vspace{2em}

	\section{Proof of Theorem~\ref{thm:adaptive-algo-one-arm}}
	\label{app:upper-bound-single}
	
	Throughout the proof, $c_0,c_1,c_2,\ldots$ will denote positive constants not depending on $t,n_P,n_Q,K$.
	
	First, it is straightforward to verify that the criterion for choosing $r_t$ on Line 7 of Algorithm~\ref{algo:one-arm} is well-defined for $t>\lceil 8K\log(K/\delta)\rceil$, i.e. $r=1$ satisfies the minimization criterion.

	\paragraph{Bias-Variance Bound.}	
	
    This first proposition establishes a standard bias-variance bound on the error of a regression function estimate $|\hat{f}_t^i(B) - f^i(x)|$ for a bin $B$, a round $t$, an arm $i\in\mathcal{I}_B$, and a covariate $x\in B$. 
    
    We use $Z_t$ to denote the randomness of Algorithm~\ref{algo:one-arm} at round $t$ in choosing the particular arm $\pi_t$ to play. 
	
	\begin{lemma}\label{lem:bias-variance}
		Consider any round $t>8K\log(K/\delta)$ with observed covariate $X_t$, and fix any bin $B$ containing $X_t$. Consider the estimate 
		$\hat{f}_t^i(B)$ as in Definition~\ref{defn:arm-pull-counts}, and let $m_t(B,i)$ be defined therein (i.e., the number of times arm $i$ is pulled in $B$ by time $t$). We then have at round $t$, that 
		with probability at least $1-\delta$ with respect to the conditional distribution $\textbf{Y}_{t-1}|\textbf{X}_{t},\{Z_s\}_{s < t}$: 
		\begin{equation}\label{eqn:bias-variance-app}
			\forall x \in B,i\in[K]:|\hat{f}_t^i(B) - f^i(x)|\leq  \sqrt{\frac{\log(2K/\delta)}{m_t(B,i)}}+2\lambda r.
		\end{equation}
	\end{lemma}
	
	\begin{proof}
		Fix bin $B$ and let $r\in\mathcal{R}$ be its side length. If $m_t(B,i)=0$, then the desired bound is vacuously true. So, suppose $m_t(B,i) >0$. Now, recall from Definition~\ref{defn:arm-pull-counts},
	\[
		\hat{f}_t^i(B) \doteq \frac{1}{m_t(B,i)}\sum_{X_s\in B,s\leq t-1,\pi_s = i} Y_s^i.
	\]	
	For the sake of introducing a bias term of $\hat{f}_t^i(B)$, define
	\begin{align*}
		\tilde{f}_t^i(B) &\doteq\mathbb{E}_{\textbf{Y}_{t-1}|\textbf{X}_{t-1}} [\hat{f}_t^i(B)]=\frac{1}{m_t(B,i)}\sum_{X_s\in B,s\leq t-1, \pi_s = i}\mathbb{E}\, [Y_s^{i}|X_s].
	\end{align*}
	Triangle inequality then yields
\[
	|\hat{f}_t^i(B) - f^{i}(x)|\leq |\hat{f}_t^i(B) - \tilde{f}_t^i(B)| + |\tilde{f}_t^i(B) - f^{i}(x)|.
\]
The second term on the R.H.S. above is at most $2\lambda r$ by the Lipschitz assumption (Assumption~\ref{assumption:lipschitz}). Now, fix the values of $\{m_t(B,i)\}_{i\in[K]}$. By Hoeffding inequality and union bound, the first term on the R.H.S. above satisfies with probability at least $1-\delta$ w.r.t. the distribution of $\textbf{Y}_{t-1}|\textbf{X}_{t},\{m_t(B,i)\}_{i\in[K]}$,
\begin{align*}
	\forall i\in [K]:|\hat{f}_t^i(B) - \tilde{f}_t^i(B)| \leq  \sqrt{\frac{\log(2K/\delta)}{m_t(B,i)}}. 
\end{align*}
In fact, by the tower property, the above holds with probability at least $1-\delta$ w.r.t. the distribution of ${\bf Y}_{t-1}|{\bf X}_t.\{Z_s\}_{s<t}$.
\end{proof}

    In particular, the bound of \eqref{eqn:bias-variance-app} holds for $r = r_t$ and $B = T_{r_t}(X_t)$. At $r = r_t$, the first term on the R.H.S. of \eqref{eqn:bias-variance-app} depends on the arm-pull count $m_t(B,i)$, while the second term is chosen according to $n_{r_t}(X_t)$ (Line 7 of Algorithm~\ref{algo:one-arm}). We next relate the counts $m_t(B,i)$ and $n_{r_t}(X_t)$ to show that the R.H.S. of \eqref{eqn:bias-variance-app} is $O(r_t)$.

	\paragraph{Relating Arm-Pull Counts $m_t(B,i)$ to Covariate Counts $n_r(X_t)$.}

	
	
	We start by showing that the chosen level $r_t$ and bin $B \doteq T_{r_t}(X_t)$ cannot ``skip'' levels in the sense that $B$ is selected before any of its descendants. Thus, each arm in ${\cal I}_B$ had a chance of being pulled once (by the randomization in Line 12 of Algorithm~\ref{algo:one-arm}) for each of the $n_{r_t}(X_t)$ covariates.
	
	This will ensure that each arm in ${\cal I}_B$ has been played roughly the same number of times, so that $m_t(B,i) \gtrsim n_{r_t}(X_t)/K$ (Lemma~\ref{lem:arm-count-randomization} further below), thus relating the two counts.

	\begin{lemma}\label{lem:no-skip}
		Fix a bin $B$, and suppose $t$ is the first round that $B$ is selected. Then, no descendant bin of $B$ was selected in a round previous to $t$.
	\end{lemma}
	
	\begin{proof}
		For contradiction, suppose a descendant $B'\subset B$ was selected at round $s<t$. W.L.O.G., let $s$ be the first round that any descendant of $B$ was selected. Then, by the criterion for choosing $r_s$ on Line 7 of Algorithm~\ref{algo:one-arm}, we have:
		\[
			r_s \geq \sqrt{\frac{8 K\log(K/\delta)}{n_{r_s}(X_s)}} \implies n_{r_s}(X_s) \geq \frac{8K\log(K/\delta)}{ r_s^2}.
		\]
		However, since $r_t \geq 2r_s$, we have
		\[
		    \frac{8 K\log(K/\delta)}{ r_t^2} \leq \frac{8 K\log(K/\delta)}{ (2r_s)^2} \leq \left\lceil \frac{n_{r_s}(X_s)}{4}\right\rceil < n_{r_s}(X_s).  
		\]
		The above implies that there was an earlier round $s' < s$ and a covariate $X_{s'} \in B'$ such that
		\[
			n_{r_{s}}(X_{s'}) \geq \frac{8 K\log(K/\delta)}{ r_t^2} \implies r_t \geq \sqrt{\frac{8 K\log(K/\delta)}{n_{r_s}(X_{s'})}} \geq \sqrt{\frac{8 K\log(K/\delta)}{n_{r_t}(X_{s'})}}.
		\]	
		According to the minimization criterion for choosing $r_{s'}$ on Line 7 of Algorithm~\ref{algo:one-arm}, we must have $r_{s'} \leq r_t$, a contradiction to either $t$ being the first round that $B$ is selected or $s$ being the first round that a descendant of $B$ was selected.
		
	\end{proof}

	We next use a concentration argument to show that $m_t(B,i) \gtrsim n_{r_t}(X_t)/K$. 
	
	\begin{lemma}\label{lem:arm-count-randomization}
		Fix a round $t > \lceil 8K \log(K/\delta)\rceil$ with observed covariate $X_t$ and selected bin $B$. Suppose that $t$ is the first round that $B$ is selected. Then, with probability at least $1-\delta$ with respect to the distribution of $\textbf{Y}_{t-1},\{Z_s\}_{s<t}|\textbf{X}_t$, we have
		\begin{align*}
			\forall i \in \mathcal{I}_B: m_t(B,i) \geq \frac{n_{r_t}(X_t)}{4K}.
		\end{align*}
	\end{lemma}
	
	\begin{proof}
	Fix the values of ${\bf X}_{t},{\bf Y}_{t-1},{\cal I}_B$ and fix some $i\in{\cal I}_B$. Recall: 
		\[
			m_t(B,i) = \sum_{X_s\in B, s\leq t-1} \mathbbm{1}\{\pi_{s}=i\}.
		\]
		We first handle the case where $r_t=1$ and thus $B=[0,1]^D$. Then, we have 
		\[
		    n_{r_t}(X_t) = n_1(X_t) = \lceil 8K\log(K/\delta)\rceil + 1.
		\]
		By Line 5 of Algorithm~\ref{algo:one-arm}, for each of the $t$ rounds elapsed thus far, we pulled arm $i$ with probability $1/K$. Thus, we have
		\[
			\mathbb{E}[m_t(B,i)] = n_1(X_t)\cdot \frac{1}{K} \geq 8\log(K/\delta).
		\]
		Then, by a Chernoff bound, since $\{Z_s\}_{s<t}$ are independent: 
		\[
			\mathbb{P}\left( m_t(B,i) \leq \frac{n_{1}(X_t)}{2K}\right) \leq \mathbb{P}\left(m_t(B,i) \leq \frac{\mathbb{E}\, [m_t(B,i)]}{2}\right) \leq \delta/K.
		\]
		This gives us the desired result for $r_t=1$.
		
		The general case of $r_t<1$ will follow similarly from a Chernoff bound. However, more care is required in that the sequence $\{Z_s\}_{s<t}$ is no longer independent since each $Z_s$ depends on ${\cal I}_B$ (Line 12 of Algorithm~\ref{algo:one-arm}), a random object possibly varying with time and depending on $\{Z_s\}_{s<t},{\bf Y}_{t-1},{\bf X}_{t-1}$. To overcome this, we relate $m_t(B,i)$ to a smaller count of {\em independently randomized} arm-pulls so that we can use concentration.
		
		First, let $t_0$ be the first round that the parent $B' \doteq T_{2r_t}(X_t)$ of bin $B$ was used. Let $m_{[t_0,t]}(B,i)$ be the number of pulls of arm $i$ in bin $B$ between rounds $t_0$ and $t$. Then, it suffices to show $m_{[t_0,t]}(B,i) \geq n_{r_t}(X_t)/(4K)$.
				
		Let ${\cal A}_{t_0}$ be the set of candidate arms determined in $B'$ at time $t_0$ so that ${\cal A}_{t_0} \supseteq {\cal I}_B$. Next, let $\tilde{m}_{[t_0,t]}(B,i)$ be a draw from a $\text{Binomial}(n_{[t_0,t]}(B),1/|{\cal A}_{t_0}|)$ distribution where $n_{[t_0,t]}(B)$ is the number of covariates observed in $B$ between times $t_0$ and $t$. More plainly, $\tilde{m}_{[t_0,t]}(B,i)$ is the number of pulls of arm $i$ in $B$ between times $t_0$ and $t$ if Algorithm~\ref{algo:one-arm} was {\em prohibited} from eliminating any arms in $\mathcal{A}_{t_0}$ between times $t_0$ and $t$ in bin $B'$ -- in other words, $\tilde{m}_{[t_0,t]}(B,i)$ counts the number of times arm $i$ is selected with probability $1/|\mathcal{A}_{t_0}|$ instead of with probability $1/|{\cal I}_B|$, as in Line 12 of Algorithm~\ref{algo:one-arm}.
		
		Then, since $|{\cal I}_B|\leq |{\cal A}_{t_0}|$, we have that $m_{[t_0,t]}(B,i) \geq \tilde{m}_{[t_0,t]}(B,i)$ so that it remains to show $\tilde{m}_{[t_0,t]}(B,i) \geq n_{r_t}(X_t)/(4K)$. 
		
		We first show that the number $n_{[t_0,t]}(B)$ of covariates observed in $B$ between $t_0$ and $t$ is at least $n_{r_t}(X_t)/2$. Since $t$ is the first time that $B$ is chosen, we must have by Line 7 of Algorithm~\ref{algo:one-arm}
		\begin{equation}\label{eqn:b-count}
			r_t \geq \sqrt{\frac{8K\log(K/\delta)}{n_{r_t}(X_t)}} \implies n_{r_t}(X_t) \geq \frac{8K\log(K/\delta)}{r_t^2}.
		\end{equation}
		By similar reasoning, since $t_0$ is the first time that $B$'s parent at level $2r_t$ is chosen, we must have
		\begin{equation}\label{eqn:b-count-parent}
			\sqrt{\frac{8K\log(K/\delta)}{n_{2r_t}(X_{t_0})-1}} > 2r_t \implies n_{2r_t}(X_{t_0}) \leq \frac{8K\log(K/\delta)}{4r_t^2}+1.
		\end{equation}
		Then, putting \eqref{eqn:b-count} and \eqref{eqn:b-count-parent} together, we have $n_{r_t}(X_t) - n_{[t_0,t]}(B) \leq n_{2r_t}(X_{t_0}) \leq \frac{n_{r_t}(X_t)}{2}$. This implies that $n_{[t_0,t]}(B) \geq n_{r_t}(X_t)/2$.		
		
		Then, by Lemma~\ref{lem:no-skip}, at every round $s\in [t_0,t) \cap \mathbb{N}$ during which a covariate was observed in $B$, we pulled arm $i$ with probability at least $1/K$. Thus, we have \[
			\mathbb{E}\,[\tilde{m}_{[t_0,t]}(B,i)\mid \textbf{X}_{t_0},\textbf{Y}_{t_0},{\cal A}_{t_0}] \geq n_{[t_0,t]}(B) \cdot \frac{1}{K} \geq  \frac{ n_{r_t}(X_t)}{2K}.
		\]
		Furthermore, we have from \eqref{eqn:b-count} and the fact that $r_t \leq 1/2$ that the above R.H.S. is further lower bounded by:
		\[
			\frac{8\log(K/\delta)}{2r_t^2} \geq 8\log(K/\delta).
		\]		
		Then, since $\tilde{m}_{[t_0,t]}(B,i)$ is a sum of {\em independent} Bernoulli's conditioned on ${\bf X}_{t_0},{\bf Y}_{t_0},{\cal A}_{t_0}$, by a Chernoff bound, we have: \[
		\mathbb{P}\left( \tilde{m}_{[t_0,t]}(B,i) \leq \frac{n_{r_t}(X_t)}{4K}\right) \leq \mathbb{P}\left( \tilde{m}_{[t_0,t]}(B,i) \leq \frac{\mathbb{E}\, [\tilde{m}_{[t_0,t]}(B,i)\mid {\bf X}_{t_0},{\bf Y}_{t_0},{\cal A}_{t_0}]}{2}\right) \leq \delta/K.
		\]
		Re-tracing our previous steps, we have with probability at least $1-\delta/K$:
		\[
			m_t(B,i) \geq m_{[t_0,t]}(B,i) \geq \tilde{m}_{[t_0,t]}(B,i) > \frac{n_{r_t}(X_t)}{4K}.
		\]
		
		By a union bound and the tower property, we have that the event $\{\forall i\in\mathcal{I}_B: m_t(B,i) \geq \frac{n_{r_t}(X_t)}{4K}\}$ holds with probability at least $1-\delta$ w.r.t. the distribution of $\textbf{Y}_{t-1},\{Z_s\}_{s<t}|\textbf{X}_t$.
	\end{proof}
		
	Now, consider a round $t> \lceil 8K\log(K/\delta)\rceil$ with observed covariate $X_t$ and selected bin $B$. Suppose $s\leq t$ is the first round when bin $B$ is selected. Then, the set of candidate arms determined at round $s$ in $B$ must contain any arm currently retained in $B$ at round $t$.
	
	 Furthermore, by Lemma~\ref{lem:arm-count-randomization}, we have that with probability at least $1-\delta$, $m_t(B,i)\geq m_s(B,i)> n_{r_s}(X_s)/(4K)$. Thus, combining this with our earlier bias-variance bound \eqref{eqn:bias-variance-app}:
		\begin{align*}
			|\hat{f}_t^i(B)-f^i(x)| &\leq \sqrt{\frac{4K\log(2K/\delta)}{n_{r_s}(X_s)}} + 2\lambda r_t \leq 2\lambda r_s + 2\lambda r_t = 4\lambda  r_t.
		\end{align*}
		This shows that the regression error at round $t$ is at most $4\lambda r_t$. 
	
	\paragraph{Justifying Arm Eliminations.}
	\label{subsec:check-arm-eliminations}
	
	
	For each round $t > 8K\log(K/\delta)$, define the event $G_t$ on which Lemma~\ref{lem:bias-variance} and Lemma~\ref{lem:arm-count-randomization} hold or \[
		G_t \doteq \left\{\forall i \in \mathcal{I}_B, x \in B:|\hat{f}_t^i(B) - f^i(x)|\leq 4\lambda r_t, B=T_{r_t}(X_t)\right\}.
\]
	This is the ``good'' event on which our regression function estimates $\hat{f}_t^i(B)$ are accurate enough to be able to discern which arms have low and high rewards. For the sake of brevity, from here on, let $B$ be the selected bin at round $t$. This first proposition asserts that an eliminated arm cannot have a better reward than the best candidate arm. 
	
\begin{prop}\label{prop:better-arm}
Suppose at round $t$, under event $G_t$, we select bin $B$. Then for any two arms $i,j\in\mathcal{I}_{B}$ and any $x\in B$, 
\[
	\hat{f}_t^i(B)-\hat{f}_t^j(B)> 8\lambda r_t\implies f^i(x)>f^j(x)
\]
\end{prop}

\begin{proof}
Using the definition of $G_t$, we have
	\[
		f^i(x)-f^j(x)\geq \hat{f}_t^i(B)-\hat{f}_t^j(B)-8\lambda r_t>0.
	\]
\end{proof}

The first implication of Proposition~\ref{prop:better-arm} is that under event $\cap_{s=1}^t G_s$, the best arm $i^*(x)$ at any covariate $x \in B$ is always retained in $\mathcal{I}_B$. This is immediate since $i^*(x)$ can never be discarded for any $x\in B$ as long as the regression bounds of $G_s$ hold.

\begin{cor}\label{cor:best-arm-candidate}
	Suppose at round $t$, under event $(\cap_{s=1}^t G_s)$, we select bin $B$. Then $\mathcal{I}_{B}$ contains the best arm $i^*(x)=\argmax_{j\in[K]} f^j(x)$ for all $x\in B$.
\end{cor}


The next corollary gives us that the margin and regret of playing any candidate arm at any point in $B$ is bounded by $\lambda r_t$. Following the discussion of Section~\ref{sec:regret-analysis}, it will then suffice to bound $r_t$ in terms of $n_P$ and $t$.

\begin{cor}\label{cor:classification-hard}
	Suppose at round $t$, under event $(\cap_{s=1}^t G_s)$, we select bin $B$. Then, both of the following hold for all $x\in B$:
	\begin{enumerate}[label=(\arabic*)]
		\item $|f^{(1)}(x)-f^{j}(x)|\leq 16\lambda r_t$ for all $j\in\mathcal{I}_{B}$.
		\item Either $0<|f^{(1)}(x)-f^{(2)}(x)|\leq 16\lambda r_t$ or $f^j(x)=f^{(1)}(x)$ for all $j\in\mathcal{I}_B$
	\end{enumerate}
\end{cor}

\begin{proof}
	Fix $x \in B$, and let $i^*(x)=\underset{j\in[K]}{\argmax\  } f^j(x)$. Using the definition of $G_t$ and the fact that $i^*(x)\in\mathcal{I}_{B}$ (Corollary~\ref{cor:best-arm-candidate}), we first establish (1):
	\begin{align*}
		f^{i^*(x)}(x)-f^{j}(x)	&\leq \hat{f}^{i^*(x)}(x) - \hat{f}^j(x) + 8\lambda r_t\\
		&\leq \hat{f}^{(1)}(x) - \hat{f}^j(x) + 8\lambda r_t\\
		&\leq 16\lambda r_t \text{ (because $j$ was not eliminated)}
	\end{align*}
	To show (2), we have if $\mathcal{I}_{B}$ contains a sub-optimal arm $j\in\mathcal{I}_{B}$ at $x$, then by (1),
	\begin{align*}
		|f^{(1)}(x)-f^{(2)}(x)|=f^{(1)}(x)-f^{(2)}(x)\leq f^{(1)}(x)-f^{j}(x)\leq 16\lambda r_t.
	\end{align*}
    On the other hand, if ${\cal I}_B$ does not contain a sub-optimal arm at $x$, then $f^j(x)=f^{(1)}(x)$ for all $j\in{\cal I}_B$.
\end{proof}

	\paragraph{Relating adaptive $r_t$ to an oracle choice $r_t^{*}$}
	
	Following the outline of Section~\ref{sec:regret-analysis}, we relate $r_t$ to the oracle choice $r_t^* \doteq r_t(\gamma,n_P)$ for rounds $t > n_P$. We do this by first establishing that $r_t$ leads to near optimal regression estimates. Let
	\[
		\phi_t(r) \doteq \sqrt{\frac{8K\log(K/\delta)}{n_r(X_t)}} + \lambda r.
	\]
	This is essentially the bias-variance decomposition from earlier bounding the regression error $|f^i(x) - \hat{f}_t^i(B)|$. Next, we show that $r_t$, the high probability regret bound achieved thus far (Corollary~\ref{cor:classification-hard}), is less than $\phi_t(r)$ for any $r\in {\cal R}$.
	
\begin{prop}[$r_t$ minimizes $\phi_t(\cdot)$]\label{prop:rt-minimizer}
	We have $r_t \leq 2\cdot\underset{r \in \mathcal{R}}\min\, \phi_t(r)$.
\end{prop}

\begin{proof}
	Fix $r\in\mathcal{R}$. If $r_t\leq r$, then using the fact that $\lambda \geq 1$:
	\[
		r_t \leq r \leq \lambda r \leq \phi_t(r).
	\]
	If $r_t>r \implies r_t/2 \geq r$, then using the definition of $r_t$ on Line 7 of Algorithm~\ref{algo:one-arm}:
	\[
	    r_t = 2(r_t/2) \leq 2\sqrt{\frac{8K \log(K/\delta)}{n_{r_t/2}(X_t)}} \leq 2\sqrt{\frac{8K \log(K/\delta)}{n_r(X_t)}} \leq 2\cdot \phi_t(r).
	   \]
\end{proof}

Proposition~\ref{prop:rt-minimizer} directly gives us that $\lambda r_t \leq 2\lambda \cdot \min_{r\in{\cal R}} \phi_t(r)$.  
Thus, using the level $r_t$ at time $t$ achieves regression error {\em no worse than that of any other choice of level} $r$.

Recall from Section~\ref{sec:Algorithms}, we let $\tau \doteq t-n_P-1$, the time elapsed after round $n_P$, to further simplify notation. Also, let $r_t^{*}$ be the oracle choice of level introduced in Section~\ref{sec:results}, or the smallest level of $\mathcal{R}$ greater than or equal to
\[
	\min\left(\left(\frac{K\log(K/\delta)}{n_P}\right)^{\frac{1}{2+\alpha+d+\gamma}},\left(\frac{K\log(K/\delta)}{\tau}\right)^{\frac{1}{2+\alpha+d}}\right).
\]
Then, by Proposition~\ref{prop:rt-minimizer}, we have $\lambda r_t \leq 2\lambda \cdot \phi_t(r_t^*)$. From here, it remains to integrate $\phi_t(r_t^*)$ over the covariate space. 

\paragraph{Expected regret at time $t$: integrating over margin distribution.} To later use the margin condition (Definition~\ref{assumption:margin}) for all relevant bounds $\Delta$ on the margin, we first need to ensure that $\Delta \leq \delta_0$, where $\delta_0$ is as in Definition~\ref{assumption:margin}. It turns out that this will only amount to constraining our analysis to rounds for which $\tau \gtrsim K\log(K/\delta)$, which will not pose an issue since the regret of any fixed $O(K\log(K/\delta))$ rounds among rounds $\{n_P+1,\ldots,n\}$ is of the right order with respect to Theorem~\ref{algo:one-arm}. Formally, let $\tau_0$ be the largest positive integer such that
\begin{equation}\label{eqn:delta0}
	\tau_0 \leq \left(\frac{c_0 \vee c_3}{\delta_0}\right)^{2+\alpha+d} K\log(K/\delta),
\end{equation}
where $c_0,c_3 \geq 1$ are constants to be determined (see Lemma~\ref{lem:indicator-split} and \eqref{eqn:a<eps-app}, respectively, for where they arise). Rearranging \eqref{eqn:delta0}, we obtain:
\[
	c_0\left(\frac{K\log(K/\delta)}{\tau_0}\right)^{\frac{1}{2+\alpha+d}} \geq \delta_0.
\]
Thus, for rounds $t$ such that $\tau > \tau_0$, from the above, we have $c_0 r_t^* \leq \delta_0$. 

Similarly, again rearranging \eqref{eqn:delta0}, we obtain for $\tau > \tau_0$:
\begin{equation}\label{eqn:delta0-variance}
	c_3\sqrt{\left(\frac{K\log(K/\delta)}{\tau}\right)^{1-\frac{\alpha+d}{2+\alpha+d}}} \leq \delta_0.
\end{equation}
The above inequality will later be useful in integrating over low-margin regions of ${\cal X}$ with respect to the variance term in $\phi_t(r_t^*)$.


Next, following the outline of Section~\ref{sec:regret-analysis},  we consider bins of sufficient mass at level $r_t^*$ to use concentration. So, define the event $A_t$:
\[
	A_t=\left\{\max(\tau Q_X(T_{r_t^{*}}(X_t)),n_P P_X(T_{r_t^{*}}(X_t))) \geq 8\log(1/\delta) \right\}.
\]
Then, in the following proposition, we use concentration to relate the variance term of $\phi_t(r_t^*)$ to the masses $Q_X(T_{r_T^*}(X_t)),P_X(T_{r_t^*}(X_t))$, under event $A_t$.

\begin{prop}\label{prop:optimal-bias-variance}
	Consider any round $t$ with $\tau > K\log(K/\delta)$ and with observed covariate $X\doteq X_{t}$. We then have, at round $t$, that with probability at least $1-\delta$ w.r.t. the conditional distribution ${\bf X}_{t-1}|X_{t},A_{t}$:
	\begin{equation}\label{eqn:phi-decomposition}
		\phi_t(r_t^*) \leq 2\sqrt{\frac{8K\log(K/\delta)}{\max(\tau Q_X(T_{r_t^{*}}(X_t)),n_P P_X(T_{r_t^{*}}(X_t)))}} + 2\lambda r_t^{*}.
	\end{equation}
\end{prop}

\begin{proof}
	We note that 
\[
	\mathbb{E}\  (n_{r_t^{*}}(X_t))=n_P P_X(T_{r_t^{*}}(X)) + \tau Q_X(T_{r_t^{*}}(X)) \geq \max(\tau Q_X(T_{r_t^{*}}(X_t)),n_P P_X(T_{r_t^{*}}(X_t))).
\]
Then, we have, under event $A_t$, by a Chernoff bound that:
\[
	\mathbb{P}\left( n_{r_t^{*}}(X_t) \leq \frac{1}{2}\mathbb{E}[n_{r_t^{*}}(X_t)]\right) \leq \exp\left(-\frac{1}{8}\mathbb{E}[n_{r_t^{*}}(X_t)]\right) \leq \delta.
\]
Thus, with probability at least $1-\delta$:
\begin{align*}
	\phi_t(r_t^{*}) \leq 2\left(\sqrt{\frac{8K \log(K/\delta)}{\max(\tau Q_X(T_{r_t^{*}}(X_t)),n_P P_X(T_{r_t^{*}}(X_t))}} + \lambda r_t^{*}\right).
\end{align*}
\end{proof}

Combining Proposition~\ref{prop:optimal-bias-variance} with Proposition~\ref{prop:rt-minimizer}, we have with probability at least $1-\delta$:
	\begin{equation}\label{eqn:rt-bias-variance}
		\lambda r_t \leq 2\lambda \sqrt{\frac{8K\log(K/\delta)}{\max(\tau Q_X(T_{r_t^{*}}(X_t)),n_P P_X(T_{r_t^{*}}(X_t)))}} + 2\lambda^2 r_t^{*}
	\end{equation}
	From here on, fix a round $t$ with corresponding $\tau > \tau_0$. 
\begin{note}
	We recall some of the notation from Section~\ref{sec:regret-analysis}, which we make more precise here. Denote the first term on the R.H.S. of \eqref{eqn:rt-bias-variance} by $\sigma_t^{*}$. Let $\delta_f(X_t) \doteq f^{(1)}(X_t) - f^{(2)}(X_t)$ be the margin at $X_t$. Define the event $E_t$ as the event on which the bound of Proposition~\ref{prop:optimal-bias-variance} holds or:
	\[
		E_t \doteq \left\{ n_{r_t^{*}}(X_t) > \frac{1}{2}\mathbb{E}[n_{r_t^{*}}(X_t)] \right\}.
	\]
	Finally, let $F_t \doteq A_t \cap E_t\cap \left(\cap_{s=1}^t G_s\right)$, which is the event on which all the high-probability bounds established thus far hold. 
\end{note}

\begin{lemma}\label{lem:indicator-split}
	For $\tau> K\log(K/\delta)$:
	\begin{equation}\label{eqn:bias+variance}
		\mathbb{E}[(f^{(1)}(X_t) - f^{\pi_t}(X_t))\mathbbm{1}\{F_t\}] \leq 
c_{0}\left( \mathbb{E} [r_t^{*} \cdot \mathbbm{1}\{0 < \delta_f(X_t) \leq c_0 r_t^{*}\}] + \mathbb{E} [\sigma_t^{*} \cdot \mathbbm{1}\{0 < \delta_f(X_t) \leq c_0\sigma_t^* \} ] \right).
	\end{equation}
\end{lemma}

\begin{proof}
As in Section~\ref{sec:regret-analysis}, to simplify notation, let $X \doteq X_t$. From Corollary~\ref{cor:classification-hard}, we have that a non-zero regret (of at most $16\lambda r_t$) is incurred only if $\delta_f(X_t)\leq 16r_t$. Combining this fact with \eqref{eqn:rt-bias-variance}, we have the regret at round $t$ of pulling arm $j$ under event $F_t$ is
\begin{align*}
	\mathbb{E}[(f^{(1)}(X) - f^{\pi_t}(X))\mathbbm{1}\{F_t\}] &\leq \mathbb{E}[(f^{(1)}(X) - f^{\pi_t}(X))\mathbbm{1}\{0<(f^{(1)}(X) - f^{\pi_t}(X))\\
	&\qquad \vee \delta_f(X) \leq c_1(\sigma_t^*+r_t^*)\}].
\end{align*}
Next, using $\mathbbm{1}\{x\leq a+b\} \leq \mathbbm{1}\{x\leq 2a\} + \mathbbm{1}\{x\leq 2b\}$, we decompose the above R.H.S.:
\begin{align*}
	\mathbb{E} [ (f^{(1)}(X) - f^{\pi_t}(X))\mathbbm{1}\{F_t\} ] &\leq 	\mathbb{E}  \left[ (f^{(1)}(X) - f^{\pi_t}(X))\cdot\mathbbm{1}\{0 < \delta_f(X) \vee (f^{(1)}(X) - f^{\pi_t}(X)) \leq c_0 r_t^{*}\} \right] \\ 
	&\enspace + \mathbb{E} \left[ (f^{(1)}(X) - f^{\pi_t}(X))\cdot\mathbbm{1}\{0 < \delta_f(X) \vee (f^{(1)}(X) - f^{\pi_t}(X)) \leq c_0 \sigma_t^{*}\} \right] \\
	&\leq c_0\left(	\mathbb{E}  \left[ r_t^*\cdot\mathbbm{1}\{0 < \delta_f(X) \leq c_0 r_t^{*}\} \right]  + \mathbb{E} \left[ \sigma_t^*\cdot\mathbbm{1}\{0 < \delta_f(X) \leq c_0 \sigma_t^{*}\} \right] \right).
\end{align*}
\end{proof}

We next show that the R.H.S. of \eqref{eqn:bias+variance} is of order at most $(r_t^*)^{1+\alpha}$. For the first term on the R.H.S. of \eqref{eqn:bias+variance}, this is immediate: we use the margin assumption (Definition~\ref{assumption:margin}) along with \eqref{eqn:delta0} to write:
\[
	\mathbb{E}_{X_t} \left[ r_t^{*} \cdot \mathbbm{1}\{0 < \delta_f(X_t) \leq c_0 r_t^{*}\} \right] \leq c_{2} (r_t^*)^{1+\alpha}.
\]
Thus, it remains to analyze the second term, involving $\sigma_t^*$, on the R.H.S. of \eqref{eqn:bias+variance}. Now, since $r_t^*$ is a minimum of two terms, one involving $\tau$ and one involving $n_P$, it suffices to show separately that $\mathbb{E}[\sigma_t^*\cdot \mathbbm{1}\{0 < \delta_f(X_t)\leq c_0 \sigma_t^*\}]$ is $\tilde{O}(\tau^{-\frac{\alpha+1}{2+\alpha+d}})$ and is also $\tilde{O}(n_P^{-\frac{\alpha+1}{2+\alpha+d+\gamma}})$. We show the first bound involving $\tau$; the bound involving $n_P$ will follow from the same arguments with the appropriate modifications (which we will make explicit hereafter).

Proceeding with the bound involving $\tau$, first consider the quantity $a(X_t) \doteq Q_X(T_{r_t^*}(X_t))^{-1/2}$. We note that
\begin{equation}\label{eqn:v-bound}
	\sigma_t^{*} \leq 2\lambda \cdot a(X_t) \cdot \sqrt{\frac{8K \log(K/\delta)}{\tau}}
\end{equation}
Thus, it suffices to bound the expectation of $a(X_t)$ over regions of low-margin. We achieve this by splitting the integral into two terms, conditioned on the value of $a(X_t)$ being above or below a threshold $\sqrt{\epsilon}$ for some fixed $\epsilon>0$ (to be optimized over later). More precisely, for some $\epsilon>0$, we decompose the second term on the R.H.S. of \eqref{eqn:bias+variance} as:
\begin{equation}\label{eqn:variance-decomposition}
	\mathbb{E}\left[ \sigma_t^* \cdot \mathbbm{1}\{0 < \delta_f(X_t) \leq c_0 \sigma_t^*\}\cdot (\mathbbm{1}\{a(X_t) < \sqrt{\epsilon}\} + \mathbbm{1}\{ a(x_t) \geq \sqrt{\epsilon}\} )\right].
\end{equation}
This gives us two cases to consider: $a(X_t) < \sqrt{\epsilon}$ and $a(X_t) \geq \sqrt{\epsilon}$. We start by handling the case where $a(X_t) < \sqrt{\epsilon}$. Plugging \eqref{eqn:v-bound} into the first term above gives (using $c_3$ to collect constants):
\[
	c_3 \cdot \sqrt{\frac{K \log(K/\delta)}{\tau}} \cdot \mathbb{E}\left[ a(X_t) \mathbbm{1}\left\{0< \delta_f(X_t) \leq  c_3 \cdot a(X_t) \cdot \sqrt{\frac{K \log(K/\delta)}{\tau}}\right\}\cdot \mathbbm{1}\{a(X_t) < \sqrt{\epsilon}\}\right].
\]
Next, bounding $a(X_t)$ by $\sqrt{\epsilon}$ inside the expectation and using the margin assumption (justified by \eqref{eqn:delta0-variance} for our eventual optimal value $\epsilon = \left(\frac{\tau}{K\log(K/\delta)}\right)^{\frac{\alpha+d}{2+\alpha+d}}$), the above is bounded by:
\begin{equation}\label{eqn:a<eps-app}
	c_{4}\left(\sqrt{\frac{\epsilon K \log(K/\delta)}{\tau}}\right)^{\alpha+1}.
\end{equation}
For the case where $a(X_t) \geq \sqrt{\epsilon}$, we again have by \eqref{eqn:v-bound} that:
\begin{equation}\label{eqn:larger-eps}
	\mathbb{E}\left[ \sigma_t^* \cdot \mathbbm{1}\{0 < \delta_f(X_t) \leq c_0 \sigma_t^*\}\cdot \mathbbm{1}\{ a(x_t) \geq \sqrt{\epsilon}\} \right] \leq c_{5}\sqrt{\frac{K\log(K/\delta)}{\tau}} \mathbb{E}_{X_t} [ a(X_t)\mathbbm{1}\{ a(X_t) \geq \sqrt{\epsilon}\}].
\end{equation}
Note that the expectation on the R.H.S. is only over $X_t \sim Q_X$ since $a(\cdot)$ depends only on $X_t$. We use the following proposition to bound this integral over $X_t$. 

\begin{prop}\label{prop:tail-bound}
	For some $c_{6}.c_7>0$, we have for any $r\in\mathcal{R}$:
	\begin{align*}
		\mathbb{E}_{X\sim Q_X}\left[\frac{1}{\sqrt{Q_X(T_{r}(X))}} \cdot  \mathbbm{1}\left\{\frac{1}{\sqrt{Q_X(T_{r}(X))}} \geq \sqrt{\epsilon}\right\}\right] &\leq c_{6} \frac{r^{-d}}{\sqrt{\epsilon}} \numberthis \label{eqn:q-integral}\\
		\mathbb{E}_{X\sim Q_X}\left[ \frac{1}{\sqrt{P_X(T_{r}(X))}} \cdot  \mathbbm{1}\left\{\frac{1}{\sqrt{P_X(T_{r}(X))}} \geq \sqrt{\epsilon}\right\}\right] &\leq c_{7} \frac{r^{-d-\gamma}}{\sqrt{\epsilon}} \numberthis \label{eqn:p-integral}
	\end{align*}
	
\end{prop}

\begin{proof}
	We first show \eqref{eqn:q-integral}; \eqref{eqn:p-integral} will be shown nearly identically. In an abuse of notation, let $a(X) \doteq Q_X(T_{r}(X))^{-1/2}$. We first use the tail probability formula for expectations:
\begin{equation*}
	\mathbb{E}_{X}[ a(X) \cdot \mathbbm{1}\{ a(X) \geq \sqrt{\epsilon}\} ] = \int_0^{\infty} Q_X(a(X) \cdot \mathbbm{1}\{a(X) \geq \sqrt{\epsilon}\}\geq s)\,ds
\end{equation*}
	Next, we observe that
\[
	a(X)\mathbbm{1}\{a(X) \geq \sqrt{\epsilon}\} \geq s \iff a(X) \geq \sqrt{\epsilon} \vee s.
\]
This gives us
\begin{align*}
	\int_0^{\infty} Q_X(a(X)\mathbbm{1}\{a(X) \geq \sqrt{\epsilon}\}\geq s)\,ds &= \int_{\sqrt{\epsilon}}^{\infty} Q_X(a(X) \geq s)\,ds + \int_0^{\sqrt{\epsilon}} Q_X(a(X) \geq \sqrt{\epsilon})\,ds.\\
	&\leq \int_{\sqrt{\epsilon}}^{\infty} \frac{\mathbb{E}_{Q_X}(a(X)^2)}{s^2}\,ds + \sqrt{\epsilon}\cdot \frac{\mathbb{E}_{Q_X}(a(X)^2)}{\epsilon},  \numberthis\label{eqn:tail-prob}
\end{align*}
where we used Chebyshev's inequality above to bound the masses $Q_X(\cdot)$. To bound the second moments in \eqref{eqn:tail-prob}, we observe a fairly standard implication of the {\em box cover dimension} $d$ (Definition~\ref{defn:bcn}): for any $r\in\mathcal{R}$:
\begin{align*}
	\mathbb{E}_{Q_X}(a^2(X))=\mathbb{E}_{Q_X}\left(\frac{1}{Q_X(T_r(X))}\right)\leq C_d\cdot r^{-d}.
\end{align*}
Plugging the above into \eqref{eqn:tail-prob} gives us that
\[
	\mathbb{E}_X [ a(X)\cdot \mathbbm{1}\{ a(X) \geq \sqrt{\epsilon}\}] \leq c_6 \frac{r^{-d}}{\sqrt{\epsilon}}.
\]
To show \eqref{eqn:p-integral},  we repeat the argument above using $b(X) \doteq P_X(T_r(X))^{-1/2}$ and the following analogous bound on $\mathbb{E}_{Q_X}(b^2(X))$ (using the definition of transfer exponent in Assumption~\ref{assumption:transfer-exponent}):
\[
	\mathbb{E}_{Q_X}\left(\frac{1}{P_X(T_r(X))}\right)\leq \mathbb{E}_{Q_X}\left(\frac{1}{C_{\gamma}\cdot r^{\gamma}\cdot Q_X(T_r(X))}\right) \leq  \frac{C_d}{C_{\gamma}}\cdot r^{-d-\gamma}.
\]
\end{proof}

Thus, by \eqref{eqn:q-integral} of Proposition~\ref{prop:tail-bound}, \eqref{eqn:larger-eps} is bounded by
\begin{equation}\label{eqn:a>eps-final}
	c_8 \sqrt{\frac{K\log(K/\delta)}{\tau}} \cdot \frac{(r_t^*)^{-d}}{\sqrt{\epsilon}}.
\end{equation}
Then, \eqref{eqn:a<eps-app} and \eqref{eqn:a>eps-final} are balanced by setting $\epsilon \doteq \left(\frac{\tau}{K\log(K/\delta)}\right)^{\frac{\alpha+d}{2+\alpha+d}}$, which makes \eqref{eqn:bias+variance}, and hence the regret at round $t$ under event $F_t$, order $O\left(\left(\frac{K\log(K/\delta)}{\tau}\right)^{\frac{\alpha+1}{2+\alpha+d}}\right)$.	

Next, we claim that \eqref{eqn:bias+variance} is bounded by
\[
	O\left(\left(\frac{K\log(K/\delta)}{n_P}\right)^{\frac{\alpha+1}{2+\alpha+d+\gamma}}\right).
\]
For this, we can follow an identical line of reasoning as above,  with the following substitutions: 
\begin{enumerate}[label=(\alph*)]
	\item Replace the measure $Q_X$ for $P_X$, and, thus, the quantity $a(X_t)$ by $b(X_t) \doteq P_X(T_{r_t^*}(X_t))^{-1/2}$.
	\item Use instead the threshold value $\epsilon \doteq \left(\frac{n_P}{K\log(K/\delta)}\right)^{\frac{\alpha+d+\gamma}{2+\alpha+d+\gamma}}$ in the analogue of \eqref{eqn:variance-decomposition}.
	\item Use \eqref{eqn:p-integral} instead of \eqref{eqn:q-integral} from Proposition~\ref{prop:tail-bound}.
\end{enumerate}
Thus, \eqref{eqn:bias+variance} is of order $O\left(\min\left(\left(\frac{K\log(K/\delta)}{n_P}\right)^{\frac{\alpha+1}{2+\alpha+d+\gamma}},\left(\frac{K\log(K/\delta)}{\tau}\right)^{\frac{\alpha+1}{2+\alpha+d}}\right)\right)$.

\paragraph{Bounding the Regret on Bad Events and Summing the Regrets over $t$}
	\label{subsec:bad-regret}
	
	It remains to bound the regret under the low-probability event $F_t^c$ and sum our bounds over $\tau\in[n_Q]$. By an integral approximation, we have
	\begin{align*}
		\sum_{\tau=K\log(K/\delta)}^{n_Q} \left(\frac{K\log(K/\delta)}{\tau}\right)^{\frac{\alpha+1}{2+\alpha+d}} &\leq
		 c_{9}\int_{K\log(K/\delta)}^{n_Q} \left(\frac{K\log(K/\delta)}{z}\right)^{\frac{\alpha+1}{2+ \alpha + d}}\,dz
	\end{align*}
	If $\alpha \leq d +\alpha + 1$, this integral, for some $c_{10}>0$, is bounded by
	\[
		c_{10} n_Q\left(\frac{K\log(K/\delta)}{n_Q}\right)^{\frac{\alpha+1}{2+d}}.
	\]
	Otherwise, it is bounded by $O(K\log(K/\delta))$.

Since the regret at round $t$ is bounded by $1$ on $F_t^c$, it remains to show $\mathbb{P}(F_{t}^c)$ is appropriately small. First, by the definition of $F_t$, the definition of event $A_t$, and Proposition~\ref{prop:optimal-bias-variance}, we have:
    \begin{align*}
		\mathbb{P} (F_{t}^c)&\leq \mathbb{P}(A_t^c) + \mathbb{P}(E_t^c) + \mathbb{P}\left(\cup_{s=1}^{t} G_s^c\right)
		\leq 
			\mathbb{P}(A_t^c) + \delta + 2\cdot t\cdot \delta,
	\end{align*}
	Summing $\delta+2\cdot t\cdot \delta$ over $t$ accounts for the $O(n_Q n\delta)$ term in the desired regret bound of Theorem~\ref{thm:adaptive-algo-one-arm}. Finally, to handle the term $\mathbb{P}(A_t^c)$, we use the definition of the transfer exponent (Assumption~\ref{assumption:transfer-exponent}) and the definition of the support complexity (Definition~\ref{defn:bcn}):
\begin{align*}
	\mathbb{P}\left( A_t^c \right) &= \mathbb{P}\left(Q_X(T_{r_t^{*}}(X_t)) < \frac{8\log(1/\delta)}{\tau} \cap P_X(T_{r_t^*}(X_t)) < \frac{8\log(1/\delta)}{n_P}\right)\\
	&\leq \mathbb{P}\left(Q_X(T_{r_t^{*}}(X_t)) < \frac{8\log(1/\delta)}{\tau} \cap C_{\gamma}\cdot (r_t^*)^{\gamma}\cdot Q_X(T_{r_t^*}(X_t)) < \frac{8\log(1/\delta)}{n_P}\right)\\
	&= \sum_{B \in T_{r_t^{*}}} \mathbb{P}\left(Q_X(B) < \min\left(\frac{8\log(1/\delta)}{\tau},\frac{8\log(1/\delta)}{n_P\cdot C_{\gamma}\cdot (r_t^*)^{\gamma}}\right) \cap B=T_{r_t^{*}}(X_t)\right)\\
	&\leq C_d \cdot (r_t^{*})^{-d} \cdot \min\left(\frac{8\log(1/\delta)}{\tau},\frac{8\log(1/\delta)}{n_P\cdot C_{\gamma}\cdot (r_t^*)^{\gamma}}\right)\\
	&\leq c_{11} \min\left(\left(\frac{K\log(K/\delta)}{\tau}\right)^{\frac{\alpha+2}{2+\alpha+d}},\left(\frac{K\log(K/\delta)}{n_P}\right)^{\frac{\alpha+2}{2+\alpha+d+\gamma}}\right).
\end{align*}
	This is clearly smaller than the bound on the regret at time $t$ we derived in the previous section. Thus, summing the above bound over $\tau$ using an integral approximation in the same manner as before, we see that $\sum_{t=n_P+1}^n \mathbb{P}(A_t^c)$ is of the right order. This concludes the proof of Theorem~\ref{thm:adaptive-algo-one-arm}.	$\hfill\blacksquare$

	\section{Multiple Shifts}
    \label{app:multiple-shifts}
    
	In this section, we give an extension of Theorem~\ref{thm:adaptive-algo-one-arm} to multiple distribution shifts. Let ${\cal P} \doteq \{P_j\}_{j=1}^N$ be a sequence of $N$ initial distributions on the covariate-reward pair $(X,Y)$. Suppose each $P_j$ satisfies covariate shift with respect to the terminal distribution $Q$. We then consider bandits with a sequence of shifts \[
		P_1 \to P_2 \to \cdots \to P_N \to Q.
	\]
	In this setup, data from each $P_j$ is observed for $n_j$ consecutive rounds. Then, there are $n_P \doteq \sum_{j=1}^N n_j$ total rounds played under distributions from the class ${\cal P}$. We then consider the regret $\textbf{R}_n^Q(\pi)$ of a policy $\pi$ playing $n = n_P + n_Q$ total rounds, over the last $n_Q$ rounds of data observed from $Q$.
   	
   	\begin{thm}\label{thm:multiple-shifts}
Let $\pi$ denote the procedure of Algorithm~\ref{algo:one-arm}, ran, with parameter $\delta \in (0, 1)$, up till time $n>n_P \geq 0$, with $n_P, N, \{n_j\}_{j=1}^N$ all possibly unknown. Suppose the marginal of the covariate $X$ under each $P_j$ has unknown transfer exponent $\gamma_j$ w.r.t. $Q_X$, that $Q_X$ has $(C_d,d)$ box dimension and that the average reward function $f$ satisfies a margin condition with unknown $\alpha$ under $Q_X$.
Let $n_Q \doteq n-n_P$ denote the (possibly unknown) number of rounds after the drifts, i.e., over the phase $X_t\sim Q_X$. Let $\overline{\gamma}=\sum_{j=1}^N \gamma_j\cdot \frac{n_j}{n_P}$. We have for some constant $C>0$:
\[
		\mathbb{E}\  \textbf{R}_n^Q(\pi)\leq Cn_Q\left[\min\left(\left(\frac{K\log\left(K/\delta\right)}{n_P}\right)^{\frac{\alpha+1}{2+\alpha+d+\overline{\gamma}}},\left(\frac{K\log\left(K/\delta\right)}{n_Q}\right)^{\frac{\alpha+1}{2+\alpha+d}}\right) + \frac{K\log\left(K/\delta\right)}{n_Q} + n\delta\right]
	\]
\end{thm}

\begin{proof}[Proof Outline]
    The proof is almost identical to that of Theorem~\ref{thm:adaptive-algo-one-arm}. Let $r_t^{*}$ be the oracle level or the smallest level of ${\cal R}$ greater than or equal to
    \[
        \min\left(\left(\frac{\log(K/\delta)}{n_P}\right)^{\frac{1}{2+\alpha+d+\overline{\gamma}}},\left(\frac{\log(K/\delta)}{\tau}\right)^{\frac{1}{2+\alpha+d}}\right),
    \]
    where, recall the notation $\tau \doteq t - n_P - 1$. Consider a round $t$ for $\tau > K\log(K/\delta)$. For a fixed level $r\in{\cal R}$, let $n_P(r)$ be the covariate count in $T_r(X_t)$ from any of the distribution $P_j$ for $j\in[N]$ or
	\[
			n_P(r) \doteq \sum_{s\in [n_P]} \mathbbm{1}\{ X_s \in T_r(X)\}.
	\]
	Then, similar to Proposition~\ref{prop:optimal-bias-variance}, we have that, at round $t$, with probability at least $1-\delta$ with respect to the conditional distribution ${\bf X}_{t-1}|X_t$:
	\[
	    \phi_t(r_t) \leq 2\sqrt{\frac{8K\log(K/\delta)}{\max(\tau Q_X(T_{r_t^{*}}(X_t)),\mathbb{E}[n_P(r_t^{*})])}} + 2\lambda r_t^{*}.
	  \]
	  Next, observe that if $P_j^X$ is the covariate marginal of $P_j$, then:
	  \[
	    \mathbb{E}[n_P(r_t^{*})] = n_P\cdot \left[ \sum_{j=1}^N \frac{n_j}{n_P}\cdot P_j^X(T_{r_t^{*}}(X_t))\right].
	   \]
	   Then, following the same steps and notation as the proof of Theorem~\ref{thm:adaptive-algo-one-arm}, let 
	   \[
	    a(X_t) \doteq \sqrt{\frac{1}{\sum_{j=1}^N \frac{n_j}{n_P}\cdot P_j^X(T_{r_t^{*}}(X_t))}}.
	   \]
	   Then, using the definition of the transfer exponent (Definition~\ref{assumption:transfer-exponent}) and Definition~\ref{defn:bcn}, we have an analogous inequality to that used in deriving \eqref{eqn:p-integral} in Proposition~\ref{prop:tail-bound}
	   \begin{align*}
	    \mathbb{E}\left[ a^2(X_t) \right] &= \mathbb{E}\left[ \frac{1}{\sum_{j=1}^N \frac{n_j}{n_P} P_j^X(T_{r_t^{*}}(X_t))} \right] \\
	    &\leq c_{12} \mathbb{E}\left[ \frac{1}{Q_X(T_{r_t^{*}}(X_t))\sum_{j=1}^N \frac{n_j}{n_P} (r_t^{*})^{\gamma_j}} \right]\\
	    &\leq c_{13} \frac{1}{\sum_{j=1}^N \frac{n_j}{n_P} (r_t^{*})^{\gamma_j}} \cdot (r_t^{*})^{-d}.
	   \end{align*}
	Next, since the function $x\mapsto r^x$ is convex for any $r>0$, by Jensen's inequality we have \[
		\sum_{j=1}^N \frac{n_j}{n_P}\cdot (r_t^{*})^{\gamma_j} \geq (r_t^{*})^{\frac{1}{n_P}\sum_{j=1}^N n_j\gamma_j} = (r_t^*)^{\overline{\gamma}}.
	\]
	Thus, $\mathbb{E}[a^2(X_t)] \leq c_{13} (r_t^{*})^{-d-\overline{\gamma}}$. This yields essentially the same bound as \eqref{eqn:p-integral}, except $\overline{\gamma}$ replaces $\gamma$. Using this in place of \eqref{eqn:p-integral} in the proof of Theorem~\ref{thm:adaptive-algo-one-arm}, we obtain the result.
\end{proof}

    \section{Lower Bound}
    \label{app:lower-bound}
    
    \ifx
    {\color{blue}
    \begin{defn}
    In what follows, let $a_n, b_n$ denote bounded sequences in $[0, 1]$, indexed over $n\in \mathbb{N}$. 
   An {\bf Online to batch Conversion Rate} is a mapping $F$ from 
    sequences $\{a_n\}$ to sequences $\{b_n\}$ such that the following holds: 

If there exists an online learner $\pi$ which, with respect to an i.i.d. sequence $\{X_i, Y_i\}_{i =1}^n$ achieves regret $\mathbb{E} R_{0, n}(\pi) \leq n\cdot a_n$ for some sequence $\{a_n\}$, then there exists a classifier $\hat h$ (with training sample ${X_i, Y_i}_{i =1}^n$ ) with excess 0-1 risk 
$\mathbb{E} \, {\cal E}(\hat h) \leq F(a_n)$. 

Then, for any sequence $\{b_n\}$, we can define the \emph{pseudo-inverse} 
$F^{\dagger}(b_n) \doteq \sup\{a_n: F(a_n) \leq b_n \}$. 
    \end{defn}
    
    \begin{defn}
    In what follows, let $a \doteq \{a_n\}, b \doteq \{b_n\}$ denote bounded sequences in $[0, 1]$, indexed over $n\in \mathbb{N}$. 
   An {\bf Online to batch Conversion Rate} is a mapping $F$ from 
    sequences $a \mapsto b$ such that the following holds: 

If there exists an online learner $\pi$ which, with respect to an i.i.d. sequence $\{X_i, Y_i\}_{i =1}^n$ achieves regret $\mathbb{E} R_{0, n}(\pi) \leq n\cdot a_n$ for some sequence $a$, then there exists a classifier $\hat h$ (with training sample ${X_i, Y_i}_{i =1}^n$ ) with excess 0-1 risk 
$\mathbb{E} \, {\cal E}(\hat h) \leq (F(a))_n$. 

Then, for any $b = \{b_n\}$, define the \emph{pseudo-inverse} 
$F^{\dagger}(b) \doteq \sup\{a \doteq \{a_n\}: (F(a))_n \leq b_n \}$, where the $\sup$ over a set of sequences is defined pointwise over $n \in \mathbb{N}$.   
    \end{defn}

    \begin{lemma}[Minimax Conversion]
Suppose $\{b_n\}$ denotes a minimax lower-bound for classification. Then, if there exists an online to batch conversion rate $F$, we then have that $\{ F^\dagger(b_n)\}$ is a minimax lower-bound for online learning (hence for contextual bandits).
    \end{lemma}
     }
     \fi
     
     Here, we establish that the bound of Theorem~\ref{thm:adaptive-algo-one-arm} is minimax optimal, up to $\log$ terms in the case where $K=2$, over a continuum of regimes of choices of $n_P, n_Q, \gamma, \alpha $. 
     
     Our strategy is to use {\em online-to-batch} conversion to convert an online algorithm with regret $R_{n_P,n}$ during the last $n_Q$ rounds to a classifier with excess risk of order $R_{n_P,n}/n_Q$. This then implies a conversion from classification lower-bounds to bandits lower-bounds. 
     
     We note that \emph{online-to-batch} conversion results -- which we call as a black-box -- are usually given 
     for i.i.d. sequences of covariate-reward pairs, while we instead consider a setting with a shift in distribution $P\to Q$. Therefore, in much of what follows, we treat the first phase 
     $\{(X_t, Y_t)\}_{t = 1}^{n_P} \sim P^{n_P}$ as a separate input randomness $Z$, and apply conversion arguments to the second phase $\{(X_t, Y_t)\}_{t = n_P +1}^n\sim Q^{n_Q}$.  
     
     First, we claim a bandit policy $\pi$ can be converted to an online classification algorithm where $\pi_t\in [K]$ indicates the predicted label for covariate $X_t$. This requires defining a reward $Y^i$ for each label $i\in[K]$, which is done in Definition~\ref{defn:conversion-labels-rewards} below. To simplify notation, we will denote the set of $K=2$ arms as $\{0,1\}$.
     
     \begin{defn}[Conversion from Labels to Rewards]\label{defn:conversion-labels-rewards}
     In the case of binary classification with covariate $X \in {\cal X}$ and label ${\tilde Y} \in \{0, 1\}$, we define the reward of arm $i\in\{0,1\}$ as $Y^i \doteq \mathbbm{1}\{\tilde{Y}=i\}$. We use ${\cal T}$ to denote a class of tuples $(P,Q)$ of distributions on the covariate-label pair $(X,\tilde{Y})$. Each distribution on $(X,\tilde{Y})$ then induces a distribution on the covariate-reward pair $(X,Y)$. Let ${\cal T}'$ be the class of tuples of distributions on $(X,Y)$, induced by ${\cal T}$.
        
    To simplify notation, in what follows, tuples $(P,Q)$ will refer to either tuples in ${\cal T}$ or their one-to-one mapping to tuples in ${\cal T}'$, which will be clear from context.
    
    We will also let $\{(X_t,\tilde{Y}_t)\}_{t=1}^m$ be a sequence of covariate-label pairs and let $\{(X_t,Y_t)\}_{t=1}^m$ be the sequence of corresponding covariate-reward pairs.
     \end{defn}
     
    In this constructed bandits problem, the regression function of arm $i$ is $f^i(x)=\mathbb{P}(\tilde{Y}=i|X=x)$. Next, let $h^*(x)\doteq \mathbbm{1}\{ f^1(x)\geq 1/2\}=\pi^*(x)$ be the Bayes classifier; the excess risk of a classifier $h$ w.r.t. distribution $Q$ is then given as:
	\[
		{\cal E}_Q(h) \doteq \mathbb{E}_Q\left[ \mathbbm{1}\{h(X)\neq \tilde{Y}\}-\mathbbm{1}\{h^*(X)\neq \tilde{Y}\}\right].
	\]
	Consider an arbitrary online learner $\Lambda=\Lambda(Z)$, based on additional randomness $Z$ independent of the training data. We let $\Lambda_1,\Lambda_2,\ldots$ denote the sequentially generated classifiers of $\Lambda$. We also define the {\em mistake count} ${\cal M}_m(\Lambda)$ over $m$ rounds of $\Lambda$ as:
	\[
	    {\cal M}_m(\Lambda) \doteq \sum_{t=1}^m  \mathbbm{1}\{\Lambda_t(X_t)\neq \tilde{Y}_t\} - \mathbbm{1}\{h^*(X_t)\neq \tilde{Y}_t\}.
	\]
	In expectation, we have that the mistake count ${\cal M}_m(\Lambda)$ is equal to the regret ${\bf R}_n^Q(\pi)$ of a bandits policy $\pi$ when $\Lambda$ is the online learner induced by $\pi$, via the conversion of Definition~\ref{defn:conversion-labels-rewards}. Thus, in lower bounding ${\cal M}_m(\Lambda)$, we obtain a lower bound on the regret.
	
	
	
	The next few definitions and results will be stated in terms of an arbitrary online learner $\Lambda$ and, in Corollary~\ref{cor:lower-bound}, we will specialize $\Lambda$ to the online learner induced by policy $\pi$.
	
	First, we formalize the types of black-box guarantees on online-to-batch conversion our arguments will rely on. In what follows, let $\textbf{X}_m \doteq \{X_t\}_{t=1}^m$ and $\textbf{Y}_m \doteq \{Y_t\}_{t=1}^m$.
     
     \begin{defn}
    In what follows, let $a \doteq \{a_{m}\}, b \doteq \{b_{m}\}$ denote bounded sequences in $[0, 1]$, indexed over $m\in \mathbb{N}$. 
   An {\bf online to batch conversion rate} is a mapping $F$ from 
    sequences $a \mapsto b$ such that the following holds: 
    
    If there exists an online learner $\Lambda=\Lambda(Z)$, for additional randomness $Z$, which achieves expected mistake count $\mathbb{E}_{Z, \textbf{X}_{m},\textbf{Y}_{m}} \, ({\cal M}_m(\Lambda)) \leq  m\cdot a_{m}$ for some sequence $a$, then there exists a classifier $\hat h=\hat{h}(\Lambda)$ with excess risk 
$\mathbb{E}_{Z, \textbf{X}_{m},\textbf{Y}_{m}} \, ( {\cal E}_Q (\hat h) ) \leq (F(a))_{m}$.

Now, for any $b = \{b_{m}\}$, define the {\bf pseudo-inverse} 
$F^{\dagger}(b) \doteq \inf\{a \doteq \{a_{m}\}: (F(a))_{m} > b_{m} \}$, where the $\inf$ over a set of sequences is defined pointwise over $m\in \mathbb{N}$ (that is,  $(F^{\dagger}(b))_m=\inf\{a_m:(F(a))_m>b_m\}$ for $m\in\mathbb{N}$).
    \end{defn}
    
	Next, we formally define the notion of a minimax lower bound for offline and online classification problems in terms of a rate $\{a_m\}$.

    \begin{defn}
    	Fix $n_P \in \mathbb{N}$. We say that the class ${\cal T}$ (of distribution pairs $(P, Q)$)  has a {\bf classification minimax lower bound of $b = \{b_{m}\}$} if the following holds: 
    	
    	For any $m\in\mathbb{N}$ and any classifier $\hat{h}$ learned on data $\{(X_t,\tilde{Y}_t)\}_{t=1}^{m}\sim Q^m$, and additional randomness $Z = \{(X_t', \tilde{Y}_t')\}_{t = 1}^{n_P} \sim P^{n_P}$, 
    	\[
    		\sup_{(P,Q)\in {\cal T}} \mathbb{E}_{Z,\textbf{X}_{m},\textbf{Y}_{m}} [{\cal E}_Q (\hat{h})] > b_{m}.
    	\]    
    	Similarly, a class ${\cal T}$ has an {\bf online minimax lower bound of $a = \{a_{m}\}$} if the following holds:
    	
    	For any $m\in\mathbb{N}$ and any online learner $\Lambda=\Lambda(Z)$ trained on data $\{(X_t,\tilde{Y}_t)\}_{t=1}^{m}\sim Q^m$, and additional randomness $Z=\{(X_t',\tilde{Y}_t')\}_{t=1}^{n_P} \sim P^{n_P}$, we have:
    	\[
    		\sup_{(P,Q)\in {\cal T}}\mathbb{E}_{Z,\textbf{X}_m,\textbf{Y}_m} \, [{\cal M}_{m}(\Lambda)] > m\cdot a_{m}.
    	\]
    \end{defn}
    
     Given an online to batch conversion rate, the next lemma allows us to deduce an online minimax lower bound from a classification minimax lower bound.
    
    \begin{lemma}[Minimax Lower Bound Conversion]\label{lem:minimax-conversion}
Suppose $b \doteq \{b_{m}\}$ denotes a classification minimax lower bound for the class ${\cal T}$. Then, if there exists an online to batch conversion rate $F$ with $(F^{\dagger}(b))_m > 0$ for all $m\in\mathbb{N}$, we have that $\frac{1}{2}\cdot F^\dagger(b)$ is an online minimax lower bound for the class ${\cal T}$.
    \end{lemma}
    
    \begin{proof}
    	Consider an online learner $\Lambda=\Lambda(Z)$, with additional randomness $Z=\{(X_t',\tilde{Y}_t')\}_{t=1}^{n_P}$, with mistake count $a = \{a_m\}$. For contradiction, suppose there exists $m\in\mathbb{N}$ such that:
    	\[
    	    \sup_{(P,Q)\in{\cal T}} \mathbb{E}_{Z,\textbf{X}_m,\textbf{Y}_m} \, [ {\cal M}_{m}(\Lambda)] \leq  m\cdot a_{m} \leq m\cdot\frac{1}{2}\cdot (F^{\dagger}(b))_m < m\cdot (F^\dagger(b))_{m}.
    	\]
    	Then, by the definition of $F$ and the pseudo-inverse $F^{\dagger}$, there exists a classifier $\hat{h}=\hat{h}(\Lambda)$ such that:
    	\[
    		\sup_{(P,Q)\in {\cal T}} \mathbb{E}_{Z,\textbf{X}_m,\textbf{Y}_m} \, [{\cal E}_Q (\hat h)] \leq (F(a))_m \leq b_{m}.
    	\]
    	This is a contradiction on $b$ being a classification minimax lower bound for the class ${\cal T}$.
    	\end{proof}  
	
	We next specify the online to batch conversion rate $F$ that we will use with Lemma~\ref{lem:minimax-conversion}.

\begin{thm}[Theorem 4 of \citep{cesa-bianchi2004online}, paraphrased]\label{thm:converter}
	Let $\Lambda=\Lambda(Z)$ be an arbitrary online learner, trained on $\{(X_t,\tilde{Y}_t)\}_{t=1}^m$, with additional randomness $Z$. Then, for any $\delta\in(0,1]$, there exists a classifier $\hat{h}=\hat{h}(\Lambda)$, trained on $\{(X_t,\tilde{Y}_t)\}_{t=1}^{m}$, such that:
	\[
		\mathbb{P}\left(\mathbb{P}_{(X,\tilde{Y})\sim Q}(\hat{h}(X)\neq \tilde{Y}) \geq \frac{1}{m}\sum_{t=1}^{m} \mathbbm{1}(\Lambda_{t}(X_t)\neq \tilde{Y}_t) + 6\sqrt{\frac{\log\left(2(m+1)/\delta\right)}{m}} \;\Bigg\vert\; Z \right)\leq\delta.
	\]
\end{thm}

\begin{cor}\label{cor:converter}
	Let $\Lambda=\Lambda(Z)$ be an online learner trained on data $\{(X_t,\tilde{Y}_t)\}_{i=1}^{m}$ with additional input $Z$. Then, there exists a classifier $\hat{h}=\hat{h}(\Lambda)$ such that for any distribution on $\textbf{X}_m,\textbf{Y}_m,Z$:
	\[
		\mathbb{E}_{Z,\textbf{X}_m,\textbf{Y}_m}\, [{\cal E}_Q(\hat{h})] \leq \frac{\mathbb{E}_{Z,\textbf{X}_m,\textbf{Y}_m}\, [ {\cal M}_{m}(\Lambda) ]}{m}  + 6\sqrt{\frac{2\log(m\sqrt{3})}{m}} + \frac{1}{m}.
	\]
\end{cor}

\begin{proof}
	Fix a value of $Z$ and let the event $A$ be as in Theorem~\ref{thm:converter}:
	\[
		A=\left\{\mathbb{P}_{(X,\tilde{Y})\sim Q}(\hat{h}(X)\neq \tilde{Y}) \geq \frac{1}{m}\sum_{t=1}^{m} \mathbbm{1}(\Lambda_{t}(X_t)\neq \tilde{Y}_t) + 6\sqrt{\frac{\log\left({2(m+1)}/{\delta}\right)}{m}} \;\Bigg\vert\; Z\right\}.
	\]
    Then, letting $\delta=1/m$ and conditioning on the event $A$, we have:
	\begin{align*}
		\mathbb{E}_{\textbf{X}_m,\textbf{Y}_m}\, [ {\cal E}_Q(\hat{h}) \mid Z ] &\leq \mathbb{E}_{\textbf{X}_m,\textbf{Y}_m} \, \left[ \frac{1}{m}\sum_{t=1}^{m} \mathbbm{1}(\Lambda_{t}(X_t)\neq \tilde{Y}_t)  - \mathbbm{1}(h^*(X_t)\neq \tilde{Y}_t) \;\Bigg\vert\; Z\right] \\
		&\qquad + 6\sqrt{\frac{2\log\left(m\sqrt{3}\right)}{m}}+\frac{1}{m}\\
		&= \frac{\mathbb{E}_{\textbf{X}_m,\textbf{Y}_m} \, [{\cal M}_{m}(\Lambda) \mid Z]}{m}+6\sqrt{\frac{2\log\left(m\sqrt{3}\right)}{m}}+\frac{1}{m}.
	\end{align*}
	Taking a further expectation with respect to $Z$ on both sides of the inequality gives the desired result.
\end{proof}

Theorem 1 of \citet{kpotufe-martinet} provides us the classification minimax lower-bound, which we restate it here.

\begin{thm}[Theorem 1 of \citet{kpotufe-martinet}]\label{thm:classification-lower-bound}
Let ${\cal T}'$ be the class of all tuples $(P,Q)$ of distributions satisfying Assumption~\ref{assumption:lipschitz} and Definition~\ref{defn:bcn}, and Definitions~\ref{assumption:margin} and \ref{assumption:transfer-exponent}, with some fixed parameters $(\lambda,C_d,d,C_{\alpha},\alpha,\delta_0,C_{\gamma},\gamma)$. In what follows, let ${\cal T}$ be the one-to-one mapping of ${\cal T}'$ to tuples of distributions on covariate-label pairs as in Definition~\ref{defn:conversion-labels-rewards}. Then, there exists a constant $c>0$ such that for any $n_P,n_Q\in\mathbb{N}$ and classifier $\hat{h}$ learned on $\{(X_t,\tilde{Y}_t)\}_{t=1}^{n_P}\sim P^{n_P}$ and $\{(X_t,\tilde{Y}_t)\}_{t=n_P+1}^{n_P+n_Q}\sim Q^{n_Q}$, we have:
\[
    \sup_{(P,Q)\in\mathcal{T}} \mathbb{E} \, \left( {\cal E}_Q (\hat{h}) \right) > c \left(n_P^{\frac{2+\alpha+d}{2+\alpha+d+\gamma}} + n_Q\right)^{-\frac{\alpha + 1}{2+\alpha+d}}.
\]
\end{thm}

Next, we will take $b_{m} \doteq c \left(n_P^{\frac{2+\alpha+d}{2+\alpha+d+\gamma}} + m\right)^{-\frac{\alpha + 1}{2+\alpha+d}}$ as our classification minimax lower-bound where $m$ here stands for $n_Q$. Combining Lemma~\ref{lem:minimax-conversion}, Corollary~\ref{cor:converter}, and Theorem~\ref{thm:classification-lower-bound}, we obtain the following minimax lower bound for bandits:

\begin{cor}[Matching Lower Bounds over Given Regimes]\label{cor:lower-bound}
Let the class ${\cal T}'$ and the constant $c>0$ be as in Theorem~\ref{thm:classification-lower-bound}. Suppose that $n_P,n_Q$ satisfy:
\begin{equation}\label{eqn:lower-bound-condition}
	6\sqrt{\frac{2\log(n_Q\sqrt{3})}{n_Q}} + \frac{1}{n_Q} < \frac{c}{2} \left(n_P^{\frac{2+\alpha+d}{2+\alpha+d+\gamma}} + n_Q\right)^{-\frac{\alpha + 1}{2+\alpha+d}}.
	\end{equation}
	Then, for any fixed such $n_P,n_Q$ and any contextual bandits policy $\pi$, we have:
\[
    \sup_{(P,Q)\in\mathcal{T}'} \mathbb{E}_{\textbf{X}_n,\textbf{Y}_n} \, \left[ \textbf{R}_{n}^Q(\pi)\right] \geq \frac{c}{4} n_Q\left(n_P^{\frac{2+\alpha+d}{2+\alpha+d+\gamma}} + n_Q\right)^{-\frac{\alpha + 1}{2+\alpha+d}}.
\]
\end{cor}

\begin{proof}
	Fix $n_P,n_Q$ satisfying the inequality in \eqref{eqn:lower-bound-condition} and let $n=n_P+n_Q$. Let $\pi_O$ be the online learner induced by a policy $\pi$ restricted to the $Q$ phase $\{(X_t,Y_t)\}_{t=n_P+1}^{n}$ with additional randomness $Z=\{(X_t,Y_t)\}_{t=1}^{n_P}$. Then, by Corollary~\ref{cor:converter}, we have there exists a classifier $\hat{h}$ such that:
	\[
	   \mathbb{E}_{\textbf{X}_{n},\textbf{Y}_{n}}  [{\cal E}_Q(\hat{h})] \leq \frac{\mathbb{E}_{\textbf{X}_n,\textbf{Y}_n} \, [ {\cal M}_{n_Q}(\pi_O) ]}{n_Q} + 6\sqrt{\frac{2\log(n_Q\sqrt{3})}{n_Q}} + \frac{1}{n_Q}.
	\]
	We then have that the map $F$, defined below on a sequence $a=\{a_m\}$, is an online to batch conversion rate:
	\[
		(F(a))_{m} \doteq a_{m} + 6\sqrt{\frac{2\log(m\sqrt{3})}{m}} + \frac{1}{m},
	\]
	Let $b_{m} \doteq c \left(n_P^{\frac{2+\alpha+d}{2+\alpha+d+\gamma}} + m\right)^{-\frac{\alpha + 1}{2+\alpha+d}}$ be as in Theorem~\ref{thm:classification-lower-bound}. Then, by Theorem~\ref{thm:classification-lower-bound} and Lemma~\ref{lem:minimax-conversion}, we have:
	\[
	    \sup_{(P,Q)\in{\cal T}} \mathbb{E}_{\textbf{X}_{n},\textbf{Y}_{n}} [{\cal M}_{n_Q}(\pi_O)] \geq \frac{n_Q \cdot (F^{\dagger}(b))_{n_Q}}{2}.
	\]
	Next, we observe:
    \begin{align*}
    		\mathbb{E}_{\textbf{X}_n,\textbf{Y}_n} [{\cal M}_{n_Q}(\pi_O)] &= \mathbb{E}_{\textbf{X}_n,\textbf{Y}_n}\, \left(\sum_{t=1}^{n_Q}  \mathbbm{1}(\pi_{O,t}(X_t)\neq \tilde{Y}_t) - \mathbbm{1}(h^*(X_t)\neq \tilde{Y}_t)\right)\\
    		&= \mathbb{E}_{\textbf{X}_n,\textbf{Y}_n} \left( \sum_{t=1}^{n_Q} Y_t^{\pi_t^*(X_t)} - Y_t^{\pi_t(X_t)}\right)\\
    		&= \mathbb{E}_{\textbf{X}_n,\textbf{Y}_n} \left[ \textbf{R}_{n}^Q(\pi)\right].
	\end{align*}	
	Thus,
	\begin{align*}
	      \sup_{(P,Q)\in\mathcal{T}'} \mathbb{E}_{\textbf{X}_{n},\textbf{Y}_{n}} \, \left[ \textbf{R}_{n}^Q(\pi)\right] &\geq \frac{n_Q\cdot (F^{\dagger}(b))_{n_Q}}{2} \\
	      &\geq \frac{n_Q}{2}\cdot \left(b_{n_Q} - 6\sqrt{\frac{2\log(n_Q\sqrt{3})}{n_Q}} - \frac{1}{n_Q}\right)\\
		&\geq \frac{n_Q}{2} \cdot \frac{c}{2}\left(n_P^{\frac{2+\alpha+d}{2+\alpha+d+\gamma}} + n_Q\right)^{-\frac{\alpha + 1}{2+\alpha+d}}.
	\end{align*}
\end{proof}


\begin{rmk}\label{rmk:lower-bound-regimes}
	The inequality in \eqref{eqn:lower-bound-condition} corresponds to the regime 
	$n_P={\tilde O }\left(n_Q^{\frac{2+\alpha+d+\gamma}{2 + 2\alpha}}\right)$ with $\alpha < d$. In particular this includes the following subregimes. 
	
	\begin{itemize} 
	\item {\bf Performance on $Q$ depends mostly on covariates $X_t \sim Q_X, t> n_P$ .}
	This is the subregime where $n_P \lesssim n_Q^{\frac{2 + \alpha + d + \gamma }{2 + \alpha + d }} $, roughly, that is when 
	(in the upper-bound of Theorem  \ref{thm:adaptive-algo-one-arm}) $\min(n_P^{-\frac{\alpha+1}{2+\alpha+d+\gamma}},n_Q^{-\frac{\alpha+1}{2+\alpha+d}})=n_Q^{-\frac{\alpha+1}{2+\alpha+d}}$, i.e., \emph{past experience} under $P$ is too short to significantly influence regret under $Q$. The lower-bound of Corollary \ref{cor:lower-bound} then is of the form 
	$$ \sup_{(P,Q)\in\mathcal{T}'} \mathbb{E} \left[ \textbf{R}_{n}^Q(\pi)\right] \gtrsim n_Q \cdot n_Q^{-\frac{\alpha+1}{2+\alpha+d}},$$
which confirms that the threshold $n_P = {\tilde O}\left (n_Q^{\frac{2+ \alpha + d + \gamma }{2 + \alpha + d }} \right)$ (on when past experience is too short) is indeed tight. 
	
	\item{\bf Performance on $Q$ depends mostly on covariates $X_t \sim P_X, t \leq n_P$.} 
	This is the subregime where $$ n_Q^{\frac{2+ \alpha + d + \gamma }{2+\alpha+d }}\lesssim n_P \lesssim  n_Q^{\frac{2+\alpha+d+\gamma}{2 + 2\alpha}}.$$
	In other words $\min(n_P^{-\frac{\alpha+1}{2+\alpha+d+\gamma}},n_Q^{-\frac{\alpha+1}{2+\alpha+d}})=n_P^{-\frac{\alpha+1}{2+\alpha+d+ \gamma}}$, i.e., \emph{past experience} under $P$ significantly influences regret under $Q$. The lower-bound of Corollary \ref{cor:lower-bound} then is of the form 
	$$ \sup_{(P,Q)\in\mathcal{T}'} \mathbb{E} \, \left[ {\bf R}_{n}^Q(\pi)\right] \gtrsim n_Q \cdot n_P^{-\frac{\alpha+1}{2+\alpha+d+\gamma}},$$
that is, the upper-bounds of Theorem \ref{thm:adaptive-algo-one-arm} are again tight up to $\log$ factors. 
	\end{itemize} 
	
\end{rmk}

    \label{app:bm}
	\section{Bounded Mass Assumption}
	
	Here, we consider the {\em strong density condition} of Remark~\ref{rem:strongdensityand margin}, differing from our more relaxed condition on the marginal distribution $Q_X$ in Definition~\ref{defn:bcn}. The strong density assumption ensures that $Q_X$ has good coverage of $[0, 1]^d$. It holds, for instance, if $Q_X$ has lower-bounded Lebesgue density on $[0, 1]^d$. We note that, unlike the notion of support dimension used in previous sections (Definition~\ref{defn:bcn}), the ``dimension'' of our support now coincides with the ambient dimension, denoted $D$ in previous sections. This is defined formally as follows:

\begin{assumption}[Mass under $Q$]\label{assumption:dimension}
	\sk{$\exists\ C_d>0$ s.t., $\forall$ $\ell_\infty$ balls $B\subset [0, 1]^d$ of diameter $r\in (0, 1]$:} 
	\begin{align*}
	Q_X(B) \geq C_d \cdot r^d.
    \end{align*}
\end{assumption}

	Under this assumption, we obtain similar regret rates as Theorem~\ref{thm:adaptive-algo-one-arm} with $d$ now being defined as in Assumption~\ref{assumption:dimension}. However, as discussed in Remark~\ref{rmk:strong-density-margin}, the analysis differs heavily from that of Theorem~\ref{thm:adaptive-algo-one-arm} and is closer in spirit to that of \citet{perchet-rigollet}. We first consider the case of a single shift $P\to Q$ and handle multiple shifts in Theorem~\ref{thm:multiple-shifts-dm}, by a similar extension to that made in Appendix~\ref{app:multiple-shifts}.
	
	\begin{thm}\label{thm:adaptive-algo-one-arm-dm}
Let $\pi$ denote the procedure of Algorithm~\ref{algo:one-arm}, ran, with parameter $\delta \in (0, 1)$, up till time $n>n_P \geq 0$, with $n_P$ possibly unknown. Suppose $P_X$ has unknown transfer exponent $\gamma$ w.r.t. $Q_X$, that $Q_X$ satisfies Assumption~\ref{assumption:dimension} with $(C_d,d)$, and that the average reward function $f$ satisfies a margin condition with unknown $\alpha$ under $Q_X$.
Let $n_Q \doteq n-n_P$ denote the (possibly unknown) number of rounds after the drift, i.e., over the phase $X_t\sim Q_X$. We have for some constant $C>0$:
\[
		\mathbb{E}\  \textbf{R}_n^Q(\pi)\leq Cn_Q\left[\min\left(\left(\frac{K\log\left(K/\delta\right)}{n_P}\right)^{\frac{\alpha+1}{2+d+\gamma}},\left(\frac{K\log\left(K/\delta\right)}{n_Q}\right)^{\frac{\alpha+1}{2+d}}\right) + \frac{K\log\left(K/\delta\right)}{n_Q} + n\delta\right]
	\]
\end{thm}

\begin{cor}\label{cor:one-arm-dm}
	Under the setup of Theorem~\ref{thm:adaptive-algo-one-arm}, letting $\delta=O(1/n^2)$ yields: \[
		\mathbb{E}\  \textbf{R}_{n}^Q(\pi)\leq Cn_Q  \left[\min\left(\left(\frac{K\log(Kn)}{n_P}\right)^{\frac{\alpha+1}{2+d+\gamma}},\left(\frac{K\log(Kn)}{n_Q}\right)^{\frac{\alpha+1}{2+d}}\right) + \frac{K \log(Kn)}{n_Q}\right].
	\]
\end{cor}

\subsection{Outline of the Proof of Theorem~\ref{thm:adaptive-algo-one-arm-dm}}

The first part of the proof is nearly identical to the proof of Theorem~\ref{thm:adaptive-algo-one-arm}. Lemmas~\ref{lem:bias-variance}, \ref{lem:no-skip}, and \ref{lem:arm-count-randomization} remain true without modification. Thus, we have, at round $t$ with chosen bin $B \doteq T_{r_t}(X_t)$, that with probability at least $1-\delta$ with respect to the distribution ${\bf Y}_{t-1}|{\bf X}_t$:
\[
	\forall x\in B,i \in {\cal I}_B: |\hat{f}_t^i(B) - f^i(x)| \leq 4\lambda r_t.
\]
Additionally, Proposition~\ref{prop:better-arm} and Corollaries~\ref{cor:best-arm-candidate} and \ref{cor:classification-hard} still hold so that the regret and margin at round $t$ are both bounded by $16\lambda r_t$. Thus, following the intuition of Remark~\ref{rmk:strong-density-margin}, it suffices to show $r_t$ is of optimal regression order.

    \paragraph{Showing $r_t$ is of Optimal Regression Order.}
	\label{subsec:optimal-order}
	
	As in Section~\ref{sec:regret-analysis}, for each round $t$, define the event $G_t$ on which the high-probability bound on the regression error holds or
	\[
		G_t = \left\{\forall i \in \mathcal{I}_B, x \in B:|\hat{f}_t^i(B) - f^i(x)|\leq 4\lambda r_t, B=T_{r_t}(X_t)\right\}.
    \]
    Recall that this is the ``good'' event on which we can identify unfavorable arms. Recall also, from Section~\ref{sec:regret-analysis}, that we define $\tau \doteq t - 1 - n_P$ to simplify notation.
	
	\begin{lemma}\label{lem:acb-bound}
		Fix a round $t$ with observed covariate $X_{t}$. Then, for some $c_{14}>0$, with probability at least $1-\delta$ w.r.t. the distribution of $\textbf{X}_{t-1}|X_{t}$, we have
		\begin{align*}
			 r_{t} &\leq c_{14} \min\left(
			 	\left(\frac{K\log(K/\delta)}{n_P}\right)^{\frac{1}{2+d+\gamma}}, \left(\frac{K\log(K/\delta)}{\tau}\right)^{\frac{1}{2+d}}\right).
		\end{align*}
	\end{lemma}
	
	\begin{proof}
	It suffices to show 
		\begin{align*}
			 r_{t} \leq c_{14} \begin{cases}
			 	\left(\frac{K\log(K/\delta)}{n_P}\right)^{\frac{1}{2+d+\gamma}} & \text{ when }n_P > K\log(K/\delta)\\ 						\left(\frac{K\log(K/\delta)}{\tau}\right)^{\frac{1}{2+d}} & \text{ when }\tau > K\log(K/\delta)\end{cases},
		\end{align*}
		since this implies the desired result for any $n_P,\tau$. We first show that when $\tau > K\log(K/\delta)$: 
		\[
			r_{t} \leq c_{14}\left(\frac{K\log(K/\delta)}{\tau}\right)^{\frac{1}{2+d}}.
		\]
		The other bound on $r_t$ involving $n_P$ will follow similarly, with the appropriate modifications. First, we simplify notation for the sake of this proof: at round $t$, it will be understood that the observed covariate is represented by $X \doteq X_{t}$. We also let $n(r) \doteq n_r(X)$ be the covariate count for {\em a fixed arbitrary} level $r$.
		
		First, by Assumption~\ref{assumption:dimension} and the definition of transfer exponent (Definition~\ref{assumption:transfer-exponent}), we have for some $c_{15}>0$:
		\begin{equation}\label{eqn7}
			\mathbb{E}[n(r)]=n_P P_X(T_r(X)) + \tau Q_X(T_r(X)) \geq c_{15} ( n_P r^{d + \gamma} + \tau r^d).
		\end{equation}	
		In fact, without loss of generality, we can assume
		\begin{equation}\label{eqn-c3}
			c_{15} \leq C_d^{\frac{2+d}{d}}.
		\end{equation}
		Thus, by a Chernoff bound similar to that used in the proof of Proposition~\ref{prop:optimal-bias-variance}, for any fixed level $r$ satisfying $n(r) \geq K\log(1/\delta)$, we have with probability at least $1-\delta$ that
		\begin{equation}\label{eqn5}
			\sqrt{\frac{K\log(K/\delta)}{n(r)}} \leq \sqrt{\frac{8K\log(K/\delta)}{c_{15} \tau r^d }}.
		\end{equation}
		Now, let $r_t^{*} \in \mathcal{R}$ be the smallest level greater than or equal to
		\[
			 c_{15}^{-\frac{1}{2+d}}\left(\frac{8K\log(K/\delta)}{\tau}\right)^{\frac{1}{2+d}}.
		\]
		To put things simply, $r_t^* \propto \left(\frac{K\log(K/\delta)}{\tau}\right)^{\frac{1}{2+d}}$.		Then, it suffices to show $r_{t} \leq r_t^{*}$. For the next part of the proof, we define $n_Q(r)$ as the covariate count coming exclusively from $Q_X$ for any given level $r$:
		\[
			n_Q(r) \doteq \sum_{s\in [\tau]} \mathbbm{1}\{X_{n_P+s} \in T_r(X)\}.
		\]		
		Then, we claim $r_t^{*}$ satisfies $n_Q(r_t^*) \geq \log(1/\delta)$ with probability at least $1-\delta$, so that \eqref{eqn5} holds with probability at least $1-2\delta$ for $r = r_t^{*}$, by the aforementioned Chernoff bound. Since $\tau > 8K\log(K/\delta)$ by hypothesis, we have by \eqref{eqn-c3} and Assumption~\ref{assumption:dimension} that
		\[
			\mathbb{E}[n_Q(r)] \geq C_d \cdot \tau \cdot (r_t^{*})^d \geq C_d \cdot \tau \cdot c_{15}^{-\frac{d}{2+d}} \left(\frac{8K\log(K/\delta)}{\tau}\right)^{\frac{d}{2+d}} \geq 8\log(1/\delta).
		\]
		Thus, by a Chernoff bound, we have
		\[
			Q_X(n_Q(r_t^{*})  < \log(1/\delta)) \leq Q_X\left( n_Q(r_t^{*}) < \frac{1}{2}\mathbb{E}[n_Q(r)] \right) \leq \exp\left(-\frac{1}{8}\mathbb{E}[n_Q(r)]\right) \leq \delta.
		\]
		Then, we have that by virtue of how $r_t^*$ is defined, with probability at least $1-2\delta$:
		\[
			 r_t^{*} \geq \sqrt{\frac{8K \log(K/\delta)}{c_{15}\tau (r_t^{*})^d}} \geq \sqrt{\frac{8K\log(K/\delta)}{n(r_t^{*})}}.
		\]
		This gives us that $r_{t} \leq r_t^{*}$ by the minimization criteria for selecting $r_{t}$ on Line 7 of Algorithm~\ref{algo:one-arm}, as desired.
		
		The other inequality
		\[
			r_{t} \leq c_{14} \left(\frac{K\log(K/\delta)}{n_P}\right)^{\frac{1}{2+d+\gamma}},
		\] 
		can be shown in an identical fashion to the case above with the appropriate modifications: specifically, $\tau$ is replaced with $n_P$, \eqref{eqn7} is further lower bounded by $c_{15} n_P r^{d+\gamma}$, and $n_Q(r)$ is replaced with $n_P(r)$ which is defined as the bin covariate counts from distribution $P$:
		\[
			n_P(r) \doteq \sum_{s\in [n_P]} \mathbbm{1}\{ X_s \in T_r(X)\}.
		\]
	\end{proof}
	
	\paragraph{Cumulative Regret Bound.}
	\label{subsec:sum-regret}

	Next, we put the previous conclusions together to bound the cumulative regret by bounding the regret accrued at each round $t$ and then summing over $t\in\{n_P+1,\ldots,n\}$. Similarly to the proof of Theorem~\ref{thm:adaptive-algo-one-arm}, for rounds $s<t$, define the event $E_s$ as the event on which the bound of Lemma~\ref{lem:bias-variance} holds or $E_s \doteq G_s$. For round $t$, define the event $E_t$ as the event on which the bounds of Lemma~\ref{lem:bias-variance} and Lemma~\ref{lem:acb-bound} hold or:
	\[
		E_t \doteq G_t\cap \left\{r_{t} \leq c_{14} \min\left(
			 	\left(\frac{K\log(K/\delta)}{n_P}\right)^{\frac{1}{2+d+\gamma}}, \left(\frac{K\log(K/\delta)}{\tau}\right)^{\frac{1}{2+d}}\right)\right\}.
	\]
	To sum the regrets across time $t$, the argument will involve conditioning on the event $\cap_{s=1}^t E_s$, on which (a) Algorithm~\ref{algo:one-arm} correctly eliminates arms and (b) $r_t$ is of the optimal order. 
	
	To this end, let $F_t \doteq \cap_{s=1}^t E_s$. Also, to simplify notation, let $U_t$ denote
\[
	U_t\doteq c_{14}
		\min\left(\left(\frac{K\log(K/\delta)}{n_P}\right)^{\frac{1}{2+d+\gamma}},\left(\frac{K\log(K/\delta)}{\tau}\right)^{\frac{1}{2+d}}\right).
\]
Recall $U_t$ is the earlier high-probability upper bound on $r_t$ in the definition of $E_t$. If $n_Q \leq K\log(K/\delta)$, we are already done since the regret is then bounded by $K\log(K/\delta)$, which is the right order. Assume for the rest of the proof that $K\log(K/\delta) < n_Q$.

To later use the margin condition (Definition~\ref{assumption:margin}), we require that $U_t \lesssim \delta_0$ (where $\delta_0$ is as in Definition~\ref{assumption:margin}). To ensure this, we need only constrain our analysis to rounds for which $\tau \gtrsim K\log(K/\delta)$, which is not an issue since the regret of any $O(K\log(K/\delta))$ rounds is of the desired order. More precisely, let $\tau_0$ be the largest positive integer satisfying
\[
	c_{16}\left(\frac{K\log(K/\delta)}{\tau_0}\right)^{\frac{1}{2+d}} > \delta_0,
\]
where $c_{16}>0$ will be determined later. The regret for the first $\tau_0$ rounds among rounds $\{n_P+1,\ldots,n\}$ is $O(K\log(K/\delta))$ which is always of the right order. For the rest of the proof, we now assume that the round $t$ is such that $\tau > \tau_0$ and $c_{16} U_t \leq \delta_0$.

Next, let the event $H_t$ be
\begin{align*}
	H_t\doteq\{|\hat{f}_t^{(1)}(B)-\hat{f}_t^{(2)}(B)|\leq 8\lambda r_t, B = T_{r_t}(X_t)\}.
\end{align*}
Conditioned on ${\bf X}_t$, $H_t$ is the event where one arm remains in contention at round $t$ according to Line 11 of Algorithm~\ref{algo:one-arm}. For the remainder of the proof, let $B$ be the bin that was selected at round $t$ given an understood value of ${\bf X}_t$. To further simplify notation, let ${\bf X} \doteq {\bf X}_{t-1}$ and ${\bf Y} \doteq {\bf Y}_{t-1}$ since we are fixing $t$ momentarily here.

Consider the expected regret of pulling arm $\pi_t$ at round $t$, and decomposing it by conditioning the events $F_t$ and $H_t$:
\begin{align*}
	\mathbb{E}\, f^{(1)}(X_t)-f^{\pi_t}(X_t) &= \mathbb{E}_{X_t} \left[ \mathbb{E}_{\textbf{X},\textbf{Y}|X_t}  (f^{(1)}(X_t)-f^{\pi_t}(X_t))\right.\left.\cdot(\mathbbm{1} \{F_t \} +\mathbbm{1}\{ F_t^c \}) \cdot (\mathbbm{1}\{ H_t \}+\mathbbm{1} \{ H_t^c \} )\right].
\end{align*}
The above gives us three different cases depending on whether event $F_t \cap H_t$, or $F_t \cap H_t^c$, or $F_t^c$ holds.
\begin{enumerate}[label=\arabic*)]
	\item Suppose event $F_t\cap H_t$ holds. Suppose also that there is a suboptimal arm $i\in\mathcal{I}_B$ for which $f^i(X_t) < f^{(1)}(X_t)$, lest the regret at time $t$ be zero. Then, by Corollary~\ref{cor:classification-hard}, we have for $c_{16} = c_{14}\cdot 16\lambda$ and any $j\in\mathcal{I}_B$:
	\[
		|f^{(1)}(X_t)-f^{j}(X_t)|\leq 16\lambda r_t \leq c_{16}U_t.
	\]
	Furthermore:
	\[
		0<|f^{(1)}(X_t)-f^{(2)}(X_t)|\leq 16\lambda r_t \leq c_{16} U_t.
	\]
	This above inequality happens with probability at most $C_{\alpha}(c_{16} U_t)^{\alpha}$, under $X_t\sim Q_X$, by the margin condition (Definition~\ref{assumption:margin}). Thus, we have
	\[
		\mathbb{E}_{X_t}\  \mathbb{E} _{\textbf{X},\textbf{Y}|X_t}\ (f^{(1)}(X_t)-f^{j}(X_t)) \cdot \mathbbm{1} \{ F_t\cap H_t \}\leq C_{\alpha}(c_{16})^{\alpha+1} U_t^{\alpha+1}.
	\]
	\item Next, on $F_t\cap H_t^c$, the pointwise regret is zero by Corollary~\ref{cor:best-arm-candidate} since $\mathcal{I}_{B}$ must contain the optimal arm at $X_t$ and no other arms.
	\item On $F_t^c$, the pointwise regret is trivially bounded above by $1$. By Lemma~\ref{lem:acb-bound}, this happens with probability at most 
	\begin{align*}
		\mathbb{P} (F_t^c)&\leq \mathbb{P}\left(\cup_{s=1}^t E_s^c\right)
		\leq 
			\sum_{s=1}^t 2\delta.
	\end{align*}
	Thus, $\mathbb{E}_{X_t}\, \mathbb{E}_{\textbf{X},\textbf{Y}|X_t}\ (f^{(1)}(X_t)-f^{j}(X_t)) \cdot \mathbbm{1}\{F_t^c \} \leq 2t\delta$.
\end{enumerate}

Next, we put the three cases above together. We have that, for some $c_{17}>0$, the cumulative regret over the rounds $\{\tau_0+1,\ldots,n\}$ is then at most
\begin{align*}\label{eqn:bound}
		c_{17} \left[  \sum_{\tau=\tau_0+1}^{n_Q} \min\left(\left(\frac{K\log\left(K/\delta\right)}{n_P}\right)^{\frac{\alpha+1}{2+d+\gamma}},\left(\frac{K\log\left(K/\delta\right)}{\tau}\right)^{\frac{\alpha+1}{2+d}}\right) +(n_P+\tau)\delta \right].
	\end{align*}
	Taking the sum over the last term on the R.H.S. above, we have $\sum_{\tau=1}^{n_Q} (n_P+\tau)\delta = O(n_Qn\delta)$. For the remaining term in the sum, it suffices to bound
	\[
		\sum_{\tau = \tau_0+1}^{n_Q} \left(\frac{K\log(K/\delta)}{\tau}\right)^{\frac{\alpha+1}{2+d}}.
	\]
	As in the proof of Theorem~\ref{thm:adaptive-algo-one-arm}, by an integral approximation, we have since $\tau_0\propto K\log(K/\delta)$:
	\begin{align*}
		\sum_{\tau=\tau_0+1}^{n_Q} \left(\frac{K\log(K/\delta)}{t}\right)^{\frac{\alpha+1}{2+d}} &\leq
		 c_{18}\int_{K\log(K/\delta)}^{n_Q} \left(\frac{K\log(K/\delta)}{z}\right)^{\frac{\alpha+1}{2+d}}\,dz
	\end{align*}
	If $\alpha \leq d + 1$, the above integral, for some $c_{18}>0$, is bounded by
	\[
		c_{19} n_Q\left(\frac{K\log(K/\delta)}{n_Q}\right)^{\frac{\alpha+1}{2+d}}.
	\]
	Otherwise, said integral is bounded by $O(K\log(K/\delta))$. This concludes the proof of Theorem~\ref{thm:adaptive-algo-one-arm-dm}. $\blacksquare$
	
	\subsection{Multiple Shifts under Bounded Mass}
	
	Here, we derive an analogue to Theorem~\ref{thm:multiple-shifts} under the bounded mass assumption and the same setup as Appendix~\ref{app:multiple-shifts}. The proof is nearly identical as that of Theorem~\ref{thm:multiple-shifts}, relying on the same bound on $r^{\sum_{j=1}^N \frac{n_j}{n_P}\cdot \gamma_j}$.
	
	\begin{thm}\label{thm:multiple-shifts-dm}
Let $\pi$ denote the procedure of Algorithm~\ref{algo:one-arm}, ran, with parameter $\delta \in (0, 1)$, up till time $n>n_P \geq 0$, with $n_P, N, \{n_j\}_{j=1}^N$ all possibly unknown (see definitions in Appendix~\ref{app:multiple-shifts}). Suppose the marginal of the covariate $X$ under each $P_j$ has unknown transfer exponent $\gamma_j$ w.r.t. $Q_X$, that $Q_X$ satisfies Assumption~\ref{assumption:dimension} with $(C_d,d)$, and that the average reward function $f$ satisfies a margin condition with unknown $\alpha$ under $Q_X$.
Let $n_Q \doteq n-n_P$ denote the (possibly unknown) number of rounds after the drifts, i.e., over the phase $X_t\sim Q_X$. Let $\overline{\gamma}=\sum_{j=1}^N \gamma_j\cdot \frac{n_j}{n_P}$. We have for some constant $C>0$:
\[
		\mathbb{E}\  \textbf{R}_n^Q(\pi)\leq Cn_Q\left[\min\left(\left(\frac{K\log\left(K/\delta\right)}{n_P}\right)^{\frac{\alpha+1}{2+\alpha+d+\overline{\gamma}}},\left(\frac{K\log\left(K/\delta\right)}{n_Q}\right)^{\frac{\alpha+1}{2+\alpha+d}}\right) + \frac{K\log\left(K/\delta\right)}{n_Q} + n\delta\right]
	\]
\end{thm}

\begin{proof}[Proof Idea]
	Consider a round $t > n_P$ and any level $r\in\mathcal{R}$ such that $n_r(X_t)>\log(1/\delta)$. Similarly to the proof of Lemma~\ref{lem:acb-bound}, by a Chernoff bound and the definition of the transfer exponent (Definition~\ref{assumption:transfer-exponent}), we have that the covariate count $n_r(X_t)$, with probability at least $1-\delta$ satisfies \[
		n_r(X_t) \gtrsim \left[ (t-1-n_P)r^d + \sum_{j=1}^N n_j\cdot r^{d+\gamma_j}\right].
	\]
	Next, by Jensen's inequality we have \[
		\sum_{j=1}^N \frac{n_j}{n_P}\cdot r^{d+\gamma_j} \geq r^{d+\overline{\gamma}}.
	\]
	Thus, combining the above two inequalities, $n_r(X_t) \gtrsim n_P \cdot r^{d+\overline{\gamma}}$. Using this bound, we can derive a generalization of Lemma~\ref{lem:acb-bound} where $\gamma$ is replaced by $\overline{\gamma}$. All other parts of the proof of Theorem~\ref{thm:adaptive-algo-one-arm-dm} remain the same in showing Theorem~\ref{thm:multiple-shifts-dm}.	
\end{proof}

    \subsection{Lower Bound under Bounded Mass Assumption}
    
    Finally, we present an analogue to Corollary~\ref{cor:lower-bound} under the bounded mass assumption, thus showing the regret bound of Theorem~\ref{thm:adaptive-algo-one-arm-dm} is minimax optimal in a certain regime. The proof is identical to that of Corollary~\ref{cor:lower-bound}, relying on the same online-to-batch conversion procedure. The same theorem cited in Appendix~\ref{app:lower-bound} \citep[Theorem 1 of ][]{kpotufe-martinet} provides us the classification minimax lower-bound required for the online-to-batch conversion.
    
    \begin{cor}[Matching Lower Bounds over Given Regimes]\label{cor:lower-bound-dm}
Let ${\cal T}'$ be the class of all tuples $(P,Q)$ of distributions satisfying Assumptions~\ref{assumption:lipschitz} and \ref{assumption:dimension}, and Definitions~\ref{assumption:margin} and \ref{assumption:transfer-exponent}, with some fixed parameters $(\lambda,C_d,d,C_{\alpha},\alpha,\delta_0,C_{\gamma},\gamma)$. Suppose that $n_P,n_Q$ satisfy:
\begin{equation}\label{eqn:lower-bound-condition-dm}
	6\sqrt{\frac{2\log(n_Q\sqrt{3})}{n_Q}} + \frac{1}{n_Q} < \frac{c}{2} \left(n_P^{\frac{2+d}{2+d+\gamma}} + n_Q\right)^{-\frac{\alpha + 1}{2+d}}.
	\end{equation}
	Then, for any fixed such $n_P,n_Q$ and any contextual bandits policy $\pi$, we have:
\[
    \sup_{(P,Q)\in\mathcal{T}'} \mathbb{E}_{\textbf{X}_n,\textbf{Y}_n} \, \left[ \textbf{R}_n^Q (\pi)\right] \geq \frac{c}{4} n_Q\left(n_P^{\frac{2+d}{2+d+\gamma}} + n_Q\right)^{-\frac{\alpha + 1}{2+d}}.
\]
\end{cor}

\begin{rmk}
Similarly to Remark~\ref{rmk:lower-bound-regimes}, the inequality in \eqref{eqn:lower-bound-condition-dm} corresponds to the regime $n_P = \tilde{O}\left(n_Q^{\frac{2+d+\gamma}{2+2\alpha}}\right)$ with $\alpha < d/2$.
\end{rmk}

\end{document}